%% file: main.tex
\documentclass[article]{article}
\pdfoutput=1
\usepackage[utf8]{inputenc}

\usepackage{fullpage}
\usepackage{microtype}
\usepackage{graphicx}
\usepackage{subcaption}
\usepackage{booktabs} 

\usepackage[colorlinks,linkcolor=blue,citecolor=magenta]{hyperref}
\usepackage[utf8]{inputenc}
\usepackage{graphicx}
\usepackage{amsmath}
\usepackage{amssymb}
\usepackage{mathtools}
\DeclarePairedDelimiter{\ceil}{\lceil}{\rceil}
\DeclarePairedDelimiter{\floor}{\lfloor}{\rfloor}
\usepackage[english]{babel}
\usepackage{mwe}
\newcommand{\pl}{Polyak-\L{}ojasiewicz}

\def\@fnsymbol#1{\ensuremath{\ifcase#1\or *\or \dagger\or \ddagger\or
   \mathsection\or \mathparagraph\or \|\or **\or \dagger\dagger
   \or \ddagger\ddagger \else\@ctrerr\fi}}
\newcommand{\ssymbol}[1]{^{\@fnsymbol{#1}}}

\usepackage[utf8]{inputenc} 
\usepackage[T1]{fontenc}    

\usepackage{booktabs}       
\usepackage{amsfonts}       
\usepackage{nicefrac}       
\usepackage{microtype}      
\usepackage{amsmath}
\usepackage{tcolorbox}
\definecolor{my-green}{cmyk}{0.2, 0.04, 0.1, 0.04, 0.8}
\usepackage{graphicx}
\usepackage{amsthm}
\usepackage{amssymb}
\usepackage[title]{appendix}
\usepackage{algpseudocode}
\usepackage{algorithm}
\usepackage{framed}
\usepackage{mdframed}
\usepackage{mathtools}
\usepackage{url}
\usepackage{cite}
\usepackage{framed}
\usepackage{mdframed}
\usepackage[title]{appendix}
\usepackage{float}

\usepackage{blindtext}

\definecolor{my-green}{cmyk}{0.2, 0.04, 0.1, 0.04, 0.8}
\usepackage{amssymb}
\usepackage{pifont}
\newtheorem{theorem}{Theorem}[section]
\newtheorem{lemma}[theorem]{Lemma}
\newtheorem{remark}{Remark}
\newtheorem{assumption}{Assumption}

\newtheorem{corollary}[theorem]{Corollary}
\newtheorem{definition}{Definition}

\newtheorem{fact}[theorem]{Fact}

\allowdisplaybreaks

\begin{document}
\begin{center}
{\bf{\LARGE{On the Convergence  of Local Descent Methods \\ \vspace{2mm}in Federated Learning}}}

\vspace*{.2in}

{\large{
\begin{tabular}{cc}
Farzin Haddadpour &
Mehrdad Mahdavi \\ \\
\end{tabular}
}}

\begin{tabular}{c}
School of Electrical Engineering and Computer Science\\
The Pennsylvania State University\\
University Park, PA, USA \\
\texttt{\{fxh18, mzm616\}@psu.edu}
\end{tabular}

\vspace*{.1in}
\sloppy
\date{\today}
\end{center}
\begin{abstract}
In federated distributed learning, the goal is to optimize a global training objective defined over distributed devices, where the data shard at each device is sampled from a possibly different distribution (a.k.a., heterogeneous or non i.i.d. data samples). In this paper, we generalize the  local stochastic and full gradient descent  with periodic averaging-- originally designed for homogeneous distributed optimization, to solve nonconvex  optimization problems in federated learning. Although scant research is available on the effectiveness of local SGD in reducing the number of communication rounds in homogeneous setting, its convergence and communication complexity in heterogeneous setting is mostly demonstrated empirically and lacks through theoretical understating. To bridge this gap, we demonstrate that by properly analyzing the effect of unbiased gradients and sampling schema in federated setting, as long as the \textit{gradient diversity of local data shards} is bounded, the implicit variance reduction feature of local distributed methods generalizes to  heterogeneous data shards and exhibits the best known  convergence rates  both in general nonconvex and  nonconvex under \pl~ condition  (generalization of strong-convexity). Our theoretical results   complement the recent empirical studies that demonstrate the applicability of local GD/SGD to federated learning and characterize the conditions that these methods exhibit fast convergence. We also specialize the proposed local method for networked distributed optimization. To the best of our knowledge, the obtained convergence rates are the sharpest known to date on the convergence of local decant methods  with periodic averaging for solving nonconvex federated optimization in both centralized and  networked distributed optimization. 

\end{abstract}
\input{Introduction.tex}
\input{Related-work.tex}


\input{AlgoDes.tex}



\bibliographystyle{plain}
\bibliography{ref}

\newpage
\onecolumn
\input{appendix.tex}

\end{document}

%% file: Introduction.tex
\section{Introduction}\label{sec:intro}
With the emergence of datasets of an unprecedented size and  the availability of distributed computing resources, distributed learning and the use of distributed optimization for machine learning has becoming of increasing importance and often crucial for deployment of large-scale machine learning~\cite{bottou2018optimization}. Distributed learning can leverage parallel processing resources in order to allow learning large-scale problems in reasonable time, and perhaps more importantly, allows handling massive data sets that can not be stored and processed in a single machine. 

The most notable works on distributed optimization focus on consensus problems, where each machine holds a subset of training data which share the same distribution with other machines (i.e., the data shard at each machine is sampled independently and identically (i.i.d.)  from a \textit{single} unknown distribution)  and the goal is to communicate between the machines so as to jointly optimize the average objective to learn a centralized model.  Formally,  assume there are $p$ distributed machines where each machine holds a different data shard $\mathcal{S}_i = \{(\boldsymbol{x}_1,y_1), (\boldsymbol{x}_2,y_2), \ldots, (\boldsymbol{x}_{n_i}, y_{n_i})\}$ with $n_i$ samples that are sampled i.i.d. from a source distribution $\mathcal{D}$ over instance space $\Xi = \mathcal{X} \times \mathcal{Y}$ and the goal is to collectively solve the following optimization problem associated with the empirical risk over whole training data
\begin{equation}
     \textbf{\mbox{(P1)}} \quad \min_{\boldsymbol{w}} \frac{1}{p}\sum_{i=1}^{p}{ f_i(\boldsymbol{w})}, 
\end{equation}
where $f_i(\boldsymbol{w}) =  \frac{1}{|\mathcal{S}_i|}\sum_{(\boldsymbol{x}_i, y_i) \in \mathcal{S}_i}{\ell\left(\boldsymbol{w}; (\boldsymbol{x}_i, y_i)\right)} $ is the average loss over tanning examples in $\mathcal{S}_i$ for a given convex or non-convex loss function $\ell: \mathcal{W} \times \Xi \mapsto \mathbb{R}_{+}$ with $\mathcal{W} \subseteq \mathbb{R}^d$ being the parameter space.

Motivated by learning a centralized global model from training data distributed over hundreds to millions of remote devices  with possibly \textit{different} data distribution and privacy concerns of sharing local data, the Federated Learning (FL) is pioneered as a special case of distributed learning very recently in~\cite{konevcny2016federated}  and has received much attention in the context of machine learning. Unlike the standard distributed learning,  in federated learning each machine holds a different source distribution $\mathcal{D}_i$ over instance space $\Xi = \mathcal{X} \times \mathcal{Y}$ from which it can sample training instances (data distribution across the machines/devices can be arbitrarily \textbf{heterogeneous}), and this distribution corresponds to a local generalization error or risk $R_i(\boldsymbol{w}) = \mathbb{E}_{(\boldsymbol{x}, y) \sim \mathcal{D}_i}[\ell(\boldsymbol{w}; (\boldsymbol{x}, y)))]$ for a prediction model $\boldsymbol{w} \in \mathcal{W}$  and predefined loss function $\ell: \mathcal{W} \times \Xi \mapsto \mathbb{R}_{+}$ (compare to P1 where the goal is to minimize global risk $R(\boldsymbol{w}) = \mathbb{E}_{(\boldsymbol{x}, y) \sim \mathcal{D}}[\ell(\boldsymbol{w}; (\boldsymbol{x}, y))]$). Given a distributed data sample $\mathcal{S} = \mathcal{S}_1 \cup \mathcal{S}_2 \cup \ldots \cup \mathcal{S}_p$ where data shard $\mathcal{S}_i, i=1, \ldots, p$ is sampled from $\mathcal{D}_i$, the goal is to find a single predictor $\boldsymbol{w}$ that performs well on all devices. To this end, we minimize the aggregated empirical risk over all available data either by weighting individual loss functions  proportional  to their sample sizes ~\cite{mcmahan2016communication,konevcny2016federated} (P2) or agnostic minimax loss~\cite{mohri2019agnostic} (P3):
\begin{equation}
\label{eqn:p2:p3}
     \textbf{\mbox{(P2)}} \quad \min_{\boldsymbol{w}} \sum_{i=1}^{p}{\frac{n_i}{n} f_i(\boldsymbol{w})} \quad \quad \textbf{\mbox{(P3)}} \quad \min_{\boldsymbol{w}} \max_{\boldsymbol{\lambda } \in \Delta_p}  \sum_{i=1}^{p}{\lambda_i f_i(\boldsymbol{w})},
\end{equation}
where  $n_i = |\mathcal{S}_i|$ is the size of $i$th data shard, $n$ is the total number of samples and $\Delta_p$ is the $p$-dimensional simplex, i.e., $\Delta_p = \{\boldsymbol{\lambda} \in \mathbb{R}^p\;|\; \sum_{}^{}{\lambda_i} = 1, \lambda_i \geq 0, i=1,2, \ldots, p\}$.

 Compared to standard distributed optimization as in (P1), there are few key challenges we need to overcome in federated learning. First, frequent communication is undesirable in FL as it is expensive and intrusive due to unreliable and relatively slow network connections. As a result, the key challenge in FL is reaching a consensus between possibly very different distributions with minimal number of communications, and the problem still becomes \textit{harder} when more machines are involved.  Beyond expensive communication, another key distinguishing feature of FL is data privacy, where transfer of local data to a single data center for centralized training is prohibited. As a result, federated models are learned by aggregating model updates submitted by devices.  Moreover, in FL, only a subset of devices, say $K \subseteq [p] \triangleq \{1, 2, \ldots, p\}$, participate at each round of training (with either stochastic or adversarial availability),  which requires efficient sampling methods to guarantee the convergence of final model. Last but not least,  to protect confidentiality of the training data, the central machine by design has no visibility into how these updates are generated, making the model vulnerable to a model-poisoning attacks from malicious devices~\cite{bhagoji2018analyzing}.

Since the communication overhead is one of the key challenges that hinders the scalability of distributed optimization algorithms to learn from extremely large number of devices in federated setting, in this paper we  aim at developing communication efficient algorithms for federated learning with provable convergence rates. To this end, we investigate the convergence of local Gradient Descent (GD) and local Stochastic Gradient Descent (SGD) with \textit{periodic averaging} in federated setting. In local GD/SGD, the idea is to perform \emph{local} updates with periodic averaging, wherein machines update their own local models which involve only their local training data, and the models of the different machines are averaged periodically\cite{yu2018parallel,wang2018cooperative,zinkevich2010parallelized,mcdonald2010distributed,zhang2016parallel,zhou2017convergence, stich2018local}. 

The motivating impetus for this work is recent studies that demonstrate that the local SGD is favourable to parallel SGD as it requires less number of communications to converge to the desired accuracy while preserving the linear speedup. For instance, in~\cite{stich2018local} it has been shown that for strongly convex loss functions, with a fixed mini-batch size and after $T$ iterations,  the linear speedup of the parallel SGD is attainable only with $O\left(\sqrt{pT}\right)$ rounds of communication, with each device performing $E = O(\sqrt{T/p})$ local updates before communicating its local model. If $p < T$, this is a significant improvement than the naive parallel SGD which requires $T$ rounds of communication. This result is further generalized and tightened in~\cite{haddadpour2019trading} by demonstrating that under \pl~condition, $O((pT)^{1/3})$ rounds of communication suffice to achieve a linear speed up, that is, an error of $O(1/pT)$. 


 \begin{table}[t]
\centering
\resizebox{1\linewidth}{!}{
\begin{tabular}{|c||c|c|c|c|l|}
  \hline
  \begin{tabular}{l}
  Strategy
  \end{tabular} &
  \begin{tabular}{l}
  Convergence Rate
  \end{tabular} &
  \begin{tabular}{l}
  $E$ 
  \end{tabular} & 
  \begin{tabular}{l} 
  Sampling 
  \end{tabular}  & 
  \begin{tabular}{l}
  Extra Assumptions  on Gradient 
  \end{tabular} &
  \begin{tabular}{l}
  Problem/Method 
\end{tabular} \\
  \hline \hline
   \cite{sahu2018convergence}$^{\dagger}$ &  $O\left(\frac{1}{\rho T}\right)$ & $-$ & \ding{51} &   Bounded dissimilarity (Remark~\ref{rmk:2})\:\& $\nabla^2{f}_j(\boldsymbol{w})\succeq -L\mathbf{I}$ & Non-convex \hfill Local Solver\\
    \hline
   \cite{sahu2018convergence} &  $O\left(\frac{E^2}{T}\right)$ & $O\left(1\right)$ & \ding{51} &   Bounded dissimilarity\:\& $\nabla^2{f}_j(\boldsymbol{w})\succeq -L\mathbf{I}$ & Strongly-convex \hfill Local Solver\\
  \hline
   \cite{li2019convergence} & {$O\left(\frac{E^2}{T}\right)$} & $O\left(1\right)$ & \ding{51} & Bounded gradient & Strongly-convex \hfill SGD \\
  \hline
  \cite{khaled2019first}$^{\ddagger}$  & {$O\left(\frac{1}{T}\right)+O\left(E^2\sigma^2_f\right)$} & $O\left(1\right)$ & \ding{55} & \ding{55} & Convex \hfill GD  \\
  \hline
    \cite{local2019var}$^{\ssymbol{8}}$  & {${O\left(\frac{1}{\sqrt{pT}}\right)}$} & $O\left({T}^{\frac{1}{2}}/p^{1.5}\right)$ & \ding{55} & \ding{55} & Non-convex \hfill SGD  \\
  \hline
  \cite{li2019communication}$^{\ssymbol{9}}$  & $O\left(\sqrt{\frac{E}{pT}}+\frac{pE}{T}\right)$ & $O(1)$ &\boldmath{$-$} & Degree of non-i.i.d.  & {Non-convex \& networked \hfill SGD}
  \\
  \hline
  \textbf{Theorem~\ref{thm:gd}} & \boldmath{$O\left(\exp{(-\mu \eta T)}\right)$} & \textbf{Satisfying (\ref{eq:gd-condition})} & \ding{51} & \textbf{Bounded gradient diversity} & \textbf{Non-convex (PL) \hfill GD} 
  \\
   \hline
   \textbf{Theorem~\ref{thm:proof-undrr-pl}} & \boldmath{$O\left(\frac{1}{KT}\right)$} & \boldmath{$O\left(p^{\frac{1}{3}}T^{\frac{2}{3}}\right)$} & \textbf{\ding{51}} & \textbf{Bounded gradient diversity} &  \textbf{Non-convex  (PL) \hfill SGD} \\
 \hline
  \textbf{Theorem~\ref{thm:FedAvg}} & \boldmath{${O\left(\frac{1}{\sqrt{KT}}\right)}$} & \boldmath{$O\left({T}^{\frac{1}{2}}/K^{1.5}\right)$} & \textbf{\ding{51}} & \textbf{Bounded gradient diversity} &\textbf{Non-convex \hfill SGD} 
  \\
    \hline
    \textbf{Theorem~\ref{thm:serverless-fed}} & \boldmath{${O\left(\frac{1}{\sqrt{pT}}\right)}$} & \boldmath{$O\left({\left(\frac{1-\zeta^2}{1+\zeta^2}\right)}{T}^{\frac{1}{2}}/p^{1.5}\right)$} & \boldmath{$-$} & \textbf{Bounded gradient diversity} &\textbf{Non-convex \& networked \hfill SGD}
  \\
   \hline
\end{tabular}}
\caption{\sffamily{A high level summary of the results of this paper and their comparison to prior state of the art  local GD and local SGD with periodic averaging based algorithms. This table only highlights the dependencies on $T$ (number of iterations), $E$ (the largest number of local updates achieving asymptotic optimal solution), and $K\leq p$ (number of selected devices). We note that all the results with sampling reduces to the convergence results of local-SGD/GD by simply letting $K=p$ and $q_i=\frac{1}{p}$. The rates  obtained in this paper hold as long as the drift among local data shards, quantified by their gradient diversity, is bounded (see Definition~\ref{def:gd-d}). \\
$^{\dagger}$ We note that the constant $\rho>0$  is a function of $\eta,\mu,T,L$, and $B$ which is an upper bound on the dissimilarity among gradient of local objectives.\\
$^{\ddagger}$The additive residual error is defined as $\sigma^2_f = \frac{1}{p}\sum_{j=1}^p\|\nabla{f}_j(\boldsymbol{w}^{*})\|^2$, 
where $\boldsymbol{w}^*$ is the global minimum.\\
$^{\ssymbol{8}}$We note that~\cite{local2019var} is the only scheme that uses explicit variance reduction.\\ $^{\ssymbol{9}}$The bound here is for the  proposed decaying strategy local SGD  and the analysis  for vanilla local SGD in~\cite{li2019communication} suffers from an $O(\kappa_f^2)$ additive residual error where $ \frac{1}{p}\sum_{j=1}^{p}\|\nabla{f}_j(\boldsymbol{w})-\nabla{f}(\boldsymbol{w})\|^2 \leq \kappa_f^2, \; \forall \boldsymbol{w} \in \mathbb{R}^d$ quantifies the drift among local and global gradients.}}
\label{table:1}
\end{table}

These results motivates us to examine the convergence of the local descent methods, both SGD and GD, in centralized and decentralized federated learning. However, to accomplish this goal, there are few key challenges to overcome. First, since the distribution of local data shards are different, the local gradients are \textit{biased} with respect to gradient of the global objective in (P2) and (P3) that poses numerous difficulties. In particular, the existing analysis does not generalize and the convergence of standard methods such as federated averaging~\cite{mcmahan2016communication} (variant of local SGD where a subset of machines participate in aggregation) is not guaranteed~\cite{li2019federated}. To overcome this issue, recent studies   attempt to establish convergence by quantifying the drift among local objectives via different notions of heterogeneity (e.g., dissimilarity measure between gradient vectors at different local machines) and introducing  novel aggregation methods such as proximal regularization~\cite{sahu2018convergence} or  controlled averaging~\cite{karimireddy2019scaffold}. Instead, we aim at characterizing the convergence of local GD/SGD in general federated setting with a focus on understating \textit{how does the  heterogeneity among local data shards affect the convergence?}.  Second, similar to homogeneous setting, while local updates and periodic model averaging  reduces the number of communication rounds, since the model for every iteration is not updated based on the entire data, it suffers from a residual error with respect to fully synchronous SGD. In federated setting overcoming the accumulated residual error is more involved due to heterogeneity as local data shards are far from being representative of the whole data. We also note that in federate setting, the analysis is more involved as only a subset of machines participate in aggregation at every  communication round.  Despite these difficulties, as we will elaborate later in the our theoretical analysis which is also  empirically demonstrated in recent studies~\cite{li2019convergence,sahu2018convergence}, we are able to show that if the averaging period and the learning rate are chosen properly based on the \textbf{gradient diversity of local objectives}, the residual error can be compensated and the convergence is guaranteed. 
\subsection{Contributions}
The main contribution of this paper is to theoretically analyze the convergence of local GD/SGD in federated learning.  Specifically, we show that \textbf{implicit variance reduction} of local descent methods, which is observed in homogeneous setting, even holds  in heterogeneous  distribution of local data conditioned on the fact that hyperparameters (i.e., learning rate or the number of local updates and selected devices) are properly chosen based on the \textit{gradient diversity of local objectives}. Moreover, due to \textbf{restarting property} of local SGD at each communication round, where the server  broadcasts the model to all devices, we can control the residual error caused by local updates and the algorithm is less affected by sampling of devices (we assume that devices are agnostic to random selection schema). The obtained convergence rates  in the context of existing works are elucidated in Table~\ref{table:1}. As we elaborate later in our analysis, our goal is to characterizes the choice of learning rate and number of local updates based on gradient diversity of local objectives to guarantee the convergence.   We also extend all of our convergence rates to the setting where at each communication period parameter server samples a predetermined number of devices. To summarize, the present work makes the following contributions:
\begin{itemize}
    \item We provide the convergence analysis of local SGD with periodic averaging for general non-convex optimization problems in both parameter server and decentralized distributed settings. Our convergence rate is $O\left(\frac{1}{\sqrt{KT}}\right)$ \textit{without any residual error due to dissimilarity of local objectives}. This bound, in case of full device participation ($K=p$), matches the convergence rate of ~\cite{local2019var}, which employs an explicit variance reduction in local SGD.  We note that our analysis reveals that linear speed up can be achieved as long as the learning rate is smaller than a quantity which is inversely proportional to \textit{gradient diversity of data shards}. This theoretical result is consistent with experimental result of~\cite{povey2014parallel}. 
    \item We provide the convergence analysis of local SGD with periodic averaging for heterogeneous data distribution for non-convex objectives under \pl~(PL) condition. Our convergence analysis improves the convergence rate in \cite{sahu2018convergence,li2019convergence} from $O\left(\frac{E^2}{T}\right)$ to $O\left(\frac{1}{KT}\right)$ in terms of dependence on $E$. Our analysis  removes bounded gradient assumption and at the same time it increases the size of local updates from  $E = O\left(1\right)$   to $O\left(T^{\frac{2}{3}}/K^{\frac{1}{3}}\right)$ to achieve the same convergence rate.
    \item We provide the convergence analysis of local GD with periodic averaging for heterogeneous data distribution for non-convex objectives satisfying PL  assumption. While the rate obtained  in~\cite{khaled2019first}, i.e.,  $O\left(\frac{1}{T}\right)+O\left(E^2\sigma^2_f\right)$ where $\sigma^2_f = \frac{1}{p}\sum_{j=1}^p\|\nabla{f}_j(\boldsymbol{w}^{*})\|^2$ is the average gradient at global optimal solution,  indicates a growing residual error proportional to $E$ (even for convex objectives), the rate we obtain does not suffer from a residual error and it  matches  the convergence rate of distributed GD for non-convex objectives under PL condition. Furthermore, our convergence analysis covers the convergence error of distributed GD when  $E=1$.
    \item  In networked distributed setting, we derive an $O\left({T}^{\frac{1}{2}}/p^{1.5}\right)$ convergence rate for local SGD where participating device periodically communicate their local solutions with direct neighbors. Our analysis is based on weaker assumptions compared to concurrent work~\cite{li2019communication}, and even  improves the convergence rate from $O\left(\sqrt{\frac{E}{pT}}+\frac{pE}{T}\right)$ for the proposed   decaying strategy local SGD in~\cite{li2019communication} to $O\left(\frac{1}{\sqrt{pT}}\right)$.
\end{itemize}
\vspace{-4mm}
\paragraph{Organization}{The  remainder  of  this  paper  is  organized  as  follows. In Section~\ref{sec:relwor}, we discuss the most related work in two categories of local SGD with periodic averaging and federated optimization. In Section~\ref{sec:fplg} we review the  Local Federated Descent  optimization algorithm and specialize to  stochastic and full gradient settings. We then discuss the main assumptions we make to obtain the claimed convergence rates both in centralized  and decentralized networked models. The bounds  in Section~\ref{sec:fplg} are stated with some simplifications for the sake of presentation and to  compare these results with the best known bounds in the literature. In Section~\ref{Sec:convergence}, we provide the convergence results in more detail, with more technical aspects of our proofs deferred to the appendices. Finally, in Section~\ref{sec:concl&fud} we summarize the results and mention potential future directions.  }
\vspace{-2mm}
\paragraph{Notation}{Throughout the paper, we adapt the following notation. We use bold-face lower and upper case letters such as $\boldsymbol{w}$ and $\mathbf{W}$ to denote vectors and matrices, respectively. The set of numbers $\{1, 2, \ldots, p\}$ is denoted by $[p]$ for brevity. The derivative of a finite-sum function $f(\boldsymbol{w})$ when evaluated on a subset of training examples $\mathcal{S}$ is denoted by $\nabla f(\boldsymbol{w}; \mathcal{S})$. We use $\mathbb{E}[\cdot]$ to denote the expectation of a random variable. The dot product between two vectors $\boldsymbol{w}$ and $\boldsymbol{w}'$ is  denoted by either $\langle \boldsymbol{w}, \boldsymbol{w}'\rangle$ or $\boldsymbol{w}^{\top}\boldsymbol{w}'$. Throughout this paper, we only consider the $\ell_2$ norm of vectors represented by $\|\cdot\|$. Finally, for a given symmetric matrix $\mathbf{W} \in \mathbb{R}^{p \times p}$ we use $\lambda_1(\mathbf{W}), \ldots, \lambda_p(\mathbf{W})$ to denote its eigenvalues. Finally, the notation $a|b$ is used to indicate $a$ divides $b$.}

%% file: Related-work.tex
\section{Additional Related Research}\label{sec:relwor}
There is a very large body of work on distributed optimization in different settings, and studying their convergence under  various  criteria.   Here we would like to draw connections to and put our work in context of subset  of  work  that  has  given  bounds on the convergence of local GD/SGD with periodic averaging and federated optimization.
\vspace{-3mm}
\paragraph{Local SGD with Periodic Averaging.}
The references \cite{zinkevich2010parallelized,mcdonald2010distributed} introduce the  \emph{one shot} averaging, which can be considered as an extreme case of periodic averaging ($E=T$), and  show empirically that one-shot averaging works well for a range of optimization problems. From theoretical standpoint, yet, the convergence analysis of one-shot averaging is still an open problem. It is also shown in \cite{zhang2016parallel} that one-shot averaging can result in inaccurate solutions for some non-convex optimization problems. Furthermore, they illustrate that more frequent averaging in the beginning can improve the performance. \cite{zhang2012communication,shamir2014distributed,godichon2017rates,jain2018parallelizing} analyze convergence from statistical point of view with only one-pass over the training data which usually is not sufficient for the convergence of training error. Empirical advantages of model averaging is studied in \cite{povey2014parallel,chen2016scalable,mcmahan2016communication,su2015experiments,kamp2018efficient,lin2018don}. In these references, it is indicated empirically that model averaging can speed up convergence to achieve a given accuracy by improving communication cost.  Additionally, for one-shot averaging \cite{jain2018parallelizing} provides speedup with respect to bias and variance for the special case of quadratic square optimization problems. Focusing on distributed linear regression, \cite{haddadpour2018cross} shows that by adding careful amount of redundancy via coding theoretic tools, linear regression can be solved with one-shot communication. There are a few recent work such as  \cite{wang2018cooperative,stich2018local,yu2018parallel}  try to maximize $E$, while achieving linear speed up. The largest  $E=O\left(T^{\frac{3}{4}}/p^{\frac{1}{3}}\right)$ to achieve linear speed up, is provided by  \cite{haddadpour2019local} recently. 

While the majority of the convergence analysis of previous studies of local SGD with periodic averaging such as \cite{wang2018cooperative,stich2018local} is based on i.i.d. data distribution at each machine/device, the reference \cite{haddadpour2019trading} shows a trade-off between the amount of data redundancy and the accuracy of local SGD for general non-convex optimization for non-i.i.d. data distribution at each machine. Also, \cite{yu2019linear} provides the convergence of local SGD with momentum for non-i.i.d. data distribution with maximum allowable $E=O\left(T^{\frac{1}{4}}/ p^{\frac{3}{4}}\right)$.~\cite{local2019var} shows that applying some variance reduction technique over Local SGD with $E=O\left(T^{\frac{1}{2}}/p^{\frac{3}{4}}\right)$, can achieve linear speed up for non-i.i.d. data distribution for general non-convex optimization. In this paper, we show that we can achieve same performance without applying variance reduction technique  and  provide  convergence rates for  local GD algorithm on heterogeneous data and compare our analysis over the recent work of \cite{khaled2019first}.
\vspace{-3mm}
\paragraph{Federated Optimization.}{}
Federated optimization is pioneered  in  \cite{konevcny2016federated,mcmahan2016communication}.  \cite{guha2019one,caldas2018leaf}  study the empirical performance of federated optimization. Even though it is shown that federated optimization works well empirically, theoretical understanding for the case of general non-convex objective and non-i.i.d. data distribution is still lacking. There are a few research effort to analyze the convergence in general. The references \cite{smith2017federated,sahu2018convergence,li2019federated} study the convergence analysis for both strong convex optimization and under some sort of dissimilarity assumption between \emph{optimal local objective function and global optimal solution or gradients at various devices}. We provide the convergence analysis for both general non-convex and non-convex under PL assuming bounded gradient diversity among local devices. To  reduce the drift among local workers, few studies attempt to either regularize updates such as proximal regularization~\cite{sahu2018convergence} or utilize  controlled averaging~\cite{karimireddy2019scaffold} (In Subsection~\ref{subsec-controlled-vatiance-diversity} we make connection between this paper and reducing diversity). In a  concurrent work to the present paper,~\cite{li2019communication} proposed a decaying strategy decentralized local SGD that alternates between multiple local updates and multiple decentralized communications where every device in networks makes multiple local updates followed by multiple decentralized communications with its neighbors. 

Finally,  we note that another research direction in federated learning is the analysis of fairness. In particular, to satisfy fairness with respect to different local objectives,~\cite{mohri2019agnostic} casts the federated optimization into a minmax optimization problem (problem P2 in Eq.~(\ref{eqn:p2:p3})) and provides the convergence analysis for obtained solution. Another recent work~\cite{li2019fair} suggests fair algorithm for federated learning and evaluates their algorithms empirically. For a more comprehensive and up-to-date overview of recent progress in federated learning and  interesting potential future directions see~\cite{li2019federated}.

%% file: AlgoDes.tex
\begin{algorithm}[t]
\caption{\texttt{LFD}($E,K, \boldsymbol{q}$): Local Federated Descent  with Periodic Averaging. }\label{Alg:one-shot-using data samoples}
\begin{algorithmic}[1]
\State \textbf{Inputs:} $\boldsymbol{w}^{(0)}$ as an initial  model shared by all local devices, $E$ as the number of local updates, $K$ as the number of devices selected by server with corresponding sampling probabilities $\boldsymbol{q} = [q_1, q_2, \ldots, q_p]^{\top}$.
\State Server chooses a subset $\mathcal{P}_0$ of $K$ devices at random (device $j$ is chosen with probability $q_j$);
\State \textbf{for $t=0, 1, \ldots, T$ do}
\State $\qquad$\textbf{parallel for all chosen devices $j\in \mathcal{P}_t$ do}: \State $\qquad\quad $\textbf{if} $t$ does not divide $E$ \textbf{do}
\State $\qquad\quad\quad$ $\boldsymbol{w}^{(t+1)}_{j}=\boldsymbol{w}^{(t)}_j-\eta_t~ \boldsymbol{d}_{j}^{(t)}$\label{eq:update-rule-alg}
\State $\qquad\quad$\textbf{else}
\State $\qquad\quad\quad$Each chosen device $j$ sends $\boldsymbol{w}^{(t)}_j$ for $j\in \mathcal{P}_t$ back to the server.
\State $\qquad$Server \textbf{computes} 
\State $\qquad\qquad\bar{\boldsymbol{w}}^{(t+1)}=\frac{1}{K}\sum_{j\in\mathcal{P}_t}\Big[\boldsymbol{w}^{(t)}_j-\eta_t ~{\boldsymbol{d}}_{j}^{(t)}\Big]$
\State $\qquad$Server \textbf{broadcasts} $\bar{\boldsymbol{w}}^{(t+1)}$ to all devices.
\State $\qquad$Sever chooses a set of devices  $\mathcal{P}_t$ with distribution $q_j$.
\State $\qquad\quad$ \textbf{end if}
\State $\qquad$\textbf{end parallel for}
\State \textbf{end}
\State \textbf{Output:} $\bar{\boldsymbol{w}}^{(T)}=\frac{1}{K}\sum_{j\in \mathcal{P}_{T}}\boldsymbol{w}^{(T)}_j$
\vspace{- 0.1cm}
\end{algorithmic}
\end{algorithm}

\section{Local Federated Optimization}\label{sec:fplg}
In this section, we set up the distributed optimization algorithms of interest. Our goal is to show that the local gradient and stochastic gradient descent with periodic averaging also converge for solving distributed optimization problems in federated setting for both general non-convex functions and non-convex functions satisfying  PL condition as long as the drift among local objectives, quantified by gradient diversity, is bounded.  

To do so, we first formally state the optimization problem that we aim at solving and present the Local Federated Decent (\texttt{LFD}) algorithm with periodic averaging, that is a  modified version of local SGD, and thereafter, specialize it to full and stochastic  gradient settings. We also extend \texttt{LFD} to networked optimization where every device can communicate with direct neighbors in communication rounds. We present the main converge rates for proposed algorithms under different standard assumptions and defer the detailed convergence analysis to appendix.  
\subsection{Distributed federated optimization}
As mentioned earlier, in this paper,  we focus on the following distributed optimization problem:
\begin{align}
    \min_{\boldsymbol{w}}f(\boldsymbol{w})\triangleq \sum_{j=1}^p q_jf_j(\boldsymbol{w})\label{eq:global-cost}
\end{align}
where $p$ is the number of devices, and $q_j$ is the weight of $j$th device such that $q_j\geq 0$ and $\sum_{j=1}^pq_j=1$, and $f(\boldsymbol{w})$ is global objective function. 

In federated setting, we assume that the $j$th device holds $n_j$ training data  $\mathcal{S}_j = \{(\boldsymbol{x}_1,y_1), (\boldsymbol{x}_2,y_2) \ldots, (\boldsymbol{x}_{n_j}, y_{n_j})\}$ sampled i.i.d. from $j$th local distribution $\mathcal{D}_j$.  The local cost function $f_j(.)$ is defined by
\begin{align}
    f_j(\boldsymbol{w})\triangleq \frac{1}{n_j}\sum_{(\boldsymbol{x}_i, y_i) \in \mathcal{S}_j}^{}\ell\left(\boldsymbol{w};(\boldsymbol{x}_i, y_i))\right),
\end{align}
where $\ell(\cdot; \cdot)$ is the loss function that could be convex or nonconvex.  We note that when all the local distributions are same $\mathcal{D}_j=\mathcal{D}, j=1,2, \ldots, p$ and local objectives are weighted equally, the optimization problem reduces to the standard homogeneous distributed optimization. 

Before delving into the proposed algorithm and its convergence analysis, we would like to pause and highlight the key challenges we focus on in solving the above optimization problem. First, we note that since different local data shards are generated from a different distribution, in designing optimization algorithms for solving Eq.~(\ref{eq:global-cost}) the heterogeneity in distributions needs to be taken into account. Specifically, the local (stochastic) gradients, while being unbiased with respect to the local objectives, are no longer unbiased estimate of gradient of global objective. Moreover, due to unreliable network connections, e.g., in IoT devices, the frequent communication is undesirable which necessities communication efficient  optimization algorithms. Finally,  since not all devices can participate  at each round of communication, the server needs to sample a subset of devices in aggregating the local solutions. 

To resolve above three key issues, we propose the Local Federated Descant with Periodic Averaging which is the specialization of local SGD with a sampling schema to federated setting. The proposed algorithm, dubbed as  \texttt{LFD}$(E, K, \boldsymbol{q})$,  has three parameters: i) the number of local updates before communicating the local model wih sever denoted by $E$, ii) the number of devices to be sampled at every communication step denoted by $K$, and iii) the weight vector of individual machines $\boldsymbol{q} \in \Delta_p$ (e.g., $q_j = n_j/n$). Assuming the algorithm is running for $T$ iterations, at every iteration $t$ the $j$th device updates its own local version of the model $\boldsymbol{w}^{(t)}_{j}$ via the update rule:
\begin{align}
    \boldsymbol{w}^{(t+1)}_{j}=\boldsymbol{w}^{(t)}_j-\eta_t~ {\mathbf{\boldsymbol{d}}}_{j}^{(t)},\label{eq:risgd-up}
\end{align}
where ${\boldsymbol{d}}_{j}^{(t)}$ is the (stochastic) gradient utilized by $j$th machine at $t$th iteration to locally update the solution.

After every $E$ iterations, we do the model averaging, where the server performs averaging step over local versions of the model received from a randomly selected subset  $\mathcal{P}_t \subseteq [p]$  of devices which is equivalently can be written as:
\begin{align}
    \bar{\boldsymbol{w}}^{(t+1)}=\frac{1}{K}\sum_{j\in\mathcal{P}_t}\Big[\boldsymbol{w}^{(t)}_j-\eta_t~ {\mathbf{\boldsymbol{d}}}_{j}^{(t)}\Big]\label{eq:g-risgd-up}
\end{align}
To pick a subset of devices at communication step, we use the sampling scheme introduced in~\cite{li2019convergence}. Specifically, after each averaging step, server randomly selects a subset $\mathcal{P}_t\subset\{1,\ldots,p\}$ of devices with $|\mathcal{P}_t|=K\leq p$ uniformly at random  \emph{with replacement} according to the sampling probabilities $q_1,\ldots,q_p$.  It is worthwhile to mention that the devices are agnostic to sampling strategy and their updates are exactly same to the case where $K = p$. Our results can be extended to the sampling scheme without replacement, but for the ease of exposition, we only discuss sampling with replacement. 

The detailed steps of the proposed algorithm are shown in Algorithm~\ref{Alg:one-shot-using data samoples}. We note that \texttt{LFD} significantly reduces the number of communications as the model of local machines are aggregated periodically. It is noticeable that  by letting $q_i=\frac{1}{p}$ and $K=p$ in Algorithm~\ref{Alg:one-shot-using data samoples},  \texttt{LFD} reduces to the local GD/SGD algorithm with the key difference that the data shards at different machines do not share the same distribution. This is the key hurdle in analyzing the convergence which necessities careful tuning of learning rate $\eta_t$ and proper choice of number of local updates $E$ or the number of selected devices $K$ as we  elaborate later in the analysis of convergence rates.

\paragraph{Degree of heterogeneity and convergence}{As mentioned earlier, the heterogeneity among local data shards might cause the divergence of vanilla  (local) distributed GD/SGD if the learning rate or the number of local updates are not chosen properly. In fact, the federated averaging is shown to suffer from slow convergence rate or even divergence when data shards drift significantly. The main reason is that when the local objectives drift significantly, e.g., the local gradients are orthogonal or even are at opposite directions, there is no gain in aggregated optimization as the global optimum might significantly depart from the local optimum of individual devices.   As a result, different notions are introduced in literature to quantify the degree of heterogeneity by measuring the discrepancy among local objectives in convergence analysis:~\cite{khaled2019first} assumes    the average gradient at optimal global solution $\sigma^2_f = \frac{1}{p}\sum_{j=1}^p\|\nabla{f}_j(\boldsymbol{w}^{*})\|^2$ is bounded, 
where $\boldsymbol{w}^*$ is the global minimum,the analysis in~\cite{li2019communication} and~\cite{li2019convergence} relies on the bounded variance of local gradients with respect to global gradient  $ \frac{1}{p}\sum_{j=1}^{p}\|\nabla{f}_j(\boldsymbol{w})-\nabla{f}(\boldsymbol{w})\|^2 \leq \kappa_f^2, \; \forall \boldsymbol{w} \in \mathbb{R}^d$ (a quantity that is also called degree of non-i.i.d.), which is also used in~\cite{yu2019linear} to derive the convergence analysis of local SGD with momentum over non-i.i.d. distribution, and finally the analysis of~\cite{sahu2018convergence}   is based on local bounded dissimilarity $\mathbb{E}\|\left[\nabla{f}_j(\boldsymbol{w})\right]\|^2\leq \|\nabla{f}(\boldsymbol{w})\|^2B^2$ where $B$ is an upper bound for $\mathbb{B}(\boldsymbol{w})\triangleq \sqrt{\frac{\|\mathbb{E}\left[\nabla{f}_j(\boldsymbol{w})\right]\|^2}{\|\nabla{f}(\boldsymbol{w})\|^2}}\leq B$ for $1\leq j\leq p$. Similar to previous studies, to guarantee the convergence,  our analyses rely on a condition over hyperpaprameters (i.e., learning rate, number of local updates, and the number of sampled devices) that depends on the degree of heterogeneity among local objectives as defined below:

\begin{definition}[Weighted Gradient Diversity]\label{def:gd-d}
We indicate the following quantity as weighted gradient diversity among local objectives:
\begin{align}
    \Lambda(\boldsymbol{w}, \boldsymbol{q}) \triangleq \frac{\sum_{j=1}^pq_j\left\|\nabla{f}_j(\boldsymbol{w})\right\|_2^2}{\left\|\sum_{j=1}^pq_j\nabla{f}_j(\boldsymbol{w})\right\|_2^2} 
\end{align}
\end{definition}
We  note that the notion of gradient  diversity is first introduced in~\cite{yin2018gradient} to  measures  the  \textit{dissimilarity  between  concurrent  gradient updates} in i.i.d. distributed setting to understand the effect of mini-bath size on speedup, i.e., the largest permitted mini-batch size that can be used without any decay in performance or to avoid speedup saturation. Unlike~\cite{yin2018gradient}, our modified notion of gradient diversity is introduced to \textit{quantify heterogeneity among local objectives} to establish conditions on convergence in non-i.i.d.  stetting.  In federated setting, a different version of this quantity is also appeared in the convergence analysis of~\cite{sahu2018convergence} as denoted by $\mathbb{B}(\boldsymbol{w})$ above.  

In the remainder of this section we specialize \texttt{LFD} to full and stochastic settings and state the main assumptions we make  to establish convergence rates. We also discuss the convergence of \texttt{LFD} in networked distributed optimization.


\subsection{Local Federated GD (\texttt{LFGD})}

In the first specialization of \texttt{LFD} algorithm, dubbed as \texttt{LFGD}, we consider the setting where the local machines compute  the  gradient of their own entire data shard in updating the local solutions in Eq.~(\ref{eq:risgd-up}), i.e,
\begin{align}
    \boldsymbol{d}_j^{(t)}=\mathbf{g}_j^{(t)}\triangleq\nabla{f}_j(\boldsymbol{w}^{(t)}_j;\mathcal{S}_j),
\end{align}
We now turn to state the convergence rate of the local \texttt{LFD} with full gradients. Our convergence analysis is based on the following standard assumptions:

\begin{assumption}[Smoothness and Lower Boundedness]\label{Ass:1}
The local objective function $f_j(\cdot)$ of $j$th device is differentiable for $1\leq j\leq p$ and $L$-smooth, i.e., $\|\nabla f_j(\boldsymbol{u})-\nabla f_j(\mathbf{v})\|\leq L\|\boldsymbol{u}-\mathbf{v}\|,\: \forall \;\boldsymbol{u},\mathbf{v}\in\mathbb{R}^d$. We also assume that the value of global objective function $f(\cdot)$  is bounded below by a scalar ${f^*}$. 
\end{assumption}

\begin{assumption}[$\mu$-\pl~(PL)]\label{Ass:3}
The global objective functions $f(\cdot)$ is differentiable and satisfy the \pl~ condition with constant $\mu$, i.e.,  $\frac{1}{2}\|\nabla f(\boldsymbol{w})\|_2^2\geq \mu\left(f(\boldsymbol{w})-f(\boldsymbol{w}^*)\right)$ holds  $\forall \boldsymbol{w}\in\mathbb{R}^d $ with $\boldsymbol{w}^*$ being the optimal solution of global objective.
\end{assumption}
We remark that the PL condition is a generalization of strong convexity, as  $\mu$-strong convexity implies $\mu$-\pl ~(PL), e.g., see \cite{karimi2016linear} for more details. Therefore, all of our results based on $\mu$-PL assumption also leads to similar convergence rate for  $\mu$-strongly convex functions.  We also note that  while many popular convex optimization problems such as logistic regression and least-squares are often not strongly convex, but satisfy $\mu$-PL condition~\cite{karimi2016linear}. Furthermore,  the PL condition does not require convexity necessarily.

The convergence rates of \texttt{LFGD}  is summarized in the following theorem. The convergence rate is presented with some simplifications  for the sake of presentation, while it is presented in full generality in  Section~\ref{Sec:convergence}. 

\begin{theorem}[informal]\label{thm:gd-informal} For \texttt{LFGD}$(E, K, \boldsymbol{q})$ with $E$ local updates, under Assumptions \ref{Ass:1} and \ref{Ass:3}, if we choose the learning rate $\eta$ and local updates $E$ such that 
\begin{align}        L\eta \Big(1+\frac{4\kappa E}{{(1-\mu\eta)}^{E-1}}\Big)\leq \frac{1}{\lambda}\label{eq:gd-condition-informal},\end{align} holds, where $\lambda$ is an upper bound over the weighted gradient diversity, i.e.,   $\Lambda(\boldsymbol{w}, \boldsymbol{q}) \leq \lambda$ and  all local models are initialized at the same point $\bar{\boldsymbol{w}}^{(0)}$, after $T$ iterations we have:
\small{\begin{align}
    \mathbb{E}\left[f(\bar{\boldsymbol{w}}^{(T)})-f^*\right]&\leq {O\left(e^{-\mu\eta T}\right)},\label{eq:thm1-result-informal} 
\end{align}}
where $f^*$ is the global minimum, $\kappa=\frac{L}{\mu}$,  and the expectation is with respect to randomness is selecting the devices.
\end{theorem}
A few remarks about the rate stated in Theorem~\ref{thm:gd-informal} are in place.
\begin{remark}
From the condition on learning rate, a smaller gradient diversity is favorable to achieve a fast convergence rate as the learning rate is inversely proportional to $\lambda$. To elaborate on this we note that when the data shards at different machines are sampled from same distribution (i.i.d. setting) or even identical, the gradient diversity would be very small and as a result a larger learning rate or number of local updates can be taken. In contrast, when the gradients at different data shards are almost orthogonal, or even on opposite directions, the weighted gradient diversity becomes large which leads to a very small learning rate (in a sense that potentially there is no descent direction in the span of local gradients that could result in reducing the objective as the gradients are at opposite directions which equivalently implies the choice of $\eta \approx 0$ to not hurt the current solution). This observation explains  the slow convergence (or even divergence) of federated averaging when applied to highly heterogeneous data shards.
\end{remark}



\textcolor{black}{\begin{remark}
To understand Theorem~\ref{thm:gd-informal}, let $E=T^{\beta}$ for $0\leq \beta\leq 1$. In one extreme of spectrum, setting $\beta=0$, condition~(\ref{eq:gd-condition-informal}) reduces to $\eta\leq \frac{1}{L\lambda(1+4L\kappa)}$, hence Theorem~\ref{thm:gd-informal} leads to the convergence of fully synchronous distributed \emph{GD} where all the local models are averaged at every iteration. On the other side of the spectrum, if we set $\beta=1$ which corresponds to one-shot distributed GD, condition~(\ref{eq:gd-condition-informal}) reduces to \begin{align}\eta\leq \frac{{(1-\eta\mu)}^{T-1}}{\lambda L\left({(1-\eta\mu)}^{T-1}+4\kappa T\right)}\leq \frac{{(1-\eta\mu)}^{T-1}}{4\lambda L\kappa T}\end{align}
indicating that $\eta$ needs to be set to \emph{zero} for large $T$. Moreover, note that even for the choice of $\eta=O\left(\frac{{(1-\eta\mu)}^{T-1}}{T}\right)$ the convergence bound becomes $$\mathbb{E}\left[f(\bar{\boldsymbol{w}}^{(T)})-f^*\right]\leq O\left(e^{-\mu\eta T}\right)=O\left(e^{-\mu\frac{{(1-\eta\mu)}^{T-1}T}{T}}\right)=O\left(e^{-\mu{(1-\eta\mu)}^{T-1}}\right)$$
which means algorithm does not converge. Therefore, the \texttt{LFGD} operates between these two extremes. 
\end{remark}
}
\textcolor{black}{
\begin{remark}\label{rmk:2}
As discussed before, the original convergence analysis of  federated optimization algorithms also relies on quantifying the dissimilarity of local gradients and global gradient via different notions along with an additional assumption on the boundedness of gradients. Interestingly,  our convergence analysis for both local GD and local SGD does not rely on any additional assumption which is consistent with more recent analysis  in~\cite{khaled2019first}.
\end{remark}}

\paragraph{Comparison to past work} {Before proceeding further, we would like to compare the achived bound to the one obtained in~\cite{khaled2019first} that analyzes the convergence of \texttt{LFGD} when all the machines participate in communication round (full device participation with $K = p$). \textcolor{black}{The convergence analysis in~\cite{khaled2019first} only holds for convex optimization problems and suffers from residual error which is \textit{proportional to $E^2$}, i.e., $O(E^2 \sigma_f^2)$.} In contrast,  while our convergence analysis focuses on non-convex optimization under PL condition, it  demonstrates similar asymptotic performance with distributed GD, and at the same time it allows much bigger $E$ as long as condition in Eq.~(\ref{eq:gd-condition-informal}) is satisfied \textcolor{black}{\emph{where the residual error is fixed and  independent of $E$}}. The detailed comparison of two bounds is summarized in Table~\ref{table:1}. \textcolor{black}{ Interestingly, both  in our analysis and the one proposed in~\cite{khaled2019first}, the convergence rate of  \texttt{LFGD} does not depend on the number of devices and  no assumption is made about the dissimilarity of local gradients as mentioned above.} }

\subsection{Local federated SGD (\texttt{LFSGD})}
We now shift our attention to the case where local machines are only allowed to sample a mini-batch of fixed size to update their local models, dubbed as local federated SGD (\texttt{LFSGD}). In particular, at every iteration $t$ the $j$th device samples a mini-batch $\xi^{(t)}_{j}$ of size $B$,  identically and independently  from  its own data shard $\mathcal{S}_j$ and uses to calculate the local gradient:
\begin{align}
        \boldsymbol{d}_j^{(t)} = \tilde{\mathbf{g}}_{j}^{(t)} &\triangleq \frac{1}{B}\nabla f_j(\boldsymbol{w}_j^{(t)},\xi_{j}^{(t)})\label{eq:sg_workerj}
\end{align}
Following the convention, we assume that mini-batches are unbiased over each machine's data shard.  In other words, using the notation $\mathbf{g}_j^{(t)}\triangleq \nabla{f}_j(\boldsymbol{w}_j^{(t)},\mathcal{S}_j)$, for the stochastic gradient $\tilde{\mathbf{g}}_j^{(t)}$ computed on a minibatch $\xi_j^{(t)}\subset \mathcal{S}_j$ and $|\xi_j^{(t)}|=B$ we have:
\begin{align}
    \mathbb{E}_{\xi_j^{(t)}}\left[\tilde{\mathbf{g}}_j^{(t)}|\boldsymbol{w}_j^{(t)}\right]=\mathbf{g}_j^{(t)}
\end{align}
The convergence analysis of \texttt{LFSGD} relies on the following standard assumption on the varaince of local stochastic gradients~\cite{bottou2018optimization}.
\begin{assumption}[Bounded Local Variance]\label{Ass:2}
For every local data shard $\mathcal{S}_j, j=1, 2, \ldots, p$, we can sample an independent mini-batch $\xi\subset\mathcal{S}_j$   with $|\xi| = B$ and compute an unbiased stochastic gradient $\tilde{\mathbf{g}}_j = \frac{1}{B}\nabla f_j(\boldsymbol{w}; \xi), \mathbb{E}_{\xi}[\tilde{\mathbf{g}}_j] = {\mathbf{g}}_j = \frac{1}{|\mathcal{S}_j|} \nabla f_j(\boldsymbol{w}; \mathcal{S}_j)$ with  the variance bounded as 
\begin{align}
\mathbb{E}_{\xi}\left[\|\tilde{\mathbf{g}}_j-{\mathbf{g}}_j\|^2\right]\leq C_1 \|{\mathbf{g}}_j\|^2+\frac{\sigma^2}{B}   
\end{align}
where $C_1$ is a non-negative constants and inversely proportion to the mini-batch size and $\sigma$ is another constant controlling the variance bound.

\end{assumption}


We begin with a simple theorem that shows the convergence of \texttt{LFSGD} for  non-convex functions  under PL condition. 
\begin{theorem}[Informal]\label{thm:sgd-pl-informal}
\texttt{LFSGD}$(E, K, \boldsymbol{q})$ with $E$ local updates, under Assumptions~\ref{Ass:1}-\ref{Ass:2}, if we choose the learning rate  as  $\eta_t=\frac{4}{\mu(t+a)}$ where $a=\alpha E+4$ with $\alpha$ being a constant satisfying $\alpha\exp{(-\frac{2}{\alpha})}<\kappa\sqrt{192\lambda\left(\frac{K+1}{K}\right)}$ where $\lambda$ is an upper bound on the weighted gradient diversity, for sufficiently large $E$, if initialize all local model parameters at the same point $\bar{\boldsymbol{w}}^{(0)}$, after $T$ iterations, leads to the following error bound:
\begin{align*}
  \mathbb{E}\left[f(\bar{\boldsymbol{w}}^{(T)})-f^*\right]
  \leq  O\left(\frac{1}{K T}\right),
\end{align*}
where $f^*$ is the lower bound over cost function.
\end{theorem}
\begin{remark}
To better illustrate the obtained rate in Theorem~\ref{thm:sgd-pl-informal}, consider the special case of i.i.d. data distributions with $q_j=\frac{1}{p}$ and $K=p$. In this case, the rate reduces to $O(\frac{1}{KT})=O(\frac{1}{pT})$ with $E=O\left(\frac{T^{\frac{2}{3}}}{p^{\frac{1}{3}}B ^{\frac{1}{3}}}\right)$ which matches the convergence error of local SGD in~\cite{haddadpour2019local} with sampling, which is  the known sharpest bound to the best of our knowledge.  
\end{remark}


The next theorem states the convergence of \texttt{LFSGD} for  general non-convex loss functions.

{\color{black}\begin{theorem}[Informal]\label{thm:FedAvg_informal} For \texttt{LFSGD}$(E, K, \boldsymbol{q})$, if for all $0\leq t\leq T-1$,  under Assumptions \ref{Ass:1} and \ref{Ass:2}, for learning rate $\eta=\frac{1}{L}\sqrt{\frac{K}{T}}$, $E=O\left(\sqrt{\frac{T}{K^{3}}}\right)$, $K=\Omega\left(\sqrt{\lambda}\right)$ where $\lambda$ being an upper bound on the weighted gradient diversity, if all local model parameters are initialized at   $\bar{\boldsymbol{w}}^{(0)}$, the average-squared gradient after $T$ iterations is bounded as follows:
\begin{align}
  \frac{1}{T}\sum_{t=0}^{T-1}\mathbb{E}\left[\|\nabla f(\bar{\boldsymbol{w}}^{(t)})\|^2\right] \leq O\left(\frac{1}{\sqrt{K T}}\right).\label{eq:thm1-result} 
\end{align}
\end{theorem}}

\begin{remark}
It is worthy to note that one important implication of Theorem~\ref{thm:FedAvg_informal} is that for smaller $\lambda$, i.e. more drifting local data distributions, linear speed up can be achieved in expectation even when smaller number of device  $K$ are sampled in aggregation step. For instance when data shards are exactly same (i.e., $\nabla f_1(\boldsymbol{w}) = \ldots = \nabla f_p(\boldsymbol{w})$), then $\Lambda (\boldsymbol{w}, \boldsymbol{q}) = 1$  which requires $K = \Omega(1)$. On the other hand, when the gradients at local machines are  orthogonal, e.g., $\nabla f_i(\boldsymbol{w}) = \boldsymbol{e}_i$ where $\boldsymbol{e}_i$ is the $i$th standard basis, $\Lambda (\boldsymbol{w}, \boldsymbol{q}) = p$ which results in $K = \Omega(\sqrt{p})$, assuming $d > p$. In worst case scenario,  when the data on different devices drift significantly or the gradients are at opposite directions, the vanilla aggregation may result in divergence. Finally, we note that the \texttt{LFSGD} algorithm reduces to the local SGD with periodic averaging in homogeneous setting \cite{wang2018cooperative,stich2018local} by letting $K=p$ and sampling mini-batch $\xi_j$ from entire data set, rather than individual data shards, and the obtained bounds match with communication complexity of $E = O\left(\sqrt{\frac{T}{p^{3}}}\right)$.

\end{remark}

\paragraph{Comparison to past work}{To better illustrate the significance of obtained bounds, we compare our  rates to  best known results in stochastic setting in the context of federated optimization. In contrast to~\cite{li2019convergence} and~\cite{sahu2018convergence} that focus on analyzing the convergence of federated optimization for  strongly-convex functions under restricted assumptions as shown in Table~\ref{table:1}, our analysis is for general non-convex problems yet under weaker assumptions (e.g., removing bounded gradient assumption). Specifically, in comparison to \cite{sahu2018convergence}, where  the analysis of  convergence for general non-convex problems uses the assumption of $\nabla^2{f}_j(\boldsymbol{w})\succeq -L\mathbf{I}$, the linear speed up is not achievable. In contrast to~\cite{sahu2018convergence}, while  removing second moment assumption, we are able to  achieve a linear speed up. Furthermore, we improve the convergence error in~\cite{li2019convergence,sahu2018convergence} from $O\left(\frac{E^2}{T}\right)$ to $O\left(\frac{1}{KT}\right)$ for non-convex function satisfying PL condition. Also, we note that  our convergence rate reduces to the $O\left(\frac{1}{KT}\right)$, allowing us to improve the number of local update from $E=O(1)$ in~\cite{li2019convergence,sahu2018convergence} (which is required to  asymptotically match the rate of parallel SGD, i.e., $O(1/(pT))$) to $E={T^{\frac{2}{3}}}/{K^{\frac{1}{3}}}$, which leads to much smaller number of communications. These key differences make our convergence rate tighter and more general than those obtained in~\cite{li2019convergence,sahu2018convergence}.}

\subsection{Networked local federated SGD (\texttt{NFSGD})}
Thus far, we consider the distribution optimization in a centralized setting where a single server node communicates with the local machines to periodically update the global centerlized solution. We now turn to the setting wherein the local machines are distributed in a network and can only communicate with their direct neighbors in communication rounds. Towards this end, we assume that the $p$ devices form a graph $\mathcal{G} = (\mathcal{V}, \mathcal{E})$, where the set of nodes $\mathcal{V}, |\mathcal{V}| = p$ constitutes the machines. 

\begin{algorithm}[t]
\caption{\texttt{NFSGD}($E, \mathbf{W}, \boldsymbol{q}$): Networked Local Federated SGD.  
}
\label{Alg:nfsgd}
\begin{algorithmic}[1]

\State \textbf{Inputs:} $\boldsymbol{w}^{(0)}$ as an initial  model shared by all devices, the mixing matrix $\mathbf{W}$, the objective weights vector $\boldsymbol{q}$, and the number of local updates $E$ (averaging period)
\State \textbf{for $t=0, 1, \ldots, T$ do}

\State $\qquad$\textbf{parallel for all devices ($1\leq j\leq p$) do}
\State $\qquad $\textbf{if} $t$ does not divide $E$ \textbf{do}
\State $\qquad\quad$ Sample a mini-bath of size $B$ and compute $\tilde{\mathbf{g}}_{j}^{(t)}$ 
\State $\qquad\quad$ $\boldsymbol{w}^{(t+1)}_{j}=\boldsymbol{w}^{(t)}_j-\eta~ \tilde{\mathbf{g}}_{j}^{(t)}$\label{eq:update-rule-alg}
\State $\qquad$\textbf{elseif}
\State $\qquad \quad j$th device broadcasts $\boldsymbol{w}^{(t)}_j$ for $1\leq j\leq p$ to its neighbours  according to the mixing matrix $\mathbf{W}$.
\State $\qquad \quad j$th device computes $\boldsymbol{w}_j^{(t+1)}=\sum_{i=1}^pq_i\left[\boldsymbol{w}^{(t)}_j-\eta ~\tilde{\mathbf{g}}_{j}^{(t)}\right]{W}_{ji}$
\State $\qquad$\textbf{end parallel for}
\State \textbf{end}
\State \textbf{Output:} \textcolor{black}{$\bar{\boldsymbol{w}}^{(T)}=\sum_{i=1}^pq_i\boldsymbol{w}^{(T)}_i$}
\vspace{- 0.1cm}
\end{algorithmic}
\end{algorithm}
In the proposed networked local federated SGD (\texttt{NFSGD}) algorithm as detailed in Algorithm~\ref{Alg:nfsgd}, every machine communicates its local solution with the immediate neighbors. We note that unlike centralized algorithms discussed before, no sampling is conducted in communication rounds due to sparsity of network. Also since the main purpose of introducing sampling in centralized setting is to mitigate the communication deficiency which might be caused by slow workers in communication rounds as the server node needs to wait for all other machines to share their local solutions. However, in the networked or decentralized setting, since every machine needs to communicate with direct neighbors, the harm caused by straggler machines is less severe and sampling can be avoided. Consequently, in \texttt{NFSGD} no sampling is conducted in communication rounds.

We now turn to analyzing the convergence of \texttt{NFSGD}. But before that we need to make a standard assumption about the connectivity of underlying communication graph. In particular, to make sure the update in a node can be propagated to other nodes in the network in reasonable number of iterations, we assume that the network is well connected.  To be precise, let $\mathbf{W} \in [0,1]^{p \times p}$ denote the mixing matrix for the network where ${W}_{ij}$ is the wright of link connecting $i$th and $j$th devices.  
\begin{definition}
The weighting matrix $\mathbf{W} \in \mathbb{R}^{p\times p}$ is called mixing matrix for network if it is symmetric and doubly symmetric that satisfies the following conditions:
\begin{itemize}
    \item  $W_{ij}\geq 0, \; \forall i,j \in [p] \times [p]$
    \item For the  all one vector $\mathbf{1} \triangleq [1, \ldots, 1]^{\top}$, we have , $\mathbf{W}\mathbf{1}=\mathbf{1}$ 
\end{itemize}
\end{definition}
We note that when $\mathbf{W} $ is simply an adjacency matrix, communication between any pair of devices, say $i$ and $j$, is possible whenever ${W}_{ij} =1$, and  ${W}_{i,j}=0$ means that there is no direct communication link between  $i$th and $j$th devices.

We make the following standard assumption about the mixing matrix that is also used in analysis of standard networked distributed optimization algorithms~\cite{jiang2017collaborative,lian2017can,wang2018cooperative}. \\

\begin{assumption}\label{Assum:4}
We assume that  for the mixing matrix $\mathbf{W}$ the magnitudes of all eigenvalues
except the largest one are strictly less than one, i.e., $$\zeta=\max \left\{|\lambda_2(\mathbf{W})|,|\lambda_{p}(\mathbf{W})|\right\}< \lambda_1(\mathbf{W})=1,$$
where $\lambda_i (\mathbf{W})$ is the $i$th eigenvalue of the $\mathbf{W}$.
\end{assumption}  

The following theorem states the convergence of \texttt{NFSGD} algorithm. Compared to standard networked distributed optimization algorithms, the key distinguishing ingredient of our analysis is heterogeneity of data distribution. 
{\color{black}\begin{theorem}[Informal]\label{thm:serverless-fed-informal}
For \texttt{NFSGD}$(E, \mathbf{W}, \boldsymbol{q})$, if for all $0\leq t\leq T-1$,  under Assumptions \ref{Ass:1}, \ref{Ass:2} and \ref{Assum:4}, with learning rate $\eta=\frac{1}{L}\sqrt{\frac{p}{T}}$ for $T\geq 4\lambda^2{\left(\frac{C_1}{p}+2\right)}^2$,  and $E=O\left({\left(\frac{1-\zeta^2}{1+\zeta^2}\right)}{T}^{\frac{1}{2}}/p^{1.5}\right)$ with $p=\Omega\left(\sqrt{\lambda}\frac{(1+\zeta^2)}{1+\zeta}\right)$ and all local model parameters are initialized at the same point $\bar{\boldsymbol{w}}^{(0)}$, the average-squared gradient after $T$ iterations is bounded as follows:
\begin{align}
  \frac{1}{T}\sum_{t=0}^{T-1}\mathbb{E} \left[\|\nabla f(\bar{\boldsymbol{w}}^{(t)})\|^2\right]\leq O\left(\frac{1}{\sqrt{pT}}\right)\label{eq:thm1-result-informal} 
\end{align}
\textcolor{black}{where $\zeta$ is the second largest eigenvalue of mixing matrix $\mathbf{W}$. }
\end{theorem}}

\begin{remark}
As can be seen in Table~\ref{table:1}, our convergence rate claimed in Theorem~\ref{thm:serverless-fed-informal} matches the convergence rate obtained in~\cite{local2019var}, but~\cite{local2019var} utilizes an explicit variance reduction in local SGD. Moreover,  the empirical results of~\cite{local2019var} demonstrates that for non-i.i.d. data distributions, local SGD with explicit variance reduction outperforms vanilla local SGD, which leaves a gap on the theoretical understanding of explicit variance reduction technique in federated setting and is worthy of further investigation.
\end{remark}
\paragraph{Comparison to past work}{In context of decentralized algorithms, the concurrent work ~\cite{li2019communication} proposed a variant of local SGD that alternatives between multiple local updates and multiple communication steps and shows that under additional assumption of $\frac{1}{p}\sum_{j=1}^p\|\nabla{f}_j(\boldsymbol{w})-\nabla{f}(\boldsymbol{w})\|^2\leq \kappa_f^2$ (i.e., degree of non-i.i.d.), the proposed algorithm convergences at rate of  $\frac{1}{T}\sum_{t=0}^{T-1}\mathbb{E}\left[\|\nabla f(\bar{\boldsymbol{w}}^{(t)})\|^2\right]\leq \frac{2\big[f(\bar{\boldsymbol{w}}^{(0)})-f^*\big]}{\eta T}+\frac{\eta L\kappa_f^2}{p B}+\frac{4L^2\eta^2\sigma^2}{B}C_1+\frac{4L^2\eta^2\kappa_f^2}{B}C_2$ where $C_1$ and $C_2$ are  constants depending on $E$ and $\zeta$ in~\cite{li2019communication}. Furthermore, it is shown that  with help of decaying $E$ and with certain choice of learning rate, convergence rate of $O\left(\sqrt{\frac{E}{pT}}+\frac{pE}{T}\right)$ is achievable. However, our analysis-- while removing additional assumption on the dissimilarity of gradients, with proper choice of $E$ and $\eta$  tightens the convergence rate to $\frac{1}{\sqrt{pT}}$ even with fixed number of local updates $E$.}

\subsection{Guaranteed convergence via reducing the gradient diversity}
\label{subsec-controlled-vatiance-diversity}
The obtained bounds heavily relies on the fact that we  properly choose the learning rate, and the number of local updates or sampled machines  based on the gradient diversity among local objectives (e.g., $\eta \propto\frac{1}{\lambda}$, where $\lambda$ is an upper bound over gradient diversity quantity in Definition~\ref{def:gd-d}), in the sense that violation of bounded gradient diversity assumption may lead to convergence failure. Therefore, when the gradient diversity quantity becomes large (as it is the case for highly non-i.i.d. data distributions), tuning learning rate can be difficult. This observation sheds light on devising  mechanisms to reduce the gradient diversity to improve the convergence. In fact, such a mechanism for convex objectives is  recently proposed in~\cite{karimireddy2019scaffold}, where authors suggest to augment the local gradients with a controlled variance variate  $\mathbf{c}_j, j=1,2,\ldots, p$  to reduce the drift among gradients shared by different devices. Specifically, the  suggested stochastic controlled averaging  strategy update can be written as $\underline{\tilde{\mathbf{g}}}_j^{(t+1)}={\tilde{\mathbf{g}}}_j^{(t+1)}-\mathbf{c}_j+\frac{1}{p}\sum_{j=1}^p\mathbf{c}_j$, where the standard averaged stochastic gradient is substituted by $\underline{\tilde{\mathbf{g}}}_j^{(t+1)}$. The claim is  that due to used correction term $\mathbf{c}_j-\frac{1}{p}\sum_{j=1}^p\mathbf{c}_j$, under the  condition that $\left\|\mathbf{c}_j-\mathbf{c}_i\right\|_2^2-2\left\langle{\tilde{\mathbf{g}}}_j-{\tilde{\mathbf{g}}}_i,\mathbf{c}_j-\mathbf{c}_i\right\rangle\leq0$, the new updating scheme reduces the gradient diversity. To see this, for a pair of devices $i \neq j$, we have \begin{align}
    \left\|\underline{\tilde{\mathbf{g}}}_j-\underline{\tilde{\mathbf{g}}}_i\right\|_2^2&=\left\|{\tilde{\mathbf{g}}}_j-{\tilde{\mathbf{g}}}_i-\left(\mathbf{c}_j-\mathbf{c}_i\right)\right\|_2^2\nonumber\\
    &=\left\|{\tilde{\mathbf{g}}}_j-{\tilde{\mathbf{g}}}_i\right\|_2^2+\left\|\mathbf{c}_j-\mathbf{c}_i\right\|_2^2-2\left\langle{\tilde{\mathbf{g}}}_j-{\tilde{\mathbf{g}}}_i,\mathbf{c}_j-\mathbf{c}_i\right\rangle\nonumber\\
    &\stackrel{\text{\ding{192}}}{\leq} \left\|{\tilde{\mathbf{g}}}_j-{\tilde{\mathbf{g}}}_i\right\|_2^2
\end{align}
where \ding{192} comes from condition $\left\|\mathbf{c}_j-\mathbf{c}_i\right\|_2^2-2\left\langle{\tilde{\mathbf{g}}}_j-{\tilde{\mathbf{g}}}_i,\mathbf{c}_j-\mathbf{c}_i\right\rangle\leq0$. Therefore, the proposed approach mitigates the drift among local gradients by reducing the gradient diversity which guarantees the convergence. Interestingly, we  note that the algorithm introduced in~\cite{local2019var} can also be considered as diversity reducing schema, where an explicit variance reduction technique is introduced to reduce the gradient diversity among devices.

\section{Convergence Analysis}\label{Sec:convergence}
In this section, we present the detailed convergence analysis of the proposed  algorithms.  We state the main results and defer the proofs to the Appendix. Before stating the convergence rates for proposed algorithms,  let us  first illustrate the key technical contribution to derive the claimed bounds.  To do so, let us define an auxiliary variable $
    \bar{\boldsymbol{w}}^{(t)}=\frac{1}{K}\sum_{j\in \mathcal{P}_t}\boldsymbol{w}_j^{(t)}$, which is
 the average model across selected machines $\mathcal{P}_t$ at iteration $t$. We note that the per iteration auxiliary average is introduced for the ease of derivations as it is only computed  in communication rounds. Using the definition of $\bar{\boldsymbol{w}}^{(t)}$, the update rule in Algorithm \ref{Alg:one-shot-using data samoples}, can be written as: 
\begin{align}
    \bar{\boldsymbol{w}}^{(t+1)}=\bar{\boldsymbol{w}}^{(t)}-\eta_t \left[\frac{1}{K}\sum_{j\in \mathcal{P}_t}{{\boldsymbol{d}}_j^{(t)}}\right],\label{eq:equivalent-update-rule}
\end{align}
which can be written equivalently as  
$$\bar{\boldsymbol{w}}^{(t+1)}=\bar{\boldsymbol{w}}^{(t)}-\eta _t\nabla f(\bar{\boldsymbol{w}}^{(t)})+\eta_t\left[\nabla f(\bar{\boldsymbol{w}}^{(t)})-\frac{1}{K}\sum_{j\in \mathcal{P}_t}{{\boldsymbol{d}}_j^{(t)}}\right],$$ 
 thus establishing a connection between our algorithm and the perturbed SGD with deviation $\big(\nabla f(\bar{\boldsymbol{w}}^{(t)})-\frac{1}{K}\sum_{j\in \mathcal{P}_t}{{\boldsymbol{d}}_j^{(t)}}\big)$. 
 Two key technical difficulties to bound the effect of perturbed gradients are as follows: i) since the distribution of data at different machines is different, the averaged local gradients are \textit{biased} with respect to gradient of the global objective and ii) periodic model averaging introduces a  residual error with respect to fully synchronous setting that needs to be controlled appropriately. Indeed, we show that by averaging with properly chosen number of local updates, we can reduce the variance of biased gradients and obtain the desired convergence rates, indicating that implicit variance reduction feature of local SGD that is observed in i.i.d setting~\cite{stich2018local}  generalizes to non-i.i.d setting as well with a careful but somehow involved  analysis.
 
We note that in dealing with non-i.i.d. data as in federated learning caution needs to be exercised. For instance, the analysis of~\cite{wang2018cooperative} for non-convex or the analysis of~\cite{stich2018local} for strongly convex problems in federated setting do not directly  handle non-i.i.d. distribution of data shards as both works make the unbiased sampling assumption over entire data set to \emph{decouple full gradient over (entire data set) at average model from full gradient observed at local models}. Then, simply utilizing the Lipschitz continuity of the global cost function $f(\boldsymbol{w})$ in Eq.~(\ref{eq:global-cost}) suffices to obtain the desired bounds. However, our analysis (in particular Lemma~\ref{lemma:cross-inner-bound-unbiased} in appendix) shows that \emph{we can decouple gradients over each device's full data shard observed at average model from corresponding local model to
use Lipschitz continuity of local cost functions (Assumption~\ref{Ass:1})}. 
\subsection{Convergence of  \texttt{LFGD} }
We start by stating the main theorem on convergence of local federated descent optimization with full local gradients.
\begin{theorem}\label{thm:gd} For \texttt{LFGD}$(E, K, \boldsymbol{q})$ with $E$ local updates, under Assumptions \ref{Ass:1} - \ref{Ass:3}, if we choose the learning rate, $\eta$, number of local updates $E$, such that \begin{align}       
\eta \left(L+\frac{4\kappa LE}{\mu{(1-\mu\eta)}^{E-1}}\right)\leq \frac{1}{\lambda}\label{eq:gd-condition},\end{align} 
holds, where $\lambda$ is an upper bound over the weighted gradient diversity, i.e.,   $\Lambda(\boldsymbol{w}, \boldsymbol{q}) \leq \lambda$ and all local model parameters initialized at the same point $\bar{\boldsymbol{w}}^{(0)}$, after $T$ iterations we have:
\small{\begin{align}
     \mathbb{E}\left[f(\bar{\boldsymbol{w}}^{(T)})-f^*\right]&\leq e^{-\mu \eta T}\mathbb{E}\left[f(\bar{\boldsymbol{w}}^{(0)})-f^*\right],\label{eq:thm1-result} 
\end{align}}
where $f^*$ is the global minimum and expectation is taken with respect to randomness in selecting the devices.
\end{theorem}
The proof of theorem is given in Appendix~\ref{sec:app:proof:gd} and relies  on  several  novel  ideas: i) as  local gradients are biased  with respect to the  gradient of global objective $\mathbb{E}\left[\mathbf{g}_j^{(t)}\triangleq\nabla{f}_j(\boldsymbol{w}^{(t)}_j;\mathcal{S}_j)\right] \neq \nabla f(\boldsymbol{w}^{(t)}_j)$, we need to bound the deviation of local gradients from gradient used in communication step, ii) since the local updates of devices may be very different from each other due to heterogeneity of local data shards,  we need to bound the deviation of local intermediate solutions from virtual averaged solution, i.e., $\|\bar{\boldsymbol{w}}^{(t)}-\boldsymbol{w}_j^{(t)}\|^2$ for all devices $j = 1, 2, \ldots, p$. Equipped with these two results, we show that  by proper choice of learning rate $\eta$ and number of local updates $E$, the desired bound is achievable conditioned on the fact that the weighted gradient diversity among local objectives is bounded.

We note that one immediate implication of Theorem~\ref{thm:gd} is that with proper choice of $E$ local GD has similar convergence rate to the distributed GD asymptotically. Similar observation can be found in~\cite{khaled2019first}. Furthermore, similar to the analysis in~\cite{khaled2019first}, an interesting observation is that the convergence rate is independent of the number of devices due to fact that the devices are agnostic to sampling schema. 
\subsection{Convergence of  \texttt{LFSGD} }

We now state the convergence rate of \texttt{LFSGD} for  non-convex objectives under PL condition.
\begin{theorem}\label{thm:proof-undrr-pl}
 \texttt{LFSGD}$(E, K, \boldsymbol{q})$ with $E$ local updates, under Assumptions \ref{Ass:1} to \ref{Ass:2}, if we choose the learning rate  as  $\eta_t=\frac{4}{\mu(t+a)}$ where $a=\alpha E+4$ with $\alpha$ being a constant satisfying $\alpha\exp{(-\frac{2}{\alpha})}<\kappa\sqrt{192\lambda\left(\frac{K+1}{K}\right)}$,
 {initializing all local model parameters at the same point $\bar{\boldsymbol{w}}^{(0)}$, {{for $E$ sufficiently large to ensure that $4{(a-3)}^{E-1}L(C_1+K)\leq \frac{64L^2 (K+1)}{\mu K}(E-1)E(a+1)^{E-2}$, $\frac{32L^2}{\mu}C_1(E-1)(a+1)^{E-2}\leq \frac{64L^2}{\mu}(E-1)E(a+1)^{E-2}$}} and }
{\begin{align}
E\geq 1+\frac{\alpha^2+6\alpha}{\lambda(\frac{K+1}{K})192\kappa^2e^{\frac{4}{\alpha}}-\alpha^2}+\frac{\sqrt{5}}{\sqrt{\lambda(\frac{K+1}{K})192\kappa^2e^{\frac{4}{\alpha}}-\alpha^2}}\end{align}}
 after $T$ iterations we have:
\small{\begin{align}
    \mathbb{E}\left[f(\bar{\boldsymbol{w}}^{(T)})-f^*\right]&\leq \frac{a^3}{(T+a)^3}\mathbb{E}\left[f(\bar{\boldsymbol{w}}^{(0)})-f^*\right]+ \frac{4\kappa \sigma^2T(T+2a)}{\mu KB(T+a)^3}+\frac{256\kappa^2\sigma^2T(E-1)}{\mu KB(T+a)^3}
\end{align}}
where $f^*$ is the global minimum and $\kappa = L/\mu$ is the condition number and $\lambda$ is the upper bound on the weighted gradient diversity. 
\end{theorem}
The proof can be found in Appendix~\ref{sec:app:sgd:undrr-pl}.

The following corollary shows that the above rate reduces to the best-known upper bound when all functions are identical (i.i.d. setting). 
\begin{corollary}
In Theorem \ref{thm:proof-undrr-pl}, choosing $E=O\left(\frac{T^{\frac{2}{3}}}{K^{\frac{1}{3}}B ^{\frac{1}{3}}}\right)$ leads to the following error bound:
\begin{align*}
  \mathbb{E}\left[f(\bar{\boldsymbol{w}}^{(T)})-f^*\right]
  \leq O\left(\frac{BK{(\alpha E+4)}^3+T^2}{B K (T+a)^3}\right) = O\left(\frac{1}{KB T}\right),
\end{align*}
\end{corollary}


The following theorem states the convergence rate of stochastic local SGD for general non-convex objectives. The proof is deferred to Appendix~\ref{sec:app:localsgd:general}.
\begin{theorem}\label{thm:FedAvg} For \texttt{LFSGD}$(E, K, \boldsymbol{q})$, if for all $0\leq t\leq T-1$,  under Assumptions \ref{Ass:1} and \ref{Ass:3}, if the learning rate satisfies \begin{align}
    -\frac{\eta}{2\lambda}+\frac{(K+1)L^2\eta^3[2C_1+E(E+1)]}{2K}+\frac{L\eta^2}{2}\left(\frac{C_1}{K}+1\right)\leq 0\label{eq:cnd-thm4.3}
\end{align}
and all local model parameters are initialized at the same point $\bar{\boldsymbol{w}}^{(0)}$, the average-squared gradient after $\tau$ iterations is bounded as follows:
\begin{align}
  \frac{1}{T}\sum_{t=0}^{T-1}\mathbb{E} \left[\|\nabla f(\bar{\boldsymbol{w}}^{(t)})\|^2\right]\leq \frac{2\big[f(\bar{\boldsymbol{w}}^{(0)})-f^*\big]}{\eta T}+\frac{\eta L\sigma^2}{K B}+\frac{2\eta^2L^2(E+1)\sigma^2}{B}\left(\frac{K+1}{K}\right)\label{eq:thm1-result} 
\end{align}
where $f^*$ is the global minimum and $\lambda$ is the upper bound over the weighted gradient diversity.
\end{theorem}
\begin{remark}
We note  that for the choice of $\eta=\frac{1}{L}\sqrt{\frac{K}{T}}$ the condition~(\ref{eq:cnd-thm4.3}) reduces to $(E+1)^2\leq\left[\frac{T}{K}\left(\frac{K}{\lambda(K+1)}-\sqrt{\frac{K}{T}}(\frac{K+C_1}{K+1})\right)-2C_1\right]=O\left(\frac{T}{\lambda K}\right)$. Therefore, setting the number of local updates to $E=O\left(\sqrt{\frac{T}{K^{3}}}\right)$ when $K=\Omega\left(\sqrt{\lambda}\right)$ does not violate the learning rate condition, resulting in an $O\left(K^{1.5}T^{0.5}\right)$ communication complexity.
\end{remark}

\begin{corollary}
From condition over learning rate in Theorem~\ref{thm:FedAvg} we derive the following relationship among $\eta, K$ and $E$ that is required for the convergence of \texttt{LFSGD}:
\begin{align}
    \eta\leq \frac{-L(\frac{C_1}{K}+1)+\sqrt{L^2{(\frac{C_1}{K}+1)^2}+\frac{4}{\lambda}(\frac{K+1}{K})L^2(2C_1+E(E+1))}}{2(\frac{K+1}{K})L^2(2C_1+E(E+1))}\approx O\left(\frac{1}{LE\sqrt{\lambda}}\right)
\end{align}
\end{corollary}

An immediate result of the Theorem~\ref{thm:FedAvg} is the following: 
\begin{corollary}\label{corol:LocalSGD}

For \texttt{LFSGD}$(E, K=p, \boldsymbol{q})$ with full participation, if for all $0\leq t\leq T-1$,  under Assumptions~\ref{Ass:1} and \ref{Ass:2}, the learning rate satisfies \begin{align}
    -\frac{\eta}{2\lambda}+\frac{(p+1)L^2\eta^3[2C_1+E(E+1)]}{2p}+\frac{L\eta^2}{2}\left(\frac{C_1}{p}+1 \right)\leq 0
\end{align}
and all local model parameters are initialized at the same point $\bar{\boldsymbol{w}}^{(0)}$, the average-squared gradient after $E$ iterations is bounded as follows:
\begin{align}
  \frac{1}{T}\sum_{t=0}^{T-1}\mathbb{E}\left[\|\nabla f(\bar{\boldsymbol{w}}^{(t)})\|^2\right]\leq \frac{2\big[f(\bar{\boldsymbol{w}}^{(0)})-f^*\big]}{\eta T}+\frac{L\eta\sigma^2}{p B}+\frac{2\eta^2L^2(E+1)\sigma^2}{B}\left(\frac{p+1}{p}\right)\label{eq:thm1-result} 
\end{align}
where $f^*$ is the global minimum. 
\end{corollary}
The above rate corresponds to the convergence analysis of local SGD with heterogeneous data distribution without any additional assumption.
\begin{remark}
Recently~\cite{local2019var} established a  convergence rate similar to Theorem~\ref{thm:FedAvg} using a variance reduction technique over local SGD for non-iid data. Due to similarity of obtained rates,  our analysis demonstrates  that the local SGD inherently has an implicit variance reduction feature. 
\end{remark}

\subsection{Convergence of  \texttt{NFSGD} }

The following theorem establishes the convergence rate of networked local SGD algorithm.
\begin{theorem}\label{thm:serverless-fed}
For \texttt{NFSGD}($E, \mathbf{W}, \boldsymbol{q}$) algorithm, if for all $0\leq t\leq T-1$,  under Assumptions \ref{Ass:1}, \ref{Ass:2} and \ref{Assum:4}, if the learning rate satisfies
    \begin{align}
    \frac{2L^2\eta^2C_1 E}{1-\zeta^2}+\frac{L^2\eta^2E^2}{1-\zeta}\left(\frac{2\zeta^2}{1+\zeta}+\frac{2\zeta}{1-\zeta}+\frac{E-1}{E}\right)+\eta L\left(\frac{C_1}{p}+1\right)\leq \frac{1}{\lambda}\label{eq:net-et-con}
\end{align}
and all local model parameters are initialized at the same point $\bar{\boldsymbol{w}}^{(0)}$, the average-squared gradient after $E$ iterations is bounded as follows:
\begin{align}
  \frac{1}{T}\sum_{t=0}^{T-1}\mathbb{E}\left[\|\nabla f(\bar{\boldsymbol{w}}^{(t)})\|^2\right]\leq \frac{2\big[f(\bar{\boldsymbol{w}}^{(0)})-f^*\big]}{\eta T}+\frac{\eta L\sigma^2}{p B}+\frac{L^2\eta^2\sigma^2}{B}\left(\frac{1+\zeta^2}{1-\zeta^2}E-1\right)\label{eq:thm1-result} 
\end{align}
where $f^*$ is the global minimum and $\zeta$ is the second largest eigenvalue of mixing matrix $\mathbf{W}$ and $\lambda$ is the upper bound on the weighted gradient diversity.
\end{theorem}

\begin{remark}
The proof of theorem can be found in Appendix~\ref{sec:app:serverless-fed}. Interestingly,  the convergence rate matches the best-known rate in homogeneous data distribution counterpart in~\cite{wang2018cooperative}, and can be extended to obtain  linear speed up and we exclude the proof here.
\end{remark}
\begin{remark}
We note that for the choice of $\eta=\frac{1}{L}\sqrt{\frac{p}{T}}$ the condition~(\ref{eq:net-et-con}) reduces to  $E\leq (1-\zeta)\sqrt{\frac{T}{10 \lambda p}}
$ (see Corollary~1 in \cite{wang2018cooperative} and proof of Theorem~\ref{thm:serverless-fed-informal} in appendix). Therefore, the choice of $E=O\left(\left(\frac{1-\zeta^2}{1+\zeta^2}\right)\sqrt{\frac{T}{p^{3}}}\right)$ with $p=\Omega\left(\sqrt{\lambda}\frac{(1+\zeta^2)}{1+\zeta}\right)$ as the number of local updates at each device does not violate the learning rate condition. 
\end{remark}

\section{Conclusion and Future Directions}
\label{sec:concl&fud}
In this paper, we studied the convergence of local full and stochastic gradient descent algorithms with periodic averaging in distributed federated learning, where the distributions of data shards are heterogeneous. For general non-convex and non-convex under \pl~condition optimization problems, we established the best known convergence rates via careful analysis of learning rate and number of local updates based on the gradient diversity of local objectives. We believe that the core techniques we introduced to derive the convergence rates for biased (stochastic) gradients and coping with residual errors manifested by local updates are interesting by their own and could play a role in further theoretical analysis of distributed optimization algorithms in federated setting. 

We leave a number of issues for future research. One potential future direction is to see whether we can improve the convergence of federated algorithm with applying adaptive synchronization scheme (i.e., reducing communication period adaptively). We note that despite recent progress on analyzing the communication complexity of local methods,   a rigorous understanding of advantage of local updates from communication complexity standpoint  still remains an open question in both homogeneous and heterogeneous settings.   Another interesting future direction would be considering the effect of dynamic mini-batch size over convergence analysis, interpolating between local GD and local SGD. Furthermore, tightening the convergence analysis of variance reduced local SGD would be another potential future work as our analysis demonstrated that the vanilla local SGD, thanks to its implicit variance reduction, enjoys the same convergence rate with a recently proposed explicit variance reduction proposal.  Generalizing our analysis to agnostic setting (i.e., problem P3) to reduce the number of communications is  an interesting open question that is worthy of investigation. Finally, as our analysis demonstrated, the convergence in federated setting heavily depends on the gradient diversity of local objectives and as a result devising  effective mechanisms to reduce  diversity among local gradients to obtain faster rates would be an interesting research question.

\section*{Acknowledgment}
We would like to thank Anit Kumar Sahu for clarification on the rates obtained in~\cite{sahu2018convergence} and bringing~\cite{karimireddy2019scaffold} to our attention.  We also would like to thank Ziyi Chen for useful comments on the early version of this paper and pointing out a technical error.

%% file: appendix.tex
\appendix
\section*{Appendix}
\label{Appendix}

Before proceeding to detailed proofs, we introduce some notation for  the clarity in presentation. Recall, we use $\mathcal{P}_t$ to denote the subset of randomly selected machines/devices at  each averaging period with  $|\mathcal{P}_t| = K$, where the probability of choosing $i$th worker  is $q_i$ with $\sum_{i=1}^pq_i=1$. We use  $\mathbf{g}_i={\nabla{f}_i(\boldsymbol{w})\triangleq \nabla{f}(\boldsymbol{w};\mathcal{S}_i)}$ and $\tilde{\mathbf{g}}_i \triangleq \nabla{f}(\boldsymbol{w};\mathcal{\xi}_i)$ for ${1\leq i\leq p}$ to denote the full gradient and stochastic gradient at $i$th data shard, respectively, where $\xi_i \subseteq \mathcal{S}_i$ is a uniformly sampled mini-bath. The corresponding quantities evaluated at $i$th machines   local solution at  $t$th iteration $\boldsymbol{w}_i^{(t)}$  are denoted by $\mathbf{g}_i^{(t)}$ and $\tilde{\mathbf{g}}_i^{(t)}$. We also define the following notations   \begin{align*}{\boldsymbol{w}}^{(t)}&=\begin{bmatrix}{\boldsymbol{w}}^{(t)}_1, & \ldots, & {\boldsymbol{w}}^{(t)}_p\end{bmatrix},\\
{\xi}^{(t)}&=\{{\xi}^{(t)}_1, \ldots, {\xi}^{(t)}_p\},
\end{align*} 
to denote the set of local solutions  and sampled mini-batches at iteration $t$ at different machines, respectively.

We use notation $\mathbb{E}[\cdot]$ to denote the conditional expectation $\mathbb{E}_{{\xi}^{(t)}|{\boldsymbol{w}}^{(t)}}
[\cdot]$. We indicate the expectation over random device selection at server at each communication round  by  $\mathbb{E}_{\mathcal{P}_{t}}[\cdot]$.

 The following  short-hand notation   will be found useful in the analysis of the convergence of variants of \texttt{LFD} algorithm:
\begin{align}
    \bar{\boldsymbol{w}}^{(t)}  \triangleq\frac{1}{K}\sum_{j\in\mathcal{P}_t}{\boldsymbol{w}}^{(t)}_j,\quad
   {{\tilde{\mathbf{g}}}^{(t)}}\triangleq \frac{1}{K}\sum_{j\in\mathcal{P}_t} {{\tilde{\mathbf{g}}}^{(t)}}_j,\quad {{\mathbf{g}}^{(t)}}\triangleq \frac{1}{K}\sum_{j\in\mathcal{P}_t} {\mathbf{g}}_j^{(t)}
\end{align}
Furthermore, in what follows we assume that $\lambda$ is an upper bound on the weighted gradient diversity among local objectives, i.e., 
\begin{align}
    \Lambda(\boldsymbol{w}, \boldsymbol{q}) \triangleq \frac{\sum_{j=1}^pq_j\left\|\nabla{f}_j(\boldsymbol{w})\right\|_2^2}{\left\|\sum_{j=1}^pq_j\nabla{f}_j(\boldsymbol{w})\right\|_2^2}  \leq \lambda
\end{align}
Finally, recall that the updating rule for the proposed  local decent algorithms is as follows:
\begin{align}
    \bar{\boldsymbol{w}}^{(t+1)}=\bar{\boldsymbol{w}}^{(t)}-\eta_t\boldsymbol{d}^{(t)},\label{eq:dec-up-rule}
\end{align}
where $\boldsymbol{d}^{(t)}\triangleq\frac{1}{K}\sum_{j\in\mathcal{P}_t}{\boldsymbol{d}}^{(t)}_j$ with $\boldsymbol{d}^{(t)}_j$ being either stochastic or full gradient  computed at $j$th machine at iteration $t$. From the updating rule in Eq.~(\ref{eq:dec-up-rule}) and assumption on the $L$-smoothness of the objective function, we have the following inequality:
\begin{align}
    f(\bar{\boldsymbol{w}}^{(t+1)})-f(\bar{\boldsymbol{w}}^{(t)})\leq -\eta_t \big\langle\nabla f(\bar{\boldsymbol{w}}^{(t)}),\boldsymbol{d}^{(t)}\big\rangle+\frac{\eta_t^2 L}{2}\|\boldsymbol{d}^{(t)}\|^2\label{eq:decent-smoothe}
\end{align}
that will be used frequently in our proofs.


\section{Proof of Theorem~\ref{thm:gd}}
\label{sec:app:proof:gd}
\textcolor{black}{In this section we prove the convergence of \texttt{LFGD} algorithm. Towards this end,  recalling the notation ${{\mathbf{g}}}^{(t)}\triangleq \frac{1}{K}\sum_{j\in \mathcal{P}_t} \mathbf{g}_j^{(t)}$, and following the $L$-smoothness gradient assumption on global objective, by using  $\mathbf{g}^{(t)}$} in inequality (\ref{eq:decent-smoothe}) we have:
\begin{align}
    f(\bar{\boldsymbol{w}}^{(t+1)})-f(\bar{\boldsymbol{w}}^{(t)})\leq -\eta_t \big\langle\nabla f(\bar{\boldsymbol{w}}^{(t)}),\mathbf{g}^{(t)}\big\rangle+\frac{\eta_t^2 L}{2}\|{\mathbf{g}^{(t)}}\|^2\label{eq:Lipschitz-c1}
\end{align}
By taking expectation on both sides of above inequality over sampling of devices $\mathcal{P}_t$, we get:
\begin{align}
    \mathbb{E}_{\mathcal{P}_t}\Big[f(\bar{\boldsymbol{w}}^{(t+1)})-f(\bar{\boldsymbol{w}}^{(t)})\Big]\leq -\eta_t\mathbb{E}_{\mathcal{P}_t}\Big[ \big\langle\nabla f(\bar{\boldsymbol{w}}^{(t)}),\mathbf{g}^{(t)}\big\rangle\Big]+\frac{\eta_t^2 L}{2}\mathbb{E}_{\mathcal{P}_t}\Big[\|{\mathbf{g}^{(t)}}\|^2\Big]\label{eq:Lipschitz-c-gd}
\end{align}
We note that for the convergence of \texttt{LFGD} we do not have mini-batch sampling as local machines use the full gradient over their local data, therefore,  the expectation is only taken with respect to the randomness is selection of $K$ devices with sampling probabilities $q_1, q_2, \ldots, q_p$. To simplify the notation, in what follows we drop $\mathcal{P}_t$ and simply use $\mathbb{E}[\cdot]$ to denote expectation with respect to this randomness.

The following lemma bounds the second term in right hand side of~(\ref{eq:Lipschitz-c-gd}) by relating the averaged gradient over sampled machines to the full gradient of individual local data shards.
\begin{lemma}\label{lemma:tasbih1-gd}
For the local federated GD algorithm~(\texttt{LFGD}), we have the following bound: 
\begin{align}
    \mathbb{E}\left[\|{\mathbf{g}}^{(t)}\|^2\right]&\leq    {\sum_{j=1}^pq_j\|\nabla{f}_j(\boldsymbol{w}_j^{(t)})\|^2}\nonumber\\
    &\leq \lambda\|{\sum_{j=1}^pq_j\nabla{f}_j(\boldsymbol{w}_j^{(t)})\|^2} \label{eq:lemma1}
\end{align}

\end{lemma}
The following lemma upper bounds the first term in right-hand side of~(\ref{eq:Lipschitz-c-gd}). 
\begin{corollary}\label{corol:pl-full-gd}Under Assumptions~\ref{Ass:1} and~\ref{Ass:3}, and according to the Algorithm~\ref{Alg:one-shot-using data samoples} we have:

  \begin{align}
    -\eta \mathbb{E}\Big[\big[\langle\nabla f(\bar{\boldsymbol{w}}^{(t)}),{{\mathbf{g}}}^{(t)}\rangle\big]\Big]&\leq -\mu \eta_t(f(\bar{\boldsymbol{w}}^{(t)})-f^*)-\frac{\eta_t}{2}\|\sum_{j=1}^pq
   _j\nabla{f}_j({\boldsymbol{w}}^{(t)}_{j})\|^2+\frac{\eta_t L^2}{2}\sum_{j=1}^pq_j\|\bar{\boldsymbol{w}}^{(t)}-\boldsymbol{w}_j^{(t)}\|^2,
\end{align}
\end{corollary}

The next lemma shows that on average how much the local solutions deviate from the average solution. Recall that the average solution  is calculated periodically after every $E$ local iterations and per iteration virtual average $\bar{\boldsymbol{w}}^{(t)}$ is introduced for analysis purposes. 
\begin{lemma}\label{lemma:dif-under-pl-gd}
 For \texttt{LFGD} algorithm, we have the following bound on the difference of virtual averaged solution and the individual local solutions:
\begin{align}
         \mathbb{E}\left[\sum_{j=1}^pq_j\|\bar{\boldsymbol{w}}^{(t)}-\boldsymbol{w}_j^{(t)}\|^2\right]&\leq 4E\sum_{k=t_c + 1}^{t-1}\eta_k^2\sum_{j=1}^pq_j\|\
\nabla{f}_j(\boldsymbol{w}_j^{(k)})\|^2,\nonumber\\
&\leq4E\lambda\sum_{k=t_c + 1}^{t-1}\eta_k^2\|\sum_{j=1}^pq_j\
\nabla{f}_j(\boldsymbol{w}_j^{(k)})\|^2
\end{align}
where $t_c\triangleq \floor{\frac{t}{E}}E$ denotes the most recent communication round,  and expectation $\mathbb{E}[\cdot]$ is taken with respect to sampling of devices at each communication round.
\end{lemma}
Note that this lemma implies that the term $\sum_{j=1}^pq_j\|\bar{\boldsymbol{w}}^{(t)}-\boldsymbol{w}_j^{(t)}\|^2$ only depends on the iterations $t_c+1$ through $t-1$, \textcolor{black}{thanks to restarting property of local GD algorithm}.


By plugging back all the above lemmas and corollary into (\ref{eq:Lipschitz-c-gd}), and considering a fixed learning rate $\eta_1 = \ldots = \eta_T = \eta$ for all iterations we get: 
\begin{align}
\mathbb{E}\left[f(\bar{\boldsymbol{w}}^{(t+1)})\right]-f^{*}&\leq (1-\mu\eta)\mathbb{E}\left[f(\bar{\boldsymbol{w}}^{(t)})-f^*\right] + \frac{\eta}{2}\left(-1+L\lambda\eta\right)\|\sum_{j=1}^pq_j\
\nabla{f}_j(\boldsymbol{w}_j^{(t)})\|^2\nonumber\\
    &\quad+\frac{\eta L^2}{2}\left[4E\lambda\sum_{k=t_c+1}^{t-1}\eta^2\|\sum_{j=1}^pq_j\
\nabla{f}_j(\boldsymbol{w}_j^{(k)})\|^2\right]\nonumber\\
&\stackrel{\text{\ding{192}}}{=}\Delta\mathbb{E}\left[f(\bar{\boldsymbol{w}}^{(t)})-f^*\right]+\frac{\eta}{2}\left(-1+L\lambda\eta\right)\|\sum_{j=1}^pq_j\
\nabla{f}_j(\boldsymbol{w}_j^{(t)})\|^2 \nonumber\\
&\quad+{B}\sum_{k=t_c+1}^{t-1}\eta^2\|\sum_{j=1}^pq_j\
\nabla{f}_j(\boldsymbol{w}_j^{(k)})\|^2,\label{eq:final-ineq-to-sum-gd}
\end{align}
where in \ding{192} we use the following abbreviations: 
\begin{align}
  \Delta&\triangleq 1-\mu\eta\\
   {B}&\triangleq 2\lambda\eta L^2E,\label{eq:middle-step}
\end{align}


In the following lemma, we show that with proper choice of learning rate the negative coefficient of the $\|\sum_{j=1}^pq_j\nabla{f}_j(\boldsymbol{w}_j^{(t)})\|_2^2$ can be dominant at each local computation period. Thus, we can remove the terms including $\|\sum_{j=1}^pq_j\nabla{f}_j(\boldsymbol{w}_j^{(t)})\|_2^2$ from the bound in (\ref{eq:final-ineq-to-sum-gd}) and obtain the desired convergence rate. To this end, we first derive a condition on learning rate that the terms including $\|\sum_{j=1}^pq_j\nabla{f}_j(\boldsymbol{w}_j^{(t)})\|_2^2$ vanish, and then state proper choices for learning rate as a function of number of local updates $E$ that guarantees the desired condition on learning rate to hold. 
\begin{lemma}\label{lemma:simplification-constant}
  Let $t_c\triangleq \floor{\frac{t}{E}}E,\{a_t\}_{t}, a_{t}\geq 0,\{e_t\}_{t}, e_{t}\geq 0$ be sequences satisfying 
  \begin{align}
      a_{t+1}\leq (1-\mu\eta)a_t+\frac{\eta}{2}\Big(-1+L\lambda\eta\Big)e_t+B\sum_{k=t_c+1}^{t-1}\eta^2e_k\label{eq:tasbih3}
  \end{align}
Then,  if
\begin{align}
    \eta\leq \frac{1}{L\lambda+\frac{2B}{\Delta^{E-1}}\left[\frac{1-\Delta^{E-1}}{1-\Delta}\right]}\label{eq:lr-c-gd}
\end{align}  
holds for $\Delta=1-\mu \eta$ and $B= 2\eta L^2 E$,  the recursive relation in~(\ref{eq:tasbih3}) reduces to
$$a_{t+1}\leq (1-\mu\eta)a_t.$$
\end{lemma}

An immediate implication of Lemma~\ref{lemma:simplification-constant}, with $a_t = \mathbb{E}\left[f(\bar{\boldsymbol{w}}^{(t)})-f^*\right]$ and $e_k = \|\sum_{j=1}^pq_j
\nabla{f}_j(\boldsymbol{w}_j^{(k)})\|^2$, is that the inequality in~(\ref{eq:final-ineq-to-sum-gd}) can be simplified as follows:
\begin{align}
    \mathbb{E}\left[f(\bar{\boldsymbol{w}}^{(t+1)})\right]-f^*&\leq \Delta\mathbb{E}\left[f(\bar{\boldsymbol{w}}^{(t)})-f^*\right]
\end{align}
Summing up the above inequality over $t = 1, 2, \ldots, T$  gives:
\begin{align}
     \mathbb{E}\left[f(\bar{\boldsymbol{w}}^{(T)})\right]-f^*&\leq (1-\mu\eta)^T\mathbb{E}\left[f(\bar{\boldsymbol{w}}^{(0)})-f^*\right]
\end{align}
as stated in the theorem.

To complete the proof, we only left with showing that the condition in~(\ref{eq:lr-c-gd}) can be satisfied. Indeed, the following lemma shows that  if the learning rate is chosen properly based on number of local updates $E$ the condition holds.
  \begin{lemma}\label{lemma:2nd-cond-gd}
If learning rate $\eta$ and the number of local updates $E$ satisfy the condition
  \begin{align}
      \eta\left(L+\frac{4L^2E}{\mu {(1-\mu\eta)}^{E-1}}\right)\leq \frac{1}{\lambda}
  \end{align}
the condition~(\ref{eq:lr-c-gd}) is implied.

  \end{lemma}

\section{Proof of Theorem~\ref{thm:proof-undrr-pl}}
\label{sec:app:sgd:undrr-pl}
In this section we prove the convergence rate of  \texttt{LFSGD} algorithm with stochastic local mini-batch gradients for non-convex objectives satisfying the PL condition. But first,  we state a sequence of key  lemmas that will be used as  the building blocks of  convergence proof. We start with a basic lemma  which forms the main ground for proof. 
\begin{lemma}\label{lemm:-1}
Under Assumption~\ref{Ass:1} we have: 
\begin{align}
    \mathbb{E}\Big[\mathbb{E}_{\mathcal{P}_t}\big[f(\bar{\boldsymbol{w}}^{(t+1)})-f(\bar{\boldsymbol{w}}^{(t)})\big]\Big]\leq -\eta_t \mathbb{E}\Big[\mathbb{E}_{\mathcal{P}_t}\big[\big\langle\nabla f(\bar{\boldsymbol{w}}^{(t)}),{\tilde{\mathbf{g}}}^{(t)}\big\rangle\big]\Big]+\frac{\eta_t^2 L}{2}\mathbb{E}\Big[\mathbb{E}_{\mathcal{P}_t}\big[\|{{\tilde{\mathbf{g}}}^{(t)}}\|^2\big]\Big]\label{eq:Lipschitz-c0}
\end{align}
\end{lemma}
\textcolor{black}{
The proof is straightforward, but for the sake of completeness we include it here. 
\begin{proof}
From the smoothness assumption in Eq.~(\ref{eq:decent-smoothe}) by setting  $\boldsymbol{d}^{(t)}=\tilde{\mathbf{g}}^{(t)}$, we have
\begin{align}
    f(\bar{\boldsymbol{w}}^{(t+1)})-f(\bar{\boldsymbol{w}}^{(t)})\leq -\eta_t \big\langle\nabla f(\bar{\boldsymbol{w}}^{(t)}),{\tilde{\mathbf{g}}}^{(t)}\big\rangle+\frac{\eta_t^2 L}{2}\|{{\tilde{\mathbf{g}}}^{(t)}}\|^2
\end{align}
By taking expectation, first, with respect to the chosen devices ($\mathcal{P}_t$) and then with respect to randomness in i.i.d \textit{local} mini-batch samples ($\{\xi_1,\ldots,\xi_p \}|\boldsymbol{w}_1,\ldots,\boldsymbol{w}_p$), the proof is concluded. Note that the order of taking expectation follows from the fact that devices are chosen first and thereafter  the  stochastic mini-batch gradients are computed and noting the fact that devices are \emph{agnostic} to the random selection  at every communication round.
\end{proof}
}


The second term in right hand side of (\ref{eq:Lipschitz-c0}) is upper-bounded by the following lemma.
\begin{lemma}\label{lemma:tasbih1}
Under Assumption \ref{Ass:2} and our sampling scheme in Algorithm~\ref{Alg:one-shot-using data samoples}, we have the following bound 
\begin{align}
    \mathbb{E}\Big[\mathbb{E}_{\mathcal{P}_t}\big[\|\tilde{\mathbf{g}}^{(t)}\|^2\big]\Big]&\leq \Big(\frac{C_1}{K}+1\Big)    \Big[{\sum_{j=1}^pq_j\|\nabla{f}_j(\boldsymbol{w}_j^{(t)})\|^2}\Big]+\frac{\sigma^2}{K B}\nonumber\\
    &\leq \lambda\Big(\frac{C_1}{K}+1\Big)    \Big[{\|\sum_{j=1}^pq_j\nabla{f}_j(\boldsymbol{w}_j^{(t)})\|^2}\Big]+\frac{\sigma^2}{K B}\label{eq:lemma1}
\end{align}

\end{lemma}

The first term in right-hand side of (\ref{eq:Lipschitz-c0}) is bounded with following lemma.
\begin{lemma}\label{lemma:cross-inner-bound-unbiased}
  Under Assumptions \ref{Ass:1}, and according to the  Algorithm \ref{Alg:one-shot-using data samoples} the expected inner product between stochastic gradient and full batch gradient can be bounded with:

\begin{align}
    -\eta_t \mathbb{E}\Big[\mathbb{E}_{\mathcal{P}_t}\big[\langle\nabla f(\bar{\boldsymbol{w}}^{(t)}),{{\tilde{\mathbf{g}}}^{(t)}}\rangle\big]\Big]&\leq -\frac{\eta_t}{2}\|\nabla f(\bar{\boldsymbol{w}}^{(t)})\|^2-\frac{\eta_t}{2}\|\sum_{j=1}^pq
   _j\nabla{f}_j({\boldsymbol{w}}^{(t)}_{j})\|^2\Big]+\frac{\eta_t L^2}{2}\sum_{j=1}^pq_j\|\bar{\boldsymbol{w}}^{(t)}-\boldsymbol{w}_j^{(t)}\|^2\label{eq:lemma3-thm2}
\end{align}

\end{lemma}

An immediate implication of Lemma~\ref{lemma:cross-inner-bound-unbiased} is the following.
\begin{corollary}\label{corol:pl}
  Under Assumptions \ref{Ass:1} and \ref{Ass:3}, and according to the  Algorithm \ref{Alg:one-shot-using data samoples} the expected per-iteration inner product between stochastic gradient and full batch gradient can be bounded by:
  \begin{align}
    -\eta_t \mathbb{E}\Big[\mathbb{E}_{\mathcal{P}_t}\big[\langle\nabla f(\bar{\boldsymbol{w}}^{(t)}),{{\tilde{\mathbf{g}}}^{(t)}}\rangle\big]\Big]&\leq -\mu \eta_t(f(\bar{\boldsymbol{w}}^{(t)})-f^*)-\frac{\eta_t}{2}\|\sum_{j=1}^pq
   _j\nabla{f}_j({\boldsymbol{w}}^{(t)}_{j})\|^2+\frac{\eta_t L^2}{2}\sum_{j=1}^pq_j\|\bar{\boldsymbol{w}}^{(t)}-\boldsymbol{w}_j^{(t)}\|^2
\end{align}
\end{corollary}
\begin{proof}
From Lemma~\ref{lemma:cross-inner-bound-unbiased}  we have:
\begin{align}
    -\eta_t \mathbb{E}\Big[\mathbb{E}_{\mathcal{P}_t}\big[\langle\nabla f(\bar{\boldsymbol{w}}^{(t)}),{{\tilde{\mathbf{g}}}^{(t)}}\rangle\big]\Big]&\leq -\frac{\eta_t}{2}\|\nabla f(\bar{\boldsymbol{w}}^{(t)})\|^2-\frac{\eta_t}{2}\|\sum_{j=1}^pq
   _j\nabla{f}_j({\boldsymbol{w}}^{(t)}_{j})\|^2+\frac{\eta_t L^2}{2}\sum_{j=1}^pq_j\|\bar{\boldsymbol{w}}^{(t)}-\boldsymbol{w}_j^{(t)}\|^2\label{eq:lemma3}\\
    &\stackrel{\text{\ding{192}}}{\leq} -\mu\eta_t(f(\bar{\boldsymbol{w}}^{(t)})-f^*)-\frac{\eta_t}{2}\|\sum_{j=1}^pq
   _j\nabla{f}_j({\boldsymbol{w}}^{(t)}_{j})\|^2+\frac{\eta_t L^2}{2}\sum_{j=1}^pq_j\|\bar{\boldsymbol{w}}^{(t)}-\boldsymbol{w}_j^{(t)}\|^2
   \label{eqn-b4}
\end{align}
where \ding{192} follows from PL property.

\end{proof}

The following lemmas bounds the last term in~(\ref{eqn-b4}), i.e., the average  distance of local solutions from their virtual average.
\begin{lemma}\label{lemma:dif-under-pl-sgd}
Under Assumptions~\ref{Ass:2} we have:
\begin{align}
      \mathbb{E}\left[\sum_{j=1}^pq_j\|\bar{\boldsymbol{w}}^{(t)}-\boldsymbol{w}_j^{(t)}\|^2\right]&\leq2\left(\frac{K+1}{K}\right)\left([C_1+E]\sum_{k=t_c+1}^{t-1}\eta_{k}^2\sum_{j=1}^pq_j\|\
\nabla{f}_j(\boldsymbol{w}_j^{(k)})\|^2+\sum_{k=t_c+1}^{t-1}\frac{\eta_{k}^2\sigma^2}{B}\right)\nonumber\\
&\leq 2\left(\frac{K+1}{K}\right)\left(\lambda[C_1+E]\sum_{k=t_c+1}^{t-1}\eta_{k}^2\|\sum_{j=1}^pq_j\
\nabla{f}_j(\boldsymbol{w}_j^{(k)})\|^2+\sum_{k=t_c+1}^{t-1}\frac{\eta_{k}^2\sigma^2}{B}\right).
\end{align}

\end{lemma}
Again, note that this lemma implies that the term $\mathbb{E}\sum_{j=1}^p\|\bar{\boldsymbol{w}}^{(t)}-\boldsymbol{w}_j^{(t)}\|^2$ only depends on iterations $t_c+1\triangleq \floor{\frac{t}{E}}E+1$ through $t-1$.

By plugging back all the above lemmas and Lemma~\ref{lemma:tasbih1} into~(\ref{eq:Lipschitz-c0}), we get:
\begin{align}
\mathbb{E}[f(\bar{\boldsymbol{w}}^{(t+1)})]-f^{*}&\leq (1-\mu\eta_t)\mathbb{E}[f(\bar{\boldsymbol{w}}^{(t)})-f^*]+\frac{L\eta_t^2 \sigma^2}{2K B}+\frac{\eta_t L^2}{K}\left(\sum_{k=t_c+1}^{t-1}\eta_{k}^2\frac{(K+1)\sigma^2}{K B}\right)\nonumber\\
    &+\frac{\eta_t}{2}\Big[-1+\frac{L\lambda\eta_t({C_1+ K})}{K}\Big]\|\sum_{j=1}^pq_j\
\nabla{f}_j(\boldsymbol{w}_j^{(t)})\|^2\nonumber\\
    &+\frac{\eta_t L^2(K+1)}{K^2}\Big[\lambda\Big(C_1+E)\Big)\sum_{k=t_c+1}^{t-1}\eta_k^2\|\sum_{j=1}^pq_j\
\nabla{f}_j(\boldsymbol{w}_j^{(k)})\|^2\Big]\nonumber\\
&\stackrel{\text{\ding{192}}}{=}\Delta_t\mathbb{E}[f(\bar{\boldsymbol{w}}^{(t)})-f^*]+{c}_t+{\frac{\eta_t}{2}\Big[-1+\lambda L\eta_t(\frac{C_1+ K}{K})\Big]}\|\sum_{j=1}^pq_j\
\nabla{f}_j(\boldsymbol{w}_j^{(t)})\|^2\nonumber\\&+{B}_{t}\sum_{k=t_c+1}^{t-1}\eta_k^2\|\sum_{j=1}^pq_j\
\nabla{f}_j(\boldsymbol{w}_j^{(k)})\|^2,\label{eq:final-ineq-to-sum}
\end{align}
where in \ding{192} we used the following abbreviations for simplicity: 
\begin{align}
  \Delta_t&\triangleq1-\mu\eta_t\\
  {c}_t&\triangleq \frac{\eta_t L \sigma^2}{KB}\Big[\frac{\eta_t}{2}+\frac{L(K+1)}{K}\sum_{k=t_c+1}^{t-1}\eta_{k}^2\Big]\\
   {B}_{t}&\triangleq\frac{\lambda(K+1)\eta_t L^2}{K^2}\Big(C_1+E\Big),\label{eq:middle-step-pl}
\end{align}
\begin{lemma}\label{lemm:main-seq}
  Let $t_c\triangleq \floor{\frac{t}{E}}E,\{a_t\}_{t}, a_{t}\geq 0,\{e_t\}_{t}, e_{t}\geq 0,\{c_t\}_{t},$ be sequences, with constant $D$, satisfying 
\begin{align}
      a_{t+1}\leq (1-\mu\eta_t)a_t+c_t+\frac{\eta_t}{2}\big(-1+D\eta_t\big)e_t+B_t\sum_{k=t_c+1}^{t-1}\eta_k^2e_k\label{eq:tasbih30}
  \end{align}
  where $B_t\triangleq A \eta_t (E+h)$ wherein $A$ and $h$ are constants.
Then, for time period $[t_c+1,t]$ if learning rates satisfy
\begin{align}
\eta_t&\leq \frac{1}{D}\nonumber\\
\eta_{t-1}&\leq\frac{1}{D+2\frac{B_t}{\Delta_{t}}}\nonumber\\
&\vdots\nonumber\\
\eta_{t_c+1}&\leq \frac{1}{D+\frac{2}{\Pi_{i=t_c+2}^t\Delta_i}\big[\Pi_{i=t_c+3}^t\Delta_i{B}_{t_c+2}+\ldots+\Delta_t{B}_{t-1}+B_t\big]}
\end{align}
the recursive relation in~(\ref{eq:tasbih30}) reduces to
\begin{align}
      a_{t+1}\leq (1-\mu\eta)a_t+c_t
  \end{align}
  
\end{lemma}

Next, using Lemma~\ref{lemm:main-seq} with constants $D=\frac{\lambda L\left(C_1+K\right)}{K}$ and sequence ${B}_t \triangleq\frac{\lambda(K+1)\eta_t L^2(C_1+E)}{K^2}$, we conclude that for time period $[t_c,t]$ the bound in (\ref{eq:final-ineq-to-sum}) reduces to
\begin{align}
\mathbb{E}\left[f(\bar{\boldsymbol{w}}^{(t+1)})\right]-f^{*}&\leq (1-\mu\eta_t)\mathbb{E}\left[f(\bar{\boldsymbol{w}}^{(t)})-f^*\right]
\end{align}

In the following lemma-- which is essentially proven  in~\cite{haddadpour2019local} and adopted here to support random sampling of a subset of devices in  communication rounds, we show that with proper choice of learning rate this bound holds for all iterations. We include a distilled proof in Appendix~\ref{sec:appendix:omitted} for completeness.

\begin{lemma}
\label{lemmba:choice-of-learning-rate-pl}
Let $\alpha$ be a positive constant that satisfies $\alpha\exp{(-\frac{2}{\alpha})}<\kappa\sqrt{192\lambda\left(\frac{K+1}{K}\right)}$ and $a=\alpha E+4$. \textcolor{black}{Initializing all local model parameters at the same point $\bar{\boldsymbol{w}}^{(0)}$, {\textcolor{black}{for $E$ sufficiently large to ensure that $4{(a-3)}^{E-1}L(C_1+K)\leq \frac{64L^2 (K+1)}{\mu K}(E-1)E(a+1)^{E-2}$, $\frac{32L^2}{\mu}C_1(E-1)(a+1)^{E-2}\leq \frac{64L^2}{\mu}(E-1)E(a+1)^{E-2}$}} and }
\textcolor{black}{\begin{align}
    E\geq 1+\frac{\alpha^2+6\alpha}{\lambda\left(\frac{K+1}{K}\right)192\kappa^2e^{\frac{4}{\alpha}}-\alpha^2}+\frac{\sqrt{5}}{\sqrt{\lambda\left(\frac{K+1}{K}\right)192\kappa^2e^{\frac{4}{\alpha}}-\alpha^2}}\nonumber
\end{align}}
then, under Assumptions~\ref{Ass:1} to~\ref{Ass:2}, if we choose the learning rate as $\eta_t=\frac{4}{\mu(t+a)}$ inequality (\ref{eq:final-ineq-to-sum}) reduces to
\begin{align}
    \mathbb{E}\left[f(\bar{\boldsymbol{w}}^{(t+1)})\right]-f^*&\leq \Delta_t\mathbb{E}\left[f(\bar{\boldsymbol{w}}^{(t)})-f^*\right]+{c}_t\label{eq:denoised-iteratt}
\end{align}
for all iterations.
\end{lemma}
We conclude the proof of Theorem~\ref{thm:proof-undrr-pl} with the following lemma:

\begin{lemma}\label{lemma:tasbihb44}
For the learning rate as given in Lemma~\ref{lemmba:choice-of-learning-rate-pl},  iterating over (\ref{eq:denoised-iteratt}) leads to the following bound:
\begin{align}
     \mathbb{E}[f(\bar{\boldsymbol{w}}^{(T)})-f^*]
&\leq \frac{a^3}{(T+a)^3}\mathbb{E}\big[f(\bar{\boldsymbol{w}}^{(0)})-f^*\big]+ \frac{4\kappa \sigma^2T(T+2a)}{\mu KB(T+a)^3}+\frac{256\kappa^2\sigma^2T(E-1)}{\mu KB(T+a)^3}
\end{align}
  
\end{lemma}

\section{Proof of Theorem \ref{thm:FedAvg}}
\label{sec:app:localsgd:general}
In this section we obtain the convergence rate of \texttt{LFSGD} for general non-convex objectives. The proof is based on Assumption~\ref{Ass:1}, by using $\tilde{\mathbf{g}}^{(t)}$ in (\ref{eq:decent-smoothe}) and taking expectation over random selection of devices that results in
\begin{align}
    \mathbb{E}\Big[\mathbb{E}_{ \mathcal{P}_t}[f(\bar{\boldsymbol{w}}^{(t+1)})]-f(\bar{\boldsymbol{w}}^{(t)})\Big]\leq -\eta \mathbb{E}\Big[\mathbb{E}_{ \mathcal{P}_t}\big[\big\langle\nabla f(\bar{\boldsymbol{w}}^{(t)}),{\tilde{\mathbf{g}}}^{(t)}\big\rangle\big]\Big]+\frac{\eta^2 L}{2}\mathbb{E}\Big[\mathbb{E}_{ \mathcal{P}_t}\Big[\|{{\tilde{\mathbf{g}}}^{(t)}}\|^2\Big]\Big]\label{eq:Lipschitz-c}
\end{align}
By taking the average of above inequality for all iterations $t=1,2, \ldots, T$, we get
\begin{align}
    \frac{1}{T} \sum_{t=0}^{T-1}\mathbb{E}\left[\mathbb{E}_{ \mathcal{P}_t}[f(\bar{\boldsymbol{w}}^{(t+1)})]-f(\bar{\boldsymbol{w}}^{(t)})\right]&\leq \frac{1}{T} \sum_{t=0}^{T-1}\left(-\eta \mathbb{E}\left[\mathbb{E}_{ \mathcal{P}_t}\big[\big\langle\nabla f(\bar{\boldsymbol{w}}^{(t)}),{{\tilde{\mathbf{g}}}^{(t)}}\big\rangle\big]\right]\right)+\frac{1}{T} \sum_{t=0}^{T-1}\frac{\eta^2 L}{2}\mathbb{E}\left[\mathbb{E}_{ \mathcal{P}_t}\Big[\|{{{\tilde{\mathbf{g}}}^{(t)}}}\|^2\Big]\right]\label{eq:avg-Lips}
\end{align}

To bound the first term in right hand side of~(\ref{eq:avg-Lips}), we need the following result which is specialization of Lemma~\ref{lemma:dif-under-pl-sgd} to over entire iterations.
\begin{lemma}\label{lemma:dif-b-u-x}
Under Assumptions~\ref{Ass:2} we have:
\begin{align}
    \frac{1}{T}\sum_{t=0}^{T-1}\sum_{j=1}^pq_j\mathbb{E}\left[\mathbb{E}_{\mathcal{P}_{t}}\|\bar{\boldsymbol{w}}^{(t)}-\boldsymbol{w}_j^{(t)}\|\right]&\leq  \frac{\big(2C_1+E(E+1)\big)}{T}\eta^2\left(\frac{K+1}{K}\right)\sum_{t=0}^{T-1}\sum_{j=1}^pq_j\|\mathbf{g}_{j}^{(t)}\|^2+\frac{\eta^2(K+1)(E+1)\sigma^2}{KB}\nonumber\\
    &\leq \frac{\lambda \eta^2\big(2C_1+E(E+1)\big)}{T}\left(\frac{K+1}{K}\right)\sum_{t=0}^{T-1}\|\sum_{j=1}^pq_j\mathbf{g}_{j}^{(t)}\|^2+\frac{\eta^2(K+1)(E+1)\sigma^2}{KB}
\end{align}

\end{lemma}

We continue the proof  by utilizing Lemmas~\ref{lemma:tasbih1} and~\ref{lemma:dif-b-u-x} to  further upper bound~(\ref{eq:avg-Lips}) as follows:

\begin{align}
     \frac{1}{T} \sum_{t=0}^{T-1}\mathbb{E}\left[\mathbb{E}_{ \mathcal{P}_t}[f(\bar{\boldsymbol{w}}^{(t+1)})]-f(\bar{\boldsymbol{w}}^{(t)})\right]&\leq \frac{1}{T} \sum_{t=0}^{T-1}\left(-\eta \mathbb{E}\left[\mathbb{E}_{ \mathcal{P}_t}\big[\big\langle\nabla f(\bar{\boldsymbol{w}}^{(t)}),{{\tilde{\mathbf{g}}}^{(t)}}\big\rangle\big]\right]\right)+\frac{1}{T} \sum_{t=0}^{T-1}\frac{\eta^2 L}{2}\mathbb{E}\left[\mathbb{E}_{ \mathcal{P}_t}\left[\|{{{\tilde{\mathbf{g}}}^{(t)}}}\|^2\right]\right]\nonumber\\
&\leq \frac{1}{T} \sum_{t=0}^{T-1}\left(-\frac{\eta}{2}\|\nabla f(\bar{\boldsymbol{w}}^{(t)})\|^2-\frac{\eta}{2}\|\sum_{j=1}^pq
   _j\nabla{f}_j({\boldsymbol{w}}^{(t)}_{j})\|^2\right)\nonumber\\
   &+\frac{\lambda\eta L^2}{2T}\left(\frac{K+1}{K}\right)\left([2C_1+E(E+1)]\eta^2\sum_{t=0}^{T-1}\|\sum_{j=1}^pq_j\nabla{f}_j(\boldsymbol{w}_j^{(t)})\|^2\right)\nonumber\\
&\qquad+\frac{\eta L^2}{2T}\left(\frac{K+1}{K}\right)\left(\frac{T(E+1)\eta^2\sigma^2}{B}\right)\nonumber\\
&+\frac{1}{T}\sum_{t=0}^{T-1}\frac{ L\eta^2}{2}\left(\lambda\left(\frac{C_1}{K}+1\right)    \left[{\|\sum_{j=1}^pq_j\nabla{f}_j(\boldsymbol{w}_j^{(t)})\|^2}\right]+\frac{\sigma^2}{K B}\right)\nonumber\\
&=\frac{1}{T} \sum_{t=0}^{T-1}\left(-\frac{\eta}{2}
   \|\nabla f(\bar{\boldsymbol{w}}^{(t)})\|^2-\frac{\eta}{2}\|\sum_{j=1}^pq
   _j\nabla{f}_j({\boldsymbol{w}}^{(t)}_{j})\|^2\right)\nonumber\\
   &+\frac{\eta L^2}{2T}\left(\frac{K+1}{K}\right)\left(\lambda\left[2C_1+E(E+1)\right]\eta^2\sum_{t=0}^{T-1}\|\sum_{j=1}^pq_j
\nabla{f}_j(\boldsymbol{w}_j^{(k)})\|^2+\frac{T(E+1)\eta^2\sigma^2}{B}\right)\nonumber\\
&+\frac{1}{T}\sum_{t=0}^{T-1}\frac{\lambda L\eta^2}{2}\left(\frac{C_1}{K}+1\right)    \left[\|\sum_{j=1}^pq_j\nabla{f}_j(\boldsymbol{w}_j^{(t)})\|^2\right]+\frac{L\eta^2}{2}\frac{\sigma^2}{K B}
\label{eq:midst0}
\end{align}
Now from (\ref{eq:midst0}) we have:
\begin{align}
\frac{1}{T} \sum_{t=0}^{T-1}\Big[\mathbb{E}[f(\bar{\boldsymbol{w}}^{(t+1)})]-f(\bar{\boldsymbol{w}}^{(t)})\Big]&\leq -\frac{1}{T}\sum_{t=0}^{T-1}\frac{\eta}{2}\|\nabla f(\bar{\boldsymbol{w}}^{(t)})\|^2\nonumber\\
   &+ \frac{1}{T}\sum_{t=0}^{T-1}\left[-\frac{\eta}{2}+\frac{\lambda(K+1)L^2\eta^3[2C_1+E(E+1)]}{2K}+\frac{\lambda L\eta^2}{2}\left(\frac{C_1}{K}+1\right)\right]\nonumber\\
&\times\Big[{\|\sum_{j=1}^pq_j\nabla{f}_j(\boldsymbol{w}_j^{(t)})\|^2}\Big]+ \frac{\eta^3L^2(E+1)\sigma^2}{B}\left(\frac{K+1}{K}\right)+\frac{L\eta^2}{2}\frac{\sigma^2}{K B}\nonumber\\
&\stackrel{\text{\ding{192}}}{\leq}-\frac{1}{T}\sum_{t=0}^{T-1}\frac{\eta}{2}\|\nabla f(\bar{\boldsymbol{w}}^{(t)})\|^2+ \frac{\eta^3L^2(E+1)\sigma^2}{B}\left(\frac{K+1}{K}\right)+\frac{L\eta^2}{2}\frac{\sigma^2}{K B}\label{eq:finstep}
\end{align}
where \ding{192} follows if the following  condition holds:
\begin{align}
    -\frac{\eta}{2}+\frac{\lambda(K+1)L^2\eta^3[2C_1+E(E+1)]}{2K}+\frac{\lambda L\eta^2}{2}\left(\frac{C_1}{K}+1\right)\leq 0\label{eq:cond-l-r-tbp}
\end{align}
By rearranging (\ref{eq:finstep}) we get:
\begin{align}
    \frac{1}{T}\sum_{t=0}^{T-1}\mathbb{E}\|\nabla f(\bar{\boldsymbol{w}}^{(t)})\|^2\leq \frac{2\big[f(\bar{\boldsymbol{w}}^{(1)})-f^*\big]}{\eta T}+\frac{L\eta\sigma^2}{K B}+\frac{2\eta^2\sigma^2L^2(E+1)}{B}\left(\frac{K+1}{K}\right)\label{eq:to-prove-informal}
\end{align}

\subsection{Proof of (Informal) Theorem~\ref{thm:FedAvg_informal}}
By plugging $\eta=\frac{1}{L}\sqrt{\frac{K}{T}}$ in (\ref{eq:to-prove-informal}) we get:
\begin{align}
     \frac{2\big[f(\bar{\boldsymbol{w}}^{(0)})-f^*\big]}{\eta T}+\frac{L\eta\sigma^2}{K B}+\frac{2\eta^2\sigma^2L^2(E+1)}{B}\left(\frac{K+1}{K}\right)&=O\left(\frac{1}{\sqrt{KT}}\right)+\frac{2\sigma^2K}{B T}(E+1)\left(\frac{K+1}{K}\right)\nonumber\\
     &\stackrel{\text{\ding{192}}}{=}O\left(\frac{1}{\sqrt{KT}}\right)+\frac{2\sigma^2K}{B T}\frac{\sqrt{T}}{K^{1.5}}\left(\frac{K+1}{K}\right)\nonumber\\
     &=O\left(\frac{1}{\sqrt{KT}}\right)
\end{align}
where \ding{192} follows from $E+1=O\left(\frac{\sqrt{T}}{K^{1.5}}\right)$. Next step is to ensure that this choice of $E$ does not violate the condition over learning rate in (\ref{eq:cond-l-r-tbp}). To this end, we derive the condition over $E+1$. We can further simplify  (\ref{eq:cond-l-r-tbp}) with plugging the value of learning rate $\eta=\frac{1}{L}\sqrt{\frac{K}{T}}$  as follows:
\begin{align}
    E(E+1)\leq(E+1)^2\leq \left[\frac{T}{K}\left(\frac{K}{\lambda(K+1)}-\sqrt{\frac{K}{T}}(\frac{K+C_1}{K+1})\right)-2C_1\right]=O\left(\frac{T}{\lambda K}\right)
\end{align}
and if ${(E+1)}^2=O\left(\frac{{T}}{K^{3}}\right)\leq O\left(\frac{T}{\lambda K}\right)$, or equivalently $K=\Omega\left(\sqrt{\lambda}\right) $ both conditions can be satisfied simultaneously.

\section{Proof of Theorem~\ref{thm:serverless-fed}}
\label{sec:app:serverless-fed}
For this part of the paper, we will use the following short-hand notations for ease of exposition:

\begin{align}
    \bar{\boldsymbol{w}}^{(t)}  \triangleq\sum_{j=1}^pq_j{\boldsymbol{w}}^{(t)}_j,\quad
   {{\tilde{\mathbf{g}}}^{(t)}}\triangleq \sum_{j=1}^pq_j {{\tilde{\mathbf{g}}}^{(t)}}_j
\end{align}

We note that one distinguishing feature of our proof compared to~\cite{wang2018cooperative}, is that how we define auxiliary variables $\bar{\boldsymbol{w}}^{(t)} $ and  ${{\tilde{\mathbf{g}}}^{(t)}}$ as these terms are adaptive to various data samples at each device.

Now using the fact that $\mathbf{W}\mathbf{1}^{p\times 1}=\mathbf{1}^{p\times 1}$, where $\mathbf{1}^{p\times 1} = [1, \ldots, 1]^{\top} \in \mathbb{R}^p$, we have the following relation for the auxiliary variable 
\begin{align}    \bar{\boldsymbol{w}}^{(t+1)}=\bar{\boldsymbol{w}}^{(t)}-\eta {\tilde{\mathbf{g}}^{(t)}}\label{eq:ax-net},
\end{align}

To see that form the following matrices:
\begin{align}
    \mathbf{W}^{(t)}&=\begin{bmatrix}
    q_1\boldsymbol{w}_1^{(t)} & \ldots & q_p\boldsymbol{w}_p^{(t)}
    \end{bmatrix}\nonumber\\
    \tilde{\mathbf{G}}^{(t)}&=\begin{bmatrix}
    q_1\tilde{\mathbf{g}}_1^{(t)} & \ldots & q_p\tilde{\mathbf{g}}_p^{(t)}
    \end{bmatrix}\nonumber\\
    \mathbf{P}^{(t)}&=\left\{\begin{array}{cc}
     \mathbf{W}    & E|t \\
    \mathbf{I}_{p\times p}     & \text{otherwise.} 
    \end{array}\right.
\end{align}
the updating  can be equivalently written as $$\mathbf{W}^{(t+1)}=\left[\mathbf{W}^{(t)}-\eta\tilde{\mathbf{G}}^{(t)}\right]\mathbf{P}^{(t)}$$ 
Then, multiplying both sides with $\mathbf{1}^{p\times 1}$ and using the Assumption~\ref{Assum:4}, we get 
$$\mathbf{W}^{(t+1)}\mathbf{1}^{p\times 1}=\left[\mathbf{W}^{(t)}-\eta\tilde{\mathbf{G}}^{(t)}\right]\mathbf{P}^{(t)}\mathbf{1}^{p\times 1}=\left[\mathbf{W}^{(t)}-\eta\tilde{\mathbf{G}}^{(t)}\right]\mathbf{1}^{p\times 1}$$
which leads to Eq.~(\ref{eq:ax-net}).

From the $L$-smoothness gradient assumption on the objective we have:
\begin{align}
    f(\bar{\boldsymbol{w}}^{(t+1)})-f(\bar{\boldsymbol{w}}^{(t)})\leq -\eta_t \big\langle\nabla f(\bar{\boldsymbol{w}}^{(t)}),{\tilde{\mathbf{g}}}^{(t)}\big\rangle+\frac{\eta_t^2 L}{2}\|{{\tilde{\mathbf{g}}}^{(t)}}\|^2\label{eq:Lipschitz-c}
\end{align}
By taking the expectation on both sides of above inequality and summing up for all iterations, we get
\begin{align}
    \frac{1}{T} \sum_{t=0}^{T-1}\mathbb{E}\left[f(\bar{\boldsymbol{w}}^{(t+1)})]-f(\bar{\boldsymbol{w}}^{(t)})\right]&\leq \frac{1}{T} \sum_{t=0}^{T-1}-\eta \mathbb{E}\left[\big\langle\nabla f(\bar{\boldsymbol{w}}^{(t)}),{{\tilde{\mathbf{g}}}^{(t)}}\big\rangle\right]+\frac{1}{T} \sum_{t=0}^{T-1}\frac{\eta^2 L}{2}\mathbb{E}\left[\|{{{\tilde{\mathbf{g}}}^{(t)}}}\|^2\right]\label{eq:avg-Lips-decent}
\end{align}

\begin{lemma}\label{lemma:decent_tasbih1}
Under Assumptions~\ref{Ass:2}, we have the following bound: 
\begin{align}
\mathbb{E}\big[\|\tilde{\mathbf{g}}^{(t)}\|^2\big]&\leq \Big(\frac{C_1}{p}+1\Big)    \Big[{\sum_{j=1}^pq_j\|\nabla{f}_j(\boldsymbol{w}_j^{(t)})\|^2}\Big]+\frac{\sigma^2}{p B}\nonumber\\
&\leq \Big(\frac{C_1}{p}+1\Big)    \Big[{\lambda\|\sum_{j=1}^pq_j\nabla{f}_j(\boldsymbol{w}_j^{(t)})\|^2}\Big]+\frac{\sigma^2}{p B}\label{eq:lemma1}
\end{align}

\end{lemma}

\begin{lemma}\label{lemma:dif-b-u-x-decent}
Under Assumptions \ref{Ass:1}, \ref{Ass:2} and \ref{Assum:4}, we have:
\begin{align}
    \frac{1}{T}\sum_{t=0}^{T-1}\sum_{j=1}^pq_j\mathbb{E}\left[\|\bar{\boldsymbol{w}}^{(t)}-\boldsymbol{w}_j^{(t)}\|\right]&\leq \frac{\eta^2\sigma^2}{B}\left(\frac{1+\zeta^2}{1-\zeta^2}E-1\right)+\frac{2\eta^2C_1 E}{1-\zeta^2}\frac{1}{T}\sum_{t=0}^{T-1}\sum_{j=1}^pq_j\|\nabla{f}_j(\boldsymbol{w}_j^{(t)})\|^2\nonumber\\
    &\quad+\frac{\eta^2E^2}{1-\zeta}\left(\frac{2\zeta^2}{1+\zeta}+\frac{2\zeta}{1-\zeta}+\frac{E-1}{E}\right)\frac{1}{T}\sum_{t=0}^{T-1}\sum_{j=1}^pq_j\|\nabla{f}_j(\boldsymbol{w}_j^{(t)})\|^2,\nonumber\\
    &\leq \frac{\eta^2\sigma^2}{B}\left(\frac{1+\zeta^2}{1-\zeta^2}E-1\right)+\frac{2\lambda\eta^2C_1 E}{1-\zeta^2}\frac{1}{T}\sum_{t=0}^{T-1}\|\sum_{j=1}^pq_j\nabla{f}_j(\boldsymbol{w}_j^{(t)})\|^2\nonumber\\
    &\quad+\frac{\lambda\eta^2E^2}{1-\zeta}\left(\frac{2\zeta^2}{1+\zeta}+\frac{2\zeta}{1-\zeta}+\frac{E-1}{E}\right)\frac{1}{T}\sum_{t=0}^{T-1}\|\sum_{j=1}^pq_j\nabla{f}_j(\boldsymbol{w}_j^{(t)})\|^2
\end{align}
\end{lemma}
where $\zeta$ is the second largest eigenvalue of mixing matrix $\mathbf{W}$.

To prove Lemma~\ref{lemma:dif-b-u-x-decent} we need the following result from~\cite{wang2018cooperative} (see the of Theorem~1 and its proof in appendix of~\cite{wang2018cooperative})  for the special case of $q_j=\frac{1}{p}$: 
\begin{align}
    \frac{1}{T}\sum_{t=0}^{T-1}\frac{1}{p}\sum_{j=1}^p\mathbb{E}\left[\|\bar{\boldsymbol{w}}^{(t)}-\boldsymbol{w}_j^{(t)}\|\right]&\leq \frac{\eta^2\sigma^2}{B}\left(\frac{1+\zeta^2}{1-\zeta^2}E-1\right)+\frac{2\eta^2C_1 E}{1-\zeta^2}\frac{1}{T}\sum_{t=0}^{T-1}\frac{1}{p}\sum_{j=1}^p\|\nabla{f}(\boldsymbol{w}_j^{(t)})\|^2\nonumber\\
    &\quad+\frac{\eta^2E^2}{1-\zeta}\left(\frac{2\zeta^2}{1+\zeta}+\frac{2\zeta}{1-\zeta}+\frac{E-1}{E}\right)\frac{1}{T}\sum_{t=0}^{T-1}\frac{1}{p}\sum_{j=1}^p\|\nabla{f}(\boldsymbol{w}_j^{(t)})\|^2.\label{eq:Gauri}
\end{align}
With an application of Lemma~\ref{lemma:tasbih1}, the above inequality  can be easily generalized to show Lemma~\ref{lemma:dif-b-u-x-decent}, in particular by using  $\sum_{j=1}^pq_j\|\nabla{f}(_j\boldsymbol{w}_j^{(t)})\|^2$ instead of $\frac{1}{p}\sum_{j=1}^p\|\nabla{f}(\boldsymbol{w}_j^{(t)})\|^2$ in original theorem, and we skip it.

Having above results in place, we now proceed to  upper bound~(\ref{eq:avg-Lips-decent}) and derive the claimed bound:
\begin{align}
    \frac{1}{T} \sum_{t=0}^{T-1}\mathbb{E}&\left[f(\bar{\boldsymbol{w}}^{(t+1)})]-f(\bar{\boldsymbol{w}}^{(t)})\right]\leq \frac{1}{T} \sum_{t=0}^{T-1}-\eta \mathbb{E}\left[\big\langle\nabla f(\bar{\boldsymbol{w}}^{(t)}),{{\tilde{\mathbf{g}}}^{(t)}}\big\rangle\right]+\frac{1}{T} \sum_{t=0}^{T-1}\frac{\eta^2 L}{2}\mathbb{E}\left[\|{{{\tilde{\mathbf{g}}}^{(t)}}}\|^2\right]\nonumber\\
    &\leq -\frac{\eta}{2}\frac{1}{T}\sum_{t=0}^{T-1}\|\nabla f(\bar{\boldsymbol{w}}^{(t)})\|^2-\frac{\eta}{2}\frac{1}{T}\sum_{t=0}^{T-1}\|\sum_{j=1}^pq
   _j\nabla{f}_j({\boldsymbol{w}}^{(t)}_{j})\|^2\nonumber\\
   &+\frac{\eta L^2}{2}\frac{1}{T}\sum_{t=0}^{T-1}\sum_{j=1}^pq_j\|\bar{\boldsymbol{w}}^{(t)}-\boldsymbol{w}_j^{(t)}\|^2\nonumber\\
   &+\frac{\lambda\eta^2 L}{2}\Big(\frac{C_1}{p}+1\Big)    \Big[\frac{1}{T}\sum_{t=0}^{T-1}{\|\sum_{j=1}^pq_j\nabla{f}_j(\boldsymbol{w}_j^{(t)})\|^2}\Big]+\frac{\eta^2 L\sigma^2}{2p B}\nonumber\\
    &\leq -\frac{\eta}{2}\frac{1}{T}\sum_{t=0}^{T-1}\|\nabla f(\bar{\boldsymbol{w}}^{(t)})\|^2-\frac{\eta}{2}\frac{1}{T}\sum_{t=0}^{T-1}\|\sum_{j=1}^pq
   _j\nabla{f}_j({\boldsymbol{w}}^{(t)}_{j})\|^2\nonumber\\
   &+\frac{L^2\eta^3\sigma^2}{2B}\left(\frac{1+\zeta^2}{1-\zeta^2}E-1\right)\nonumber\\
   &+\frac{\lambda\eta L^2}{2}\Big[\frac{2\eta^2C_1 E}{1-\zeta^2}+\frac{\eta^2E^2}{1-\zeta}\left(\frac{2\zeta^2}{1+\zeta}+\frac{2\zeta}{1-\zeta}+\frac{E-1}{E}\right)\Big]\frac{1}{T}\sum_{t=0}^{T-1}\|\sum_{j=1}^pq_j\nabla{f}_j(\boldsymbol{w}_j^{(t)})\|^2\nonumber\\
   &+\frac{\lambda\eta^2 L}{2}\Big(\frac{C_1}{p}+1\Big)    \Big[\frac{1}{T}\sum_{t=0}^{T-1}{\|\sum_{j=1}^pq_j\nabla{f}_j(\boldsymbol{w}_j^{(t)})\|^2}\Big]+\frac{\eta^2 L\sigma^2}{2p B}\nonumber\\
   &=-\frac{\eta}{2}\frac{1}{T}\sum_{t=0}^{T-1}\|\nabla f(\bar{\boldsymbol{w}}^{(t)})\|^2\nonumber\\
   &+\frac{\eta}{2}\left[-1+\lambda\Big[\frac{2L^2\eta^2C_1 E}{1-\zeta^2}+\frac{L^2\eta^2E^2}{1-\zeta}\left(\frac{2\zeta^2}{1+\zeta}+\frac{2\zeta}{1-\zeta}+\frac{E-1}{E}\right)\Big]+\lambda\eta L\left(\frac{C_1}{p}+1\right)\right]\frac{1}{T}\sum_{t=0}^{T-1}{\sum_{j=1}^pq_j\|\nabla{f}_j(\boldsymbol{w}_j^{(t)})\|^2}\nonumber\\
   &+\frac{\eta^2 L\sigma^2}{2p B}+\frac{L^2\eta^3\sigma^2}{2B}\left(\frac{1+\zeta^2}{1-\zeta^2}E-1\right)\nonumber\\
   &\stackrel{\text{\ding{192}}}{\leq}-\frac{\eta}{2}\frac{1}{T}\sum_{t=0}^{T-1}\|\nabla f(\bar{\boldsymbol{w}}^{(t)})\|^2+\frac{\eta^2 L\sigma^2}{2p B}+\frac{L^2\eta^3\sigma^2}{2B}\left(\frac{1+\zeta^2}{1-\zeta^2}E-1\right)
\end{align}
where \ding{192} follows from the condition:
\begin{align}
    \frac{2L^2\eta^2C_1 E}{1-\zeta^2}+\frac{L^2\eta^2E^2}{1-\zeta}\left(\frac{2\zeta^2}{1+\zeta}+\frac{2\zeta}{1-\zeta}+\frac{E-1}{E}\right)+\eta L\left(\frac{C_1}{p}+1\right)\leq \frac{1}{\lambda}\label{eq:net-cond}
\end{align}
The rest of the proof is similar to the proof of Theorem~\ref{thm:FedAvg}. 
\subsection{Derivation of Conditions in Theorem~\ref{thm:serverless-fed-informal}}
Adopting the approach in \cite{wang2018cooperative} (see Eq.~(145)) condition in Eq.~(\ref{eq:net-cond}) reduces to \begin{align}
    \eta L\left(\frac{C_1}{p}+2\right)+\frac{5\eta^2L^2E^2}{{(1-\zeta)}^2}\leq \frac{1}{\lambda}
\end{align}
Therefore, by plugging $\eta=\frac{1}{L}\sqrt{\frac{p}{T}}$ we get 
\begin{align}
  \sqrt{\frac{p}{T}} \left(\frac{C_1}{p}+2\right)+\frac{p}{T}\frac{5E^2}{{(1-\zeta)}^2}\leq \frac{1}{\lambda}  
\end{align}
which with $T\geq 4\lambda^2{\left(\frac{C_1}{p}+2\right)}^2$ leads to the condition
\begin{align}
    E\leq (1-\zeta)\sqrt{\frac{T}{10 \lambda p}}
\end{align}
Therefore, if $p=\Omega\left(\frac{1+\zeta^2}{1+\zeta}\sqrt{\lambda}\right)$ with $E=\left(\frac{1-\zeta^2}{1+\zeta^2}\right)\sqrt{\frac{T}{p^3}}$ linear speed up can be achieved.
\section{Proof of Omitted Lemmas}
\label{sec:appendix:omitted}
We will use the following fact (which is also used in \cite{li2019convergence}) in proving results.
\begin{fact}\label{fact:1}
Let
$\{x_i\}_{i=1}^p$ denote any fixed deterministic sequence. We sample a multiset $\mathcal{P}$ (with size $K$) uniformly at random where $x_j$ is sampled  with probability $q_j$ for $1\leq j\leq p$ with replacement.  Let $\mathcal{P} = \{i_1,\ldots, i_K\} \subset[p]$ (some $i_j$’s may have the same value). Then
\begin{align}
    \mathbb{E}_{\mathcal{P}}\left[\sum_{i\in \mathcal{P}}x_i\right]=\mathbb{E}_{\mathcal{P}}\left[\sum_{k=1}^Kx_{i_k}\right]=K\mathbb{E}_{\mathcal{P}}\left[x_{i_k}\right]=K\left[\sum_{j=1}^pq_jx_j\right]
\end{align}
\end{fact}

\subsection{Proof of Lemma~\ref{lemma:tasbih1-gd}}
The proof is straightforward and follows from the definition of gradient over sampled machines. In particular, for the the set of sampled machines $\mathcal{P}_t, |\mathcal{P}_t| = K$ we have
\begin{align}
\|{{{\mathbf{g}}}}^{(t)}\|^2
&=\left\|\frac{1}{K}\sum_{j\in\mathcal{P}_t} {{\mathbf{g}}_j^{(t)}}\right\|^2 \nonumber\\
&\stackrel{\text{\ding{192}}}{\leq}\frac{1}{K}\sum_{j\in\mathcal{P}_t}\left\|\mathbf{g}_{j}^{(t)}\right\|^2,\label{eq:mid-bounding-absg-full-gd}
\end{align}
where \text{\ding{192}} follows from the inequality   $\|\sum_{i=1}^m\mathbf{a}_i\|^2\leq m\sum_{i=1}^{m}\|\mathbf{a}_i\|^2 $ where $\mathbf{a}_i\in\mathbb{R}^n$.

Applying Fact~\ref{fact:1} on both sides of (\ref{eq:mid-bounding-absg-full-gd}) results in
the following: 
\begin{align}
    \mathbb{E}_{\mathcal{P}_t}\|{{{\mathbf{g}}}}^{(t)}\|^2&\leq \mathbb{E}_{\mathcal{P}_t}\frac{1}{K}\sum_{j\in\mathcal{P}_t}\left\|\mathbf{g}_{j}^{(t)}\right\|^2\nonumber\\
    &=\sum_{j=1}^pq_j\left\|\mathbf{g}_{j}^{(t)}\right\|^2\nonumber\\
    &\stackrel{\text{\ding{192}}}{\leq}\lambda \left\|\sum_{j=1}^pq_j\mathbf{g}_{j}^{(t)}\right\|^2 
\end{align}

where \ding{192} follows  from the definition of weighted gradient diversity and an upper bound $\lambda$ over this quantity.



\subsection{Proof of Corollary~\ref{corol:pl-full-gd}}
The proof simply follows from definition:
\begin{align}
    -\eta \mathbb{E}\Big[\langle\nabla f(\bar{\boldsymbol{w}}^{(t)}),{{\mathbf{g}}}^{(t)}\rangle\Big]&{=}-\eta\mathbb{E}\left[\langle\nabla{f}(\bar{\boldsymbol{w}}^{(t)}),\frac{1}{K}\sum_{j\in\mathcal{P}_t}\mathbf{g}_j^{(t)}\rangle\right]\nonumber\\
    &\stackrel{\text{\ding{192}}}{=}-\eta\left[\langle \nabla{f}(\bar{\boldsymbol{w}}^{(t)}),\sum_{j=1}^{p}q_j\nabla{f}_j({\boldsymbol{w}}_j^{(t)})\rangle\right]\nonumber\\
    &\stackrel{\text{\ding{193}}}{=}-\frac{\eta}{2}\left[\|\nabla{f}(\bar{\boldsymbol{w}}^{(t)})\|_2^2+\|\sum_{j=1}^{p}q_j\nabla{f}_j({\boldsymbol{w}}_j^{(t)})\|_2^2-\|\nabla{f}(\bar{\boldsymbol{w}}^{(t)})-\sum_{j=1}^{p}q_j\nabla{f}_j({\boldsymbol{w}}_j^{(t)})\|_2^2\right]\nonumber\\
    &=\frac{\eta}{2}\left[-\|\nabla{f}(\bar{\boldsymbol{w}}^{(t)})\|_2^2-\|\sum_{j=1}^{p}q_j\nabla{f}_j({\boldsymbol{w}}_j^{(t)})\|_2^2+\|\sum_{j=1}^{p}q_j\left[\nabla{f}_j(\bar{\boldsymbol{w}}^{(t)})-\nabla{f}_j({\boldsymbol{w}}_j^{(t)})\right]\|_2^2\right]\nonumber\\
    &\stackrel{\text{\ding{194}}}{\leq}\frac{\eta}{2}\left[-\|\nabla{f}(\bar{\boldsymbol{w}}^{(t)})\|_2^2-\|\sum_{j=1}^{p}q_j\nabla{f}_j({\boldsymbol{w}}_j^{(t)})\|_2^2+\sum_{j=1}^{p}q_j\|\left[\nabla{f}_j(\bar{\boldsymbol{w}}^{(t)})-\nabla{f}_j({\boldsymbol{w}}_j^{(t)})\right]\|_2^2\right]\nonumber\\
   &\stackrel{\text{\ding{195}}}{\leq} -\mu \eta(f(\bar{\boldsymbol{w}}^{(t)})-f^*)-\frac{\eta}{2}\|\sum_{j=1}^pq
   _j\nabla{f}_j({\boldsymbol{w}}^{(t)}_{j})\|^2+\frac{\eta L^2}{2}\sum_{j=1}^pq_j\|\bar{\boldsymbol{w}}^{(t)}-\boldsymbol{w}_j^{(t)}\|^2
\end{align}
where \ding{192} holds because of Fact~\ref{fact:1}, \ding{193} comes from $2\langle\mathbf{a},\mathbf{b}\rangle=\|\mathbf{a}\|_2^2+\|\mathbf{b}\|_2^2-\|\mathbf{a}-\mathbf{b}\|_2^2$, \ding{194} is due to Assumption~\ref{Ass:1}, and finally  \ding{195} follows from Assumption~\ref{Ass:3}.


\subsection{Proof of Lemma \ref{lemma:dif-under-pl-gd}}
Let us set $t_c\triangleq \floor{\frac{t}{E}}E$. Therefore, according to Algorithm \ref{Alg:one-shot-using data samoples} we have:
 \begin{align}\bar{\boldsymbol{w}}^{(t_c+1)}=\frac{1}{K}\sum_{j\in\mathcal{P}_t} \boldsymbol{w}_j^{(t_c+1)}\end{align}
 Then, for $1\leq j\leq p$, the local model can be expressed as 
 \begin{align}
 \boldsymbol{w}_j^{(t)}=  \boldsymbol{w}_j^{(t-1)}-\eta_{t-1}{\mathbf{g}}_{j}^{(t-1)}\stackrel{\text{\ding{192}}}{=}\boldsymbol{w}_j^{(t-2)}-\Big[\eta_{t-2}{\mathbf{g}}_{j}^{(t-2)}+\eta_{t-1}{\mathbf{g}}_{j}^{(t-1)}\Big]=\bar{\boldsymbol{w}}^{(t_c+1)}-\Big[\sum_{k=t_c+1}^{t-1}\eta_k{\mathbf{g}}_{j}^{(k)}\Big], \label{eq:j-model-updates}
 \end{align}
 where \ding{192} follows from the updating rule of Algorithm~\ref{Alg:one-shot-using data samoples}. Now, from (\ref{eq:j-model-updates}) we compute the average model as follows:
\begin{align}
    \bar{\boldsymbol{w}}^{(t)}=\bar{\boldsymbol{w}}^{(t_c+1)}-\Big[\frac{1}{K}\sum_{j\in\mathcal{P}_t}\sum_{k=t_c+1}^{t-1}\eta_k{\mathbf{g}}_{j}^{(k)}\Big]\label{eq:model-avg-time-t}
\end{align}
First, without loss of generality, suppose $t=s_tE+r$ where $s_t$ and $r$ denotes the indices of communication round and local updates, respectively.
We note that  for $t_c+1 < t\leq t_c+ E$, $\mathbb{E}_t\|\bar{\boldsymbol{w}}^{(t)}-\boldsymbol{w}_j^{(t)}\|^2$  does not depend on time $t\leq t_c$ for $1\leq j\leq p$. We bound the term $\mathbb{E}\|\bar{\boldsymbol{w}}^{(t)}-\boldsymbol{w}_l^{({t})}\|^2$ for $t_c+1 \leq  {t}=t_c+r\leq t_c+ E$ as follows: 

\begin{align}
 \mathbb{E}_{\mathcal{P}_t}\|\bar{\boldsymbol{w}}^{(t_c+r)}-\boldsymbol{w}_l^{(t_c+r)}\|^2 &=\mathbb{E}_{\mathcal{P}_t}\|\bar{\boldsymbol{w}}^{(t_c+1)}-\Big[\sum_{k=t_c+1}^{{t-1}}\eta_k{\mathbf{g}}_{l}^{(k)}\Big]-\bar{\boldsymbol{w}}^{(t_c+1)}+\Big[\frac{1}{K}\sum_{j\in\mathcal{P}_t}\sum_{k=t_c+1}^{{t-1}}\eta_k{\mathbf{g}}_{j}^{(k)}\Big]\|^2\nonumber\\
 &\stackrel{\text{\ding{192}}}{=}\mathbb{E}_{\mathcal{P}_t}\|\sum_{k=1}^{r}\eta_k{\mathbf{g}}_{l}^{(t_c+k)}-\frac{1}{K}\sum_{j\in\mathcal{P}_t}\sum_{k=1}^{r}\eta_k{\mathbf{g}}_{j}^{(t_c+k)}\|^2\nonumber\\
 &\stackrel{\text{\ding{193}}}{\leq} 2\Big[\|\sum_{k=1}^{r}\eta_k{\mathbf{g}}_{l}^{(t_c+k)}\|^2+\mathbb{E}_{\mathcal{P}_t}\|\frac{1}{K}\sum_{j\in\mathcal{P}_t}\sum_{k=1}^{r}\eta_k{\mathbf{g}}_{j}^{(t_c+k)}\|^2\Big]\label{eq:some-mid-step}
\end{align}

where \ding{192} holds because ${t}=t_c+r\leq t_c+ E$, and \ding{193} is due to $\|\mathbf{a}-\mathbf{b}\|^2\leq 2(\|\mathbf{a}\|^2+\|\mathbf{b}\|^2)$.

Next, we can bound Eq.~(\ref{eq:some-mid-step}) as follows:
\begin{align}
 &\stackrel{\text{\ding{194}}}{\leq}{2}\Big(r\sum_{k=1}^{r}\eta_k^2\|\mathbf{g}_{l}^{(t_c+k)}\|^2+{r}\sum_{j=1}^p\sum_{k=1}^{r}\eta_k^2q_j\|\mathbf{g}_{j}^{(t_c+k)}\|^2\Big)
\label{eq:four-term-boundinggg}
\end{align}
where \ding{194} follow from inequality $\|\sum_{i=1}^m\mathbf{a}_i\|^2\leq m\sum_{i=1}^m\|\mathbf{a}_i\|^2$.

Expanding the  expectation over sampling probabilities of devices  in (\ref{eq:four-term-boundinggg}), we obtain:
\begin{align}
\sum_{j=1}^pq_j\|\bar{\boldsymbol{w}}^{(t)}-\boldsymbol{w}_j^{(t)}\|^2&\leq{2}\Big(\Big[r\sum_{l=1}^pq_l\sum_{k=1}^{r}\eta_k^2\|\mathbf{g}_{l}^{(t_c+k)}\|^2\Big]+r\sum_{j=1}^p\sum_{k=1}^r\eta_k^2q_j\|\mathbf{g}_{j}^{(t_c+k)}\|^2\Big)\nonumber\\
 &={4r}\sum_{j=1}^p\sum_{k=1}^{r}\eta_k^2q_j\|\mathbf{g}_{j}^{(t_c+k)}\|^2\nonumber\\
 &\stackrel{\text{\ding{195}}}{\leq} 4E\sum_{k=t_c+1}^{t-1}\sum_{j=1}^pq_j\eta_k^2\|\mathbf{g}_{j}^{(k)}\|^2,\label{eq:local-sum-timess}
\end{align}
where \ding{195} is due to $r\leq E$. Finally, Eq.~(\ref{eq:local-sum-timess}) leads to
\begin{align}
    \sum_{j=1}^pq_j\|\bar{\boldsymbol{w}}^{(t)}-\boldsymbol{w}_j^{(t)}\|^2&\leq 4E\sum_{k=t_c+1}^{t-1}\sum_{j=1}^pq_j\eta_k^2\|\
\nabla{f}_j(\boldsymbol{w}_j^{(k)})\|^2\nonumber\\
&\stackrel{\text{\ding{192}}}{\leq} 4E\lambda\sum_{k=t_c+1}^{t-1}\eta_k^2\|\sum_{j=1}^pq_j\
\nabla{f}_j(\boldsymbol{w}_j^{(k)})\|^2.
\end{align}
as stated in the lemma. Note that  \ding{192} follows from upper bound over the weighted gradient diversity $\lambda$.


\subsection{Proof of Lemma~\ref{lemma:simplification-constant}}
In this section, we derive the necessary condition we need to impose on learning rate make sure the bound stated in Lemma~\ref{lemma:simplification-constant} holds. Before establishing the necessary conditions on learning rate, we note that in the statement of lemma since $c_t$ only affects $a_t$ and it is independent of  $e_t$ and its co-efficient, we can simply show the statement for $c_t=0$ and the final result follows immediately for  $c_t \neq 0$. 

To do so, we start by deriving the conditions on learning rate that, for every time instance, allows us to remove the terms involving the coefficients of $e_k = \|\sum_{j=1}^pq_j
\nabla{f}_j(\boldsymbol{w}_j^{(k)})\|^2$ from the upper bound. Recalling the notations $\Delta=1-\mu\eta$ and ${B}\triangleq 2\lambda\eta L^2E$, we have:
\begin{align}
    a_{t+1}&\leq\Delta a_t+\frac{\eta}{2}\left(-1+L\lambda\eta\right)e_t+2\eta L^2E\sum_{k=t_c+1}^{t-1}\eta^2e_k\nonumber\\
&\stackrel{\text{\ding{192}}}{\leq} (1-\eta\mu)a_t+B \sum_{k=t_c+2}^{t-1}\eta^2e_k\nonumber\\
    &{=}\Delta a_t+B\Big(\eta^2\sum_{k=t_c+1}^{t-2}e_k+\eta^2e_{t-1}\Big),\label{eq:first-iteration}
\end{align}
where \ding{192} follows from the choice of 
\begin{align}
\label{eqn:lemma-4:eta}
\eta\leq \frac{1}{\lambda L},
\end{align} 

In what follows we show that the terms multiplied $B$ can be recursively removed:
\begin{align}
    a_{t+1}&\leq \Delta a_t
    +B\Big(\sum_{k=t_c+1}^{t-2}\eta^2e_k+\eta^2e_{t-1}\Big)\nonumber\\
    &\leq \Delta\Big[\Delta a_t+\frac{\eta}{2}\Big[-1+\eta\lambda L\Big]e_{t-1}+{B}\Big(\sum_{k=t_c+1}^{t-3}\eta^2e_k+\eta^2e_{t-2}\Big)\Big]+{B}\Big(\sum_{k=t_c+1}^{t-2}\eta^2e_k+\eta^2e_{t-1}\Big)\nonumber\\
    &={\Delta}^2a_{t-1}+\frac{\eta\Delta}{2}\Big[-1+L\lambda\eta+\frac{2\eta{B}}{ \Delta}\Big]e_{t-1}+\left(\Delta{B}+B\right)\Big(\sum_{k=t_c+1}^{t-3}\eta^2e_k+\eta^2e_{t-2}\Big),\label{eq:snd-cndd}
\end{align}
Now we bound (\ref{eq:snd-cndd}) using condition $-1+\lambda L\eta+\frac{2\eta{B}}{ \Delta}\leq 0$ or equivalently
$$\eta\leq\frac{1}{\lambda L+2\frac{B}{ \Delta}},$$ which gives us the following bound:
\begin{align}
    &{\leq} \Delta^2a_{t-1}+\big[\Delta{B}+{B}\big]\Big(\sum_{k=t_c+1}^{t-3}\eta^2e_k+\eta^2e_{t-2}\Big)\nonumber\\
    &\leq {\Delta^2}\Big[{\Delta}a_{t-2}+\frac{\eta}{2}\Big[-1+L\lambda\eta\Big]e_{t-2}+{B}\Big(\eta^2\sum_{k=t_c+1}^{t-4}e_k+\eta^2e_{t-3}\Big)\Big]\nonumber\\
    &\:+\big[\Delta{B}+{B}\big]\Big(\eta^2\sum_{k=t_c+1}^{t-3}e_k+\eta^2e_{t-2}\Big)\nonumber\\
    &=\Delta^3a_{t-2}+\frac{\Delta^2\eta}{2}\Big[-1+L\lambda\eta+\frac{2\eta}{\Delta^2}\big[\Delta {B}+{B}\big]\Big]e_{t-2}\nonumber\\
    &+\Big[\Delta^2{B}+\Delta{B}+{B}\Big]\Big(\sum_{k=t_c+1}^{t-4}\eta^2e_k+\eta^2e_{t-3}\Big)\nonumber\\
    &\stackrel{\text{\ding{192}}}{\leq} \Delta^3a_{t-2}+\Big[\Delta^2{B}+\Delta{B}+{B}\Big]\Big(\eta^2\sum_{k=t_c+1}^{t-4}e_k+\eta^2e_{t-3}\Big),\label{eq:induction-step-2-unbb}
\end{align}
where \ding{192} follows from $-1+L\eta\lambda+\frac{2\eta}{\Delta^2}\big[\Delta {B}(E)+{B}(E)\big]\leq 0$ or equivalently
\begin{align}
    \eta&\leq\frac{1}{L\lambda+\frac{2B}{\Delta^2}\big[\Delta +1\big]}.
\end{align}
By induction on (\ref{eq:induction-step-2-unbb}) we get:
\begin{align}
a_{t+1}&\leq \Delta^{E-1}a_{t_c+2}+\frac{\Delta^{E-2}\eta}{2}  \Big[-1+\lambda L\eta+\frac{2\eta}{\Delta^{E-2}}\big[\Delta^{E-3}{B}+\ldots+\Delta {B}+{B}\big]\Big]e_{t_c+2}\nonumber\\
    &+\Big[\Delta^{E-2}{B}+\ldots+\Delta{B}+{B}\Big]e_{t_c+1}\nonumber\\
   &\stackrel{\text{\ding{192}}}{\leq} \Delta^{E-1}a_{t_c+2}+\Big[\Delta^{E-2}{B}+\ldots+\Delta{B}+{B}\Big]e_{t_c+1},\label{eq:induction-step-3-biased}
\end{align}
similarly \ding{192} follows from the condition $$-1+\lambda L\eta+\frac{2\eta}{\Delta^{E-2}}\big[\Delta^{E-3}{B}+\ldots+\Delta {B}+{B}\big]\leq 0$$ which gives
\begin{align}
    \eta&\leq\frac{1}{L\lambda+\frac{2}{\Delta^{E-2}}\big[\Delta^{E-3}{B}+\ldots+\Delta {B}+{B}\big]},
\end{align}
Continuing from (\ref{eq:induction-step-3-biased}) results in:
\begin{align}
a_{t+1}&\stackrel{\text{\ding{192}}}{\leq}\Delta^{E}a_{t_c+1}+\frac{\Delta^{E-1}\eta}{2}\Big[-1+\lambda{L\eta}+\frac{2\eta}{\Delta^{E-1}}\big[\Delta^{E-2}{B}+\ldots+\Delta{B}+{B}\big]\Big]e_{t_c+1}\nonumber\\
    &\stackrel{\text{\ding{193}}}{\leq}\Delta^{E}a_{t_c+1},\label{eq:induction-step-4}
\end{align}
where \ding{192} is due to the update rule that $\bar{\boldsymbol{w}}^{(t_c+1)}={\boldsymbol{w}}^{(t_c+1)}_j$ for all $1\leq j\leq p$, and \ding{193} comes from
\begin{align}
0&\geq     -1+\lambda{L\eta}+\frac{2\eta}{\Delta^{E-1}}\big[\Delta^{E-2}{B}+\ldots+\Delta{B}+B\big]
\end{align}
which simplifies as  
\begin{align}
    \eta&\leq \frac{1}{\lambda L+\frac{2{B}}{\Delta^{E-1}}\big[\Delta^{E-2}+\ldots+\Delta+1\big]},\label{eq:bnding-eta}
\end{align}
We note the final condition is  tighter than  the condition in (\ref{eqn:lemma-4:eta}) and implies it.

\subsection{Proof of Lemma~\ref{lemma:2nd-cond-gd}}
The proof follows from lower bounding the constraint imposed on the learning rate. Specifically, we have: 
\begin{align}
  \frac{1}{\lambda L+\frac{2{B}}{\Delta^{E-1}}\big[\Delta^{E-2}+\ldots+\Delta+1\big]}&= \frac{1}{L\lambda+\frac{2{B}}{\Delta^{E-1}}\frac{1-\Delta^{E-1}}{1-\Delta}},\nonumber\\
  &=\frac{1}{L\lambda+\frac{4\lambda L^2\eta E}{\Delta^{E-1}}\frac{1-\Delta^{E-1}}{1-\Delta}}\nonumber\\
  &\geq \frac{1}{L\lambda+\frac{4\lambda L^2\eta E}{\Delta^{E-1}}\frac{1}{\mu\eta}}\label{eq:lower-bnd-lambda}
\end{align}

Therefore, by the choice of learning rate $\eta$ satisfying:
\begin{align}
    \eta&\leq  \frac{1}{\lambda\left(L+\frac{4L^2\eta E}{\Delta^{E-1}}\frac{1}{\mu\eta}\right)}
\end{align}
which yields to the condition:

\begin{align}
    \eta \Big(L+\frac{4EL^2}{\mu(1-\mu\eta)^{E-1}}\Big)\leq \frac{1}{\lambda}
\end{align}


\subsection{Proof of Lemma \ref{lemma:tasbih1}}

The following lemma is a middle step in proving Lemma~\ref{lemma:tasbih1}.

\begin{lemma}\label{lemma:variance-bound-for-prrof1}
Under Assumptions \ref{Ass:2} and our sampling scheme in Algorithm~\ref{Alg:one-shot-using data samoples}, we have the following variance bound from the averaged stochastic gradient:
\begin{align}
    \mathbb{E}\left[\mathbb{E}_{\mathcal{P}_t}\Big[\|{{\tilde{\mathbf{g}}}^{(t)}}-{{{\mathbf{g}}}^{(t)}}\|^2\Big]\right]\leq \frac{C_1}{K}\sum_{j=1}^pq_j  \|\nabla{f}_j(\boldsymbol{w}_{j}^{(t)})\|^2+\frac{1}{K}C_2^2
\end{align}
\end{lemma}

\begin {proof}
We have
\begin{align}
    \mathbb{E}\Big[\|{{\tilde{\mathbf{g}}}^{(t)}}-{{{\mathbf{g}}}^{(t)}}\|^2\Big]&\stackrel{\text{\ding{192}}}{=}\mathbb{E}\Big[\|\frac{1}{K}\sum_{j\in\mathcal{P}_t} {{\tilde{\mathbf{g}}}^{(t)}}_j-\frac{1}{K}\sum_{j\in\mathcal{P}_t} {{\mathbf{g}}^{(t)}_j}\|^2\Big]\nonumber\\
    &=\frac{1}{K^2}\mathbb{E}\Big[\sum_{j\in\mathcal{P}_t}\|\Big(\tilde{\mathbf{g}}_{j}^{(t)}-\mathbf{g}_{j}^{(t)})\Big)\|^2+\sum_{i\neq j}\Big\langle{\tilde{\mathbf{g}}}_{i}^{(t)}-\mathbf{g}_{i}^{(t)},\tilde{\mathbf{g}}_{j}^{(t)}-\mathbf{g}_{j}^{(t)}\Big\rangle\Big]\nonumber\\
    &{=}\frac{1}{K^2}\sum_{j\in\mathcal{P}_t}\mathbb{E}_{\xi^{(t)}_{j}|{\boldsymbol{w}}^{(t)}}\|\Big(\tilde{\mathbf{g}}_{j}^{(t)}-\mathbf{g}_{j}^{(t)}\Big)\|^2+\sum_{i\neq j}\frac{1}{K^2}\mathbb{E}\big[\langle\tilde{\mathbf{g}}_{j}^{(t)}-\mathbf{g}_{j}^{(t)},\tilde{\mathbf{g}}_{i}^{(t)}-\mathbf{g}_{i}^{(t)}\rangle\big]\label{eq:variance bound}\\
       &\stackrel{\text{\ding{193}}}{=}\frac{1}{K^2}\sum_{j\in\mathcal{P}_t}\mathbb{E}_{\xi^{(t)}_{j}|{\boldsymbol{w}}^{(t)}}\|\Big(\tilde{\mathbf{g}}_{j}^{(t)}-\mathbf{g}_{j}^{(t)}\Big)\|^2+\frac{1}{K^2}\sum_{i\neq j}\Big\langle\mathbb{E}_{\xi^{(t)}_{j}|{\boldsymbol{w}}^{(t)}_l}\big[\tilde{\mathbf{g}}_{j}^{(t)}-\mathbf{g}_{j}^{(t)}\big],\mathbb{E}_{\xi^{(t)}_{j}|\boldsymbol{w}^{(t)}_m}\big[\tilde{\mathbf{g}}_{i}^{(t)}-\mathbf{g}_{i}^{(t)}\big]\Big\rangle\nonumber\\
    &\stackrel{\text{\ding{194}}}{\leq} \frac{1}{K^2}\sum_{j\in\mathcal{P}_t}\Big[C_1\|\mathbf{g}_{j}^{(t)}\|^2+C_2^2\Big]\label{eq:after-eq}\\
    &=\frac{C_1}{K^2}\sum_{j\in\mathcal{P}_t}\|\mathbf{g}_{j}^{(t)}\|^2+\frac{C_2^2}{K}\label{eq:var_b_mid}
    \end{align}
where in \text{\ding{192}} we use the definition of ${\tilde{\mathbf{g}}}^t$ and ${{\mathbf{g}}}^t$, in {\ding{193}} we use the fact that mini-batches are chosen in i.i.d. manner at each local machine, and {\ding{194}} immediately follows from Assumptions~\ref{Ass:2}.

Next, by taking expectation from both sides of (\ref{eq:var_b_mid}) with respect to random sampling of devices, we get:
\begin{align}
    \mathbb{E}_{ \mathcal{P}_t}\left[\mathbb{E}\Big[\|{{\tilde{\mathbf{g}}}^{(t)}}-{{{\mathbf{g}}}^{(t)}}\|^2\Big]\right]&\leq \mathbb{E}_{ \mathcal{P}_t}\Big[\frac{C_1}{K^2}\sum_{j\in\mathcal{P}_t}\|\mathbf{g}_{j}^{(t)}\|^2+\frac{C_2^2}{K}\Big]\nonumber\\
    &=\frac{C_1}{K^2}\mathbb{E}_{ \mathcal{P}_t}\left[\sum_{j\in\mathcal{P}_t}\|\mathbf{g}_{j}^{(t)}\|^2\right]+\frac{C_2^2}{K}\nonumber\\
    &\stackrel{\text{\ding{192}}}{=}\frac{K C_1}{K^2}\sum_{j=1}^pq_j\|\mathbf{g}_{j}^{(t)}\|^2+\frac{C_2^2}{K}\nonumber
    \end{align}
where \ding{192} comes from Fact~\ref{fact:1}. 

\end{proof}
Equipped with Lemma~\ref{lemma:variance-bound-for-prrof1}, we now turn to proving Lemma~\ref{lemma:tasbih1}. First we note that the Assumption~\ref{Ass:2} implies $\mathbb{E}[{\tilde{\mathbf{g}}}_j^{(t)}]={{{\mathbf{g}}}}_j^{(t)}$, from which we have
\begin{align}
\mathbb{E}\Big[\|{\tilde{\mathbf{g}}}^{(t)}\|^2\Big]&=\mathbb{E}\big[\|{\tilde{\mathbf{g}}}^{(t)}-\mathbb{E}\big[{\tilde{\mathbf{g}}}^{(t)}\big]\|^2\Big]+\|\mathbb{E}\big[{\tilde{\mathbf{g}}}^{(t)}\big]\|^2\nonumber\\
&=\mathbb{E}\big[\|{\tilde{\mathbf{g}}}^{(t)}-{{{\mathbf{g}}}}^{(t)}\big]\|^2\Big]+\|{{{\mathbf{g}}}}^{(t)}\|^2\nonumber\\
&\leq \frac{C_1}{K^2}\sum_{j\in\mathcal{P}_t}\|\mathbf{g}_{j}^{(t)}\|^2+\frac{C_2^2}{K}+\|\frac{1}{K}\sum_{j\in\mathcal{P}_t} {{\mathbf{g}}_j^{(t)}}\|^2\nonumber\\
&\stackrel{\text{\ding{193}}}{\leq}\frac{C_1}{K^2}\sum_{j\in\mathcal{P}_t}\|\mathbf{g}_{j}^{(t)}\|^2+\frac{C_2^2}{K}+\frac{1}{K}\sum_{j\in\mathcal{P}_t}\|\mathbf{g}_{j}^{(t)}\|^2\nonumber\\
&=\left(\frac{C_1+K}{K^2}\right)\sum_{j\in\mathcal{P}_t}\|\mathbf{g}_{j}^{(t)}\|^2+\frac{C_2^2}{K}\label{eq:mid-bounding-absg}
\end{align}
where \text{\ding{192}} and \text{\ding{193}} follows from the fact that $\|\sum_{i=1}^m\mathbf{a}_i\|^2\leq m\sum_{i=1}^{m}\|\mathbf{a}_i\|^2 $ where $\mathbf{a}_i\in\mathbb{R}^n$. Applying Fact~\ref{fact:1} on both sides of (\ref{eq:mid-bounding-absg}) and using the upperbound over the weighted gradient diversity, $\lambda$, 
\begin{align}
    \mathbb{E}_{\mathcal{P}_t}\mathbb{E}\Big[\|{\tilde{\mathbf{g}}}^{(t)}\|^2\Big]\leq \lambda\left(\frac{C_1+K}{K^2}\right)\|\sum_{j=1}^pq_j\mathbf{g}_{j}^{(t)}\|^2+\frac{C_2^2}{K}
\end{align}

results in the stated bound.

\subsection{Proof of Lemma \ref{lemma:cross-inner-bound-unbiased}}
Let $\mathcal{P}_t=\{i_1,\ldots,i_K\}$ and $\tilde{\mathbf{g}}^{(t)}= \frac{1}{K}\sum_{j\in\mathcal{P}_t} {{\tilde{\mathbf{g}}}^{(t)}}_j$ be the set of sampled machines and average of their local stochastic gradients at $t$th iteration, respectively.  We have:

\begin{align}
    -\mathbb{E}_{\{{\xi}^{(t)}_{1}, \ldots, {\xi}^{(t)}_{p}|{\boldsymbol{w}}^{(t)}_{1},\ldots,  {\boldsymbol{w}}^{(t)}_{p}\}} &\mathbb{E}_{\{i_1,\ldots,i_K\}\in\mathcal{P}_t}\Big[\Big\langle \nabla f(\bar{\boldsymbol{w}}^{(t)}),\tilde{\mathbf{g}}^{(t)}\Big\rangle\Big]\nonumber\\
    &=-\mathbb{E}_{\{{\xi}^{(t)}_{1}, \ldots, {\xi}^{(t)}_{p}|{\boldsymbol{w}}^{(t)}_{1},\ldots,  {\boldsymbol{w}}^{(t)}_{p}\}}\mathbb{E}_{\{i_1,\ldots,i_K\}\in\mathcal{P}_t}\Big[\Big\langle \nabla f(\bar{\boldsymbol{w}}^{(t)}),\frac{1}{K}\sum_{j\in\mathcal{P}_t} {{\tilde{\mathbf{g}}}^{(t)}}_j\Big\rangle\Big]\nonumber\\
    &\stackrel{\text{\ding{192}}}{=}-\mathbb{E}_{\{i_1,\ldots,i_K\}\in\mathcal{P}_t}\mathbb{E}_{\{{\xi}^{(t)}_{1}, \ldots, {\xi}^{(t)}_{p}|{\boldsymbol{w}}^{(t)}_{1},\ldots,  {\boldsymbol{w}}^{(t)}_{p}\}}\Big[\Big\langle \nabla f(\bar{\boldsymbol{w}}^{(t)}),\frac{1}{K}\sum_{j\in\mathcal{P}_t} {{\tilde{\mathbf{g}}}^{(t)}}_j\Big\rangle\Big]\nonumber\\
    &=-\Big\langle \nabla f(\bar{\boldsymbol{w}}^{(t)}),\mathbb{E}_{\mathcal{P}_t}\Big[\frac{1}{K}\sum_{j\in\mathcal{P}_t}\mathbb{E}_t[\tilde{\mathbf{g}}_j]\Big\rangle\nonumber\\
    &=-\Big\langle \nabla f(\bar{\boldsymbol{w}}^{(t)}),\mathbb{E}_{\mathcal{P}_t}\Big[\frac{1}{K}\sum_{j\in\mathcal{P}_t}\nabla{f}_j({\boldsymbol{w}}^{(t)}_j)\Big]\Big\rangle\nonumber\\
     &=-\Big\langle \nabla f(\bar{\boldsymbol{w}}^{(t)}),\frac{1}{K}\mathbb{E}_{\mathcal{P}_t}\Big[\sum_{k=1}^K\nabla{f}_j({\boldsymbol{w}}^{(t)}_{i_k})\Big]\Big\rangle\nonumber\\
     &\stackrel{\text{\ding{193}}}{=}-\Big\langle \nabla f(\bar{\boldsymbol{w}}^{(t)}),\frac{1}{K}\Big[K\sum_{j=1}^pq_j\nabla{f}_j({\boldsymbol{w}}^{(t)}_{j})\Big]\Big\rangle\nonumber\\
     &\stackrel{\text{\ding{194}}}{=}\frac{1}{2}\left[-\|\nabla f(\bar{\boldsymbol{w}}^{(t)})\|_2^2-\|\sum_{j=1}q_j\nabla{f}_j(\boldsymbol{w}_j^{(t)})\|_2^2+\|\nabla f(\bar{\boldsymbol{w}}^{(t)})-\sum_{j=1}q_j\nabla{f}_j(\boldsymbol{w}_j^{(t)})\|_2^2\right]\nonumber\\
     &=\frac{1}{2}\left[-\|\nabla f(\bar{\boldsymbol{w}}^{(t)})\|_2^2-\|\sum_{j=1}q_j\nabla{f}_j(\boldsymbol{w}_j^{(t)})\|_2^2+\|\sum_{j=1}q_j\left(\nabla f_j(\bar{\boldsymbol{w}}^{(t)})-\nabla{f}_j(\boldsymbol{w}_j^{(t)})\right)\|_2^2\right]\nonumber\\
     &\stackrel{\text{\ding{195}}}{\leq}\frac{1}{2}\left[-\|\nabla f(\bar{\boldsymbol{w}}^{(t)})\|_2^2-\|\sum_{j=1}q_j\nabla{f}_j(\boldsymbol{w}_j^{(t)})\|_2^2+\sum_{j=1}q_j\|\nabla f_j(\bar{\boldsymbol{w}}^{(t)})-\nabla{f}_j(\boldsymbol{w}_j^{(t)})\|_2^2\right] \nonumber\\
    &\stackrel{\text{\ding{196}}}{\leq}\frac{1}{2}\left[-\|\nabla f(\bar{\boldsymbol{w}}^{(t)})\|_2^2-\|\sum_{j=1}q_j\nabla{f}_j(\boldsymbol{w}_j^{(t)})\|_2^2+\sum_{j=1}q_jL^2\|\bar{\boldsymbol{w}}^{(t)}-\boldsymbol{w}_j^{(t)}\|_2^2\right]
   \label{eq:bounding-cross-no-redundancy}
\end{align}
where \ding{192} is due to the fact that random variables $\xi^{(t)}$ and $\mathcal{P}_t$ are independent, since the choice of random mini-batch is independent of whether or not a device is selected randomly,\ding{193} follows from Fact~\ref{fact:1}, \ding{194} is due to $2\langle \mathbf{a},\mathbf{b}\rangle=\|\mathbf{a}\|^2+\|\mathbf{b}\|^2-\|\mathbf{a}-\mathbf{b}\|^2$, \ding{195} holds because of convexity of $\|.\|_2$, and \ding{196} follows from Assumption \ref{Ass:1}.

\subsection{Proof of Lemma~\ref{lemma:dif-under-pl-sgd}}
Let us set $t_c\triangleq \floor{\frac{t}{E}}E$. Therefore, according to Algorithm \ref{Alg:one-shot-using data samoples} we have:
 \begin{align}\bar{\boldsymbol{w}}^{(t_c+1)}=\frac{1}{K}\sum_{j\in \mathcal{P}_t} \boldsymbol{w}_j^{(t_c+1)}\end{align} for $1\leq j\leq p$. Then, the update rule of Algorithm \ref{Alg:one-shot-using data samoples}, can be rewritten  as:
 \begin{align}
 \boldsymbol{w}_j^{(t)}=  \boldsymbol{w}_j^{(t-1)}-\eta_{t-1}\tilde{\mathbf{g}}_{j}^{(t-1)}\stackrel{\text{\ding{192}}}{=}\boldsymbol{w}_j^{(t-2)}-\Big[\eta_{t-2}\tilde{\mathbf{g}}_{j}^{(t-2)}+\eta_{t-1}\tilde{\mathbf{g}}_{j}^{(t-1)}\Big]=\bar{\boldsymbol{w}}^{(t_c+1)}-\sum_{k=t_c+1}^{t-1}\eta_k\tilde{\mathbf{g}}_{j}^{(k)}, \label{eq:j-model-updatee}
 \end{align}
where \ding{192} follows from the updating rule. Continuing from~(\ref{eq:j-model-updatee}), we now compute the average model as follows:
\begin{align}
    \bar{\boldsymbol{w}}^{(t)}=\bar{\boldsymbol{w}}^{(t_c+1)}-\frac{1}{K}\sum_{j\in \mathcal{P}_t}\sum_{k=t_c+1}^{t-1}\eta_k\tilde{\mathbf{g}}_{j}^{(k)}\label{eq:model-avg-time-t}
\end{align}
First, without loss of generality, suppose {$t=t_c+r$} where  $r$ denotes the indices of local updates.
We note that  for $t_c+1 < t\leq t_c+ E$, $\mathbb{E}_t\|\bar{\boldsymbol{w}}^{(t)}-\boldsymbol{w}_j^{(t)}\|^2$  does not depend on time $t\leq t_c$ for $1\leq j\leq p$. 

We bound the term {$\mathbb{E}\|\bar{\boldsymbol{w}}^{(t)}-\boldsymbol{w}_l^{({t})}\|^2$ for $t_c+1 \leq  {t}=t_c+r\leq t_c+ E$} in three steps: 
\textcolor{black}{1) We first relate this quantity to the variance between stochastic gradient and full gradient, 2) We use Assumption~\ref{Ass:1} on unbiased estimation and i.i.d sampling, 
3) We use Assumption~\ref{Ass:2} to bound the final terms.} 

We proceed to the details each of these three steps.

\noindent\textbf{Step 1: Relating to variance}

\begin{align}
 \mathbb{E}&\left[\|\bar{\boldsymbol{w}}^{(t_c+r)}-\boldsymbol{w}_l^{(t_c+r)}\|^2\right] =\mathbb{E}\left[\|\bar{\boldsymbol{w}}^{(t_c+1)}-\Big[\sum_{k=t_c+1}^{{t-1}}\eta_k\tilde{\mathbf{g}}_{l}^{(k)}\Big]-\bar{\boldsymbol{w}}^{(t_c+1)}+\Big[\frac{1}{K}\sum_{j\in \mathcal{P}_t}\sum_{k=t_c+1}^{{t-1}}\eta_k\tilde{\mathbf{g}}_{j}^{(k)}\Big]\|^2\right]\nonumber\\
 &\stackrel{\text{\ding{192}}}{=}\mathbb{E}\left[\|\sum_{k=1}^{r}\eta_{t_c+k}\tilde{\mathbf{g}}_{l}^{(t_c+k)}-\frac{1}{K}\sum_{j\in \mathcal{P}_t}\sum_{k=1}^{r}\eta_{t_c+k}\tilde{\mathbf{g}}_{j}^{(t_c+k)}\|^2\right]\nonumber\\
 &\stackrel{\text{\ding{193}}}{\leq} 2\left[\mathbb{E}\left[\|\sum_{k=1}^{r}\eta_{t_c+k}\tilde{\mathbf{g}}_{l}^{(t_c+k)}\|^2\right]+\mathbb{E}\left[\|\frac{1}{K}\sum_{j\in \mathcal{P}_t}\sum_{k=1}^{r}\eta_{t_c+k}\tilde{\mathbf{g}}_{j}^{(t_c+k)}\|^2\right]\right]\nonumber\\
 &\stackrel{\text{\ding{194}}}{=}2\left[\mathbb{E}\left[\|\sum_{k=1}^{r}\eta_{t_c+k}\tilde{\mathbf{g}}_{l}^{(t_c+k)}-\mathbb{E}\big[\sum_{k=1}^{r}\eta_{t_c+k}\tilde{\mathbf{g}}_{l}^{(t_c+k)}\big]\|^2\right]+\|\mathbb{E}\left[\sum_{k=1}^{{r}}\eta_{t_c+k}\tilde{\mathbf{g}}_{l}^{(t_c+k)}\right]\|^2\right]\nonumber\\
 &+2\mathbb{E}\left[\|\frac{1}{K}\sum_{j\in \mathcal{P}_t}\sum_{k=1}^{{r}}\eta_{t_c+k}\tilde{\mathbf{g}}_{j}^{(t_c+k)}-\mathbb{E}\left[\frac{1}{K}\sum_{j\in \mathcal{P}_t}\sum_{k=1}^{r}\eta_{t_c+k}\tilde{\mathbf{g}}_{j}^{(t_c+k)}\right]\|^2+\|\mathbb{E}\left[\frac{1}{K}\sum_{j\in \mathcal{P}_t}\sum_{k=1}^{r}\eta_{t_c+k}\tilde{\mathbf{g}}_{j}^{(t_c+k)}\right]\|^2\right]\nonumber\\
 &\stackrel{\text{\ding{195}}}{=}2\mathbb{E}\left[\|\sum_{k=1}^{r}\eta_{t_c+k}\left[\tilde{\mathbf{g}}_{l}^{(t_c+k)}-\mathbf{g}_{l}^{(t_c+k)}\right]\|^2+\|\sum_{k=1}^{r}\eta_{t_c+k}\mathbf{g}_{l}^{(t_c+k)}\|^2\right]\nonumber\\
 &\quad+2\mathbb{E}\left[\|\frac{1}{K}\sum_{j\in \mathcal{P}_t}\sum_{k=1}^{r}\eta_{t_c+k}\Big[\tilde{\mathbf{g}}_{j}^{(t_c+k)}-\mathbf{g}_{j}^{(t_c+k)}\Big]\|^2+\|\frac{1}{K}\sum_{j\in \mathcal{P}_t}\sum_{k=1}^{r}\eta_{t_c+k}\mathbf{g}_{j}^{(t_c+k)}\|^2\right],\nonumber\\
\end{align}

where \ding{192} holds because ${t}=t_c+r\leq t_c+ E$, \ding{193} is due to $\|\mathbf{a}-\mathbf{b}\|^2\leq 2(\|\mathbf{a}\|^2+\|\mathbf{b}\|^2)$, \ding{194} comes from $\mathbb{E}[\boldsymbol{w}^2]=\mathbb{E}[[\boldsymbol{w}-\mathbb{E}[\boldsymbol{w}]]^2]+\mathbb{E}[\boldsymbol{w}]^2$, \ding{195} comes from unbiased estimation Assumption~\ref{Ass:1}.

\noindent\textbf{Step 2: Local unbiased estimation and i.i.d. sampling}

\begin{align}
    &{=}{2}\mathbb{E}\Big(\Big[\sum_{k=1}^{r}\eta_{t_c+k}^2\|\tilde{\mathbf{g}}_{l}^{(t_c+k)}-\mathbf{g}_{l}^{(t_c+k)}\|^2\nonumber\\&+\sum_{w\neq z \vee l\neq v}\Big\langle\eta_w\tilde{\mathbf{g}}_{l}^{(w)}-\eta_w\mathbf{g}_{l}^{(w)},\eta_z\tilde{\mathbf{g}}_{v}^{(z)}-\eta_z\mathbf{g}_{v}^{(z)}\Big\rangle+\|\sum_{k=1}^{r}\eta_{t_c+k}\mathbf{g}_{l}^{(t_c+k)}\|^2\Big]\nonumber\\
 &\quad+\frac{1}{K^2}\sum_{j\in \mathcal{P}_t}\sum_{k=1}^{r}\eta_{t_c+k}^2\|\tilde{\mathbf{g}}_{l}^{(t_c+k)}-\mathbf{g}_{l}^{(t_c+k)}\|^2\nonumber\\
 &+\frac{1}{K^2}\sum_{w\neq z \vee l\neq v}\Big\langle\eta_w\tilde{\mathbf{g}}_{l}^{(w)}-\eta_w\mathbf{g}_{l}^{(w)},\eta_z\tilde{\mathbf{g}}_{v}^{(z)}-\eta_z\mathbf{g}_{v}^{(z)}\Big\rangle+\|\frac{1}{K}\sum_{j\in \mathcal{P}_t}\sum_{k=1}^{r}\eta_{t_c+k}\mathbf{g}_{j}^{(t_c+k)}\|^2 \Big)\nonumber\\
 &\stackrel{\text{\ding{196}}}{=}{2}\mathbb{E}\Big(\Big[\sum_{k=1}^{r}\eta^2_{t_c+k}\|\tilde{\mathbf{g}}_{l}^{(t_c+k)}-\mathbf{g}_{l}^{(t_c+k)}\|^2+\|\sum_{k=1}^{r}\eta_{t_c+k}\mathbf{g}_{l}^{(t_c+k)}\|^2\Big]\nonumber\\
 &\quad+\frac{1}{K^2}\sum_{j\in \mathcal{P}_t}\sum_{k=1}^{r}\eta^2_{t_c+k}\|\tilde{\mathbf{g}}_{j}^{(t_c+k)}-\mathbf{g}_{j}^{(t_c+k)}\|^2+\|\frac{1}{K}\sum_{j\in \mathcal{P}_t}\sum_{k=1}^{r}\eta_{t_c+k}\mathbf{g}_{j}^{(t_c+k)}\|^2\Big)\nonumber\\
 &\stackrel{\text{\ding{197}}}{\leq}{2}\mathbb{E}\Big(\Big[\sum_{k=1}^{r}\eta_{t_c+k}^2\|\tilde{\mathbf{g}}_{l}^{(t_c+k)}-\mathbf{g}_{l}^{(t_c+k)}\|^2+r\sum_{k=1}^{r}\eta_{t_c+k}^2\|\mathbf{g}_{l}^{(t_c+k)}\|^2\Big]\nonumber\\
 &\quad+\frac{1}{K^2}\sum_{j\in \mathcal{P}_t}\sum_{k=1}^{r}\|\tilde{\mathbf{g}}_{j}^{(t_c+k)}-\mathbf{g}_{j}^{(t_c+k)}\|^2+\frac{r}{K^2}\sum_{j\in \mathcal{P}_t}\sum_{k=1}^{r}\eta_{t_c+k}^2\|\mathbf{g}_{j}^{(t_c+k)}\|^2\Big)\nonumber\\
 &={2}\Big(\Big[\sum_{k=1}^{r}\eta_{t_c+k}^2\mathbb{E}\|\tilde{\mathbf{g}}_{l}^{(t_c+k)}-\mathbf{g}_{l}^{(t_c+k)}\|^2+r\sum_{k=1}^{r}\eta_{t_c+k}^2\mathbb{E}\|\mathbf{g}_{l}^{(t_c+k)}\|^2\Big]\nonumber\\
 &\quad+\frac{1}{K^2}\sum_{j\in \mathcal{P}_t}\sum_{k=1}^{r}\eta_{t_c+k}^2\mathbb{E}\|\tilde{\mathbf{g}}_{j}^{(t_c+k)}-\mathbf{g}_{j}^{(t_c+k)}\|^2+\frac{r}{K^2}\sum_{j\in \mathcal{P}_t}\sum_{k=1}^{r}\eta_{t_c+k}^2\mathbb{E}\|\mathbf{g}_{j}^{(t_c+k)}\|^2\Big),
\label{eq:four-term-bounding}
\end{align}
\ding{196} is due to independent mini-batch sampling as well as unbiased estimation assumption over individual local data shards, and  \ding{197}  follow from inequality $\|\sum_{i=1}^m\mathbf{a}_i\|^2\leq m\sum_{i=1}^m\|\mathbf{a}_i\|^2$. \\

\noindent\textbf{Step 3: Using Assumption \ref{Ass:2}}

Next step is to bound the terms in (\ref{eq:four-term-bounding}) using Assumption \ref{Ass:2} as follow:

{\begin{align}\mathbb{E}\|\bar{\boldsymbol{w}}^{({t})}-\boldsymbol{w}_l^{({t})}\|^2&\leq {2}\Big(\Big[\sum_{k=1}^{r}\eta_{t_c+k}^2\Big[C_1\|{\mathbf{g}}_l^{(t_c+k)})\|^2+{\frac{\sigma^2}{B}}\Big]+r\sum_{k=1}^{r}\eta_{t_c+k}^2\|\Big[{\mathbf{g}}^{(t_c+k)}_l\Big]\|^2\Big]\nonumber\\
 &\quad+\frac{1}{K^2}\sum_{j\in \mathcal{P}_t}\sum_{k=1}^{r}\eta_{t_c+k}^2\Big[C_1\|{\mathbf{g}}^{(t_c+k)}_j\|^2+{\frac{\sigma^2}{B}}\Big]+\frac{r}{K^2}\sum_{j\in \mathcal{P}_t}\sum_{k=1}^{r}\eta_{t_c+k}^2\|\Big[{\mathbf{g}}_j^{(t_c+k)}\Big]\|^2\Big)\nonumber\\
 &={2}\Big(\Big[\sum_{k=1}^{r}\eta_{t_c+k}^2C_1\|\mathbf{g}_{l}^{(t_c+k)}\|^2+\sum_{k=1}^{r}\eta_{t_c+k}^2{\frac{\sigma^2}{B}}+r\sum_{k=1}^{r}\eta_{t_c+k}^2\|\mathbf{g}_{l}^{(t_c+k)}\|^2\Big]\nonumber\\
 &\quad+\frac{1}{K^2}\sum_{j\in \mathcal{P}_t}\sum_{k=1}^{r}\eta_{t_c+k}^2C_1\|\mathbf{g}_{j}^{(t_c+k)}\|^2+\sum_{k=1}^r\eta^2_{t_c+k}\frac{\sigma^2}{K^2B}+\frac{r}{K^2}\sum_{j\in \mathcal{P}_t}\sum_{k=1}^{r}\eta_{t_c+k}^2\mathbb{E}\|\mathbf{g}_{j}^{(t_c+k)}\|^2\Big),\label{eq:var-n-bounnd}
\end{align}}
Now taking summation over worker indices (\ref{eq:var-n-bounnd}), we obtain:
{\begin{align}
    \mathbb{E}\sum_{j\in \mathcal{P}_t}\|\bar{\boldsymbol{w}}^{(t)}-\boldsymbol{w}_j^{(t)}\|^2&\leq{2}\Big(\Big[\sum_{l\in \mathcal{P}_t}\sum_{k=1}^r\eta_{t_c+k}^2C_1\|\mathbf{g}_{l}^{(t_c+k)}\|^2+\sum_{k=1}^r\eta^2_{t_c+k}{\frac{\sigma^2}{B}}+r\sum_{l\in \mathcal{P}_t}\sum_{k=1}^{r}\eta_{t_c+k}^2\|\mathbf{g}_{l}^{(t_c+k)}\|^2\Big]\nonumber\\
 &\quad+\frac{1}{K}\sum_{j\in \mathcal{P}_t}\sum_{k=1}^r\eta_{t_c+k}^2C_1\|\mathbf{g}_{j}^{(t_c+k)}\|^2+\sum_{k=1}^r\eta_{t_c+k}^2\frac{\sigma^2}{KB}+\frac{r}{K}\sum_{j\in \mathcal{P}_t}\sum_{k=1}^r\eta_{t_c+k}^2\|\mathbf{g}_{j}^{(t_c+k)}\|^2\Big)\nonumber\\
 &={2}\Big(\Big[\left(\frac{K+1}{K}\right)\sum_{j\in \mathcal{P}_t}\sum_{k=1}^{r}\eta_{t_c+k}^2C_1\|\mathbf{g}_{j}^{(t_c+k)}\|^2+\sum_{k=1}^r\eta_{t_c+k}^2\frac{(K+1)\sigma^2}{KB}\nonumber\\
 &\quad+{r\left(\frac{K+1}{K}\right)}\sum_{j\in \mathcal{P}_t}\sum_{k=1}^{r}\eta_{t_c+k}^2\|\mathbf{g}_{j}^{(t_c+k)}\|^2\Big)\nonumber\\
 &=2\Big(\Big[\left(\frac{K+1}{K}\right)(C_1+r)\Big]\sum_{j\in \mathcal{P}_t}\sum_{k=1}^{r}\eta_{t_c+k}^2\|\mathbf{g}_{j}^{(t_c+k)}\|^2+\sum_{k=1}^r\eta_{t_c+k}^2\frac{(K+1)\sigma^2}{KB}\Big)\nonumber\\
 &\stackrel{\text{\ding{192}}}{\leq} 2\left(\frac{K+1}{K}\right)\Big(\Big[C_1+E\Big]\big(\sum_{k=t_c+1}^{t-2}\sum_{j\in \mathcal{P}_t}\eta_{k}^2\|\mathbf{g}_{j}^{(k)}\|^2\sum_{j\in \mathcal{P}_t}\eta_{t-1}^2\|\mathbf{g}_{j}^{(t-1)}\|^2\big)+\sum_{k=t_c+1}^{t-1}\eta_{k}^2\frac{\sigma^2}{B}\Big),\label{eq:local-sum-time}
\end{align}
where \ding{192} follows from $r\leq E$. Finally, Eq.~(\ref{eq:local-sum-time}) leads to
\begin{align}
    \mathbb\sum_{j\in \mathcal{P}_t}\|\bar{\boldsymbol{w}}^{(t)}-\boldsymbol{w}_j^{(t)}\|^2&\leq 2\left(\frac{K+1}{K}\right)\Big([C_1+E]\sum_{k=t_c+1}^{t-1}\eta_{k}^2\sum_{j\in \mathcal{P}_t}\|\
\nabla{f}_j(\boldsymbol{w}_j^{(k)})\|^2+\sum_{k=t_c+1}^{t-1}\eta_{k}^2\frac{\sigma^2}{B}\Big)\label{eq:tbf}.
\end{align}
Now by applying Fact~\ref{fact:1} on both sides of Eq.~(\ref{eq:tbf}) and using the upperbound over the weighted gradient diversity the proof is concluded.}

\subsection{Proof of Lemma~\ref{lemm:main-seq}}

{\textcolor{black}{We start by deriving the conditions on the learning rate that allows us to make the coefficients $D_t$ and $B_t$ negative and consequently cancel out the contribution of the terms including $\|\sum_{j=1}^pq_j\nabla{f}(\boldsymbol{w}^{(t)})\|^2$ from upper bound.}} 
\begin{remark}
Before proceeding to the proof, note that in this Lemma since $c_t$ only influences $a_t$ and is independent from the $e_t$ and its co-efficient, if Lemma holds for $c_t=0$, it will hold also for the case of $c_t\neq 0$. Therefore, without loss of generality we prove this Lemma for the case of $c_t=0$.
\end{remark}

We have:
\begin{align}
     a_{t+1}\leq (1-\mu\eta_t)a_t+\frac{\eta_t}{2}\big(-1+D\eta_t\big)e_t+B_t\sum_{k=t_c+1}^{t-1}\eta_k^2e_k,\label{eq:fst-cond-gd}
\end{align}

Now we can derive the first condition over learning rate where the bound in (\ref{eq:fst-cond-gd}) reduces to:
\begin{align}
        &\stackrel{\text{\ding{193}}}{\leq} (1-\eta_t\mu)a_t+B_t\sum_{k=t_c+2}^{t-1}\eta_k^2e_k\nonumber\\
    &{=}(1-\eta_t\mu) a_t+B_t\Big(\sum_{k=t_c+1}^{t-2}\eta_k^2e_k+\eta_{t-1}^2e_{t-1}\Big),\label{eq:first-iteration}
\end{align}

where \ding{193} follows from $\Delta_t\triangleq1-\mu\eta_t$ and the choice of 
\begin{align}\eta_t\leq \frac{1}{D},
\end{align}

 With this notation in mind, we continue from (\ref{eq:first-iteration}) as follows:
\begin{align}
    a_{t+1}&\leq\Delta_t a_t
    +B_t\Big(\sum_{k=t_c+1}^{t-2}\eta_k^2e_k+\eta_{t-1}^2e_{t-1}\Big)\nonumber\\
    &\leq \Delta_t\Big[\Delta_{t-1} a_t+\frac{\eta_{t-1}}{2}(-1+\eta_{t-1} D\big)e_{t-1}+{B}_{t-1}\Big(\sum_{k=t_c+1}^{t-3}\eta_k^2e_k+\eta_{t-2}^2e_{t-2}\Big)\Big]+{B}_t\Big(\sum_{k=t_c+1}^{t-2}\eta_k^2e_k+\eta^2_{t-1}e_{t-1}\Big)\nonumber\\
    &={\Delta}_t{\Delta}_{t-1}a_{t-1}+\frac{\eta_{t-1}\Delta_t}{2}\Big[-1+D\eta_{t-1}+\frac{2C\eta_{t-1}{B}_t}{ \Delta_{t}}\Big]e_{t-1}+\left(\Delta_t{B}_{t-1}+B_t\right)\Big(\sum_{k=t_c+1}^{t-3}\eta_k^2e_k+\eta_{t-2}^2e_{t-2}\Big),\label{eq:snd-cnd}
\end{align}
Now the bound in (\ref{eq:snd-cnd}) under condition $-1+D\eta_{t-1}+\frac{2\eta_{t-1}{B}_t}{ \Delta_{t}}\leq 0$ or equivalently 
$$\eta_{t-1}\leq\frac{1}{D+2\frac{B_t}{\Delta_{t}}},$$ gives us the following bound:
\begin{align}
    &{\leq} \Delta_{t}\Delta_{t-1}a_{t-1}+\big[\Delta_{t}{B}_{t-1}+{B}_t\big]\Big(\sum_{k=t_c+1}^{t-3}\eta_k^2e_k+\eta_{t-2}^2e_{t-2}\Big)\nonumber\\
    &\leq {\Delta_{t}\Delta_{t-1}}\Big[{\Delta}_{t-2}a_{t-2}+\frac{\eta_{t-2}}{2}\Big[-1+D\eta_{t-2}\Big]e_{t-2}+{B}_{t-2}\Big(\sum_{k=t_c+1}^{t-4}\eta_k^2e_k+\eta_{t-3}^2e_{t-3}\Big)\Big]\nonumber\\
    &\:+\big[\Delta_t{B}_{t-1}+{B}_t\big]\Big(\sum_{k=t_c+1}^{t-3}\eta^2_k e_k+\eta_{t-2}^2e_{t-2}\Big)\nonumber\\
    &=\Delta_{t}\Delta_{t-1}\Delta_{t-2}a_{t-2}+\frac{\Delta_{t}\Delta_{t-1}\eta_{t-2}}{2}\Big[-1+D\eta_{t-2}+\frac{2\eta_{t-2}}{\Delta_{t}\Delta_{t-1}}\big[\Delta_t {B}_{t-1}+{B}_t\big]\Big]e_{t-2}\nonumber\\
    &+\Big[\Delta_{t}\Delta_{t-1}{B}_{t-2}+\Delta_t{B}_{t-1}+{B}_t\Big]\Big(\sum_{k=t_c+1}^{t-4}\eta_k^2e_k+\eta_{t-3}^2e_{t-3}\Big)\nonumber\\
    &\stackrel{\text{\ding{193}}}{\leq} \Delta_{t}\Delta_{t-1}\Delta_{t-2}a_{t-2}+\Big[\Delta_{t}\Delta_{t-1}{B}_{t-2}+\Delta_{t}{B}_{t-1}+{B}_t\Big]\Big(\sum_{k=t_c+1}^{t-4}\eta_k^2e_k+\eta_{t-3}^2e_{t-3}\Big),\label{eq:induction-step-2-unb}
\end{align}

where \ding{193} follows from $-1+D\eta_{t-2}+\frac{2\eta_{t-2}}{\Delta_{t}\Delta_{t-1}}\big[\Delta_t {B}_{t-1}+{B}_t\big]\leq 0$ or equivalently
\begin{align}
    \eta_{t-2}&\leq\frac{1}{D+\frac{2}{\Delta_{t}\Delta_{t-1}}\big[\Delta_t {B}_{t-1}+{B}_t\big]}.
\end{align}

By induction on (\ref{eq:induction-step-2-unb}) we get:

\begin{align}
a_{t+1}&{\leq}\Pi_{i=t_c+1}^{t}\Delta_ia_{t_c+1}+\frac{\Pi_{i=t_c+2}^t\Delta_i\eta_{t_c+1}}{2}\Big[-1+D\eta_{t_c+1}+\frac{2\eta_{t_c+1}}{\Pi_{i=t_c+2}^t\Delta_i}\big[\Pi_{i=t_c+3}^t\Delta_i{B}_{t_c+2}\nonumber\\
    &\qquad\qquad\qquad+\ldots+\Delta_t{B}_{t-1}+{B}_t\big]\Big]e_{t_c+1}\nonumber\\
    &\stackrel{\text{\ding{192}}}{\leq}\Pi_{i=t_c+1}^{t}\Delta_ia_{t_c+1},\label{eq:induction-step-4}
\end{align}
where \ding{192} follows from
\begin{align}
0&\geq     -1+{D\eta_{t_c+1}}+\frac{2\eta_{t_c+1}}{\Pi_{i=t_c+2}^t\Delta_i}\big[\Pi_{i=t_c+3}^t\Delta_i{B}_{t_c+2}+\ldots+\Delta_t{B}_{t-1}+B_t\big]
\end{align}
which leads to 
\begin{align}
    \eta_{t_c+1}&\leq \frac{1}{D+\frac{2}{\Pi_{i=t_c+2}^t\Delta_i}\big[\Pi_{i=t_c+3}^t\Delta_i{B}_{t_c+2}+\ldots+\Delta_t{B}_{t-1}+B\big]}.\label{eq:bnding-eta}
\end{align}


{\subsection{Proof of Lemma~\ref{lemmba:choice-of-learning-rate-pl}}
Before proceeding to the proof of next Lemma, according to the condition derived in Lemma~\ref{lemm:main-seq}, we would like to highlight the fact that 
\begin{align}
\frac{1}{\lambda L(C_1+K)}&\leq\frac{1}{\lambda L(C_1+K)+\frac{2K{B}_t}{ \Delta_t}}\nonumber\\
&\vdots\nonumber\\
&\leq\frac{1}{\lambda L(C_1+K)+\frac{2K}{\Pi_{i=t_c+2}^t\Delta_i}\big[\Pi_{i=t_c+3}^t\Delta_i{B}_{t_c+2}+\ldots+\Delta_t{B}_{t-1}+B\big]} 
\end{align}
Therefore, for the proof of Lemma~\ref{lemmba:choice-of-learning-rate-pl}, we focus on the minimum quantity of term $$\frac{1}{\lambda L(C_1+K)+\frac{2K}{\Pi_{i=t_c+2}^t\Delta_i}\big[\Pi_{i=t_c+3}^t\Delta_i{B}_{t_c+2}(E)+\ldots+\Delta_t{B}_{t-1}(E)+B\big]} .$$

In the following, we show that the imposed conditions on the learning rate are satisfied for all of the iterations.
We use some properties over the learning rate related quantities as follows:
\begin{enumerate}
    \item[1)] $\eta_{t_1}>\eta_{t_2}$ if $t_1<t_2$.
    \item[2)] $\Delta_{t_1}<\Delta_{t_2}$ if $t_1<t_2$.
    \item[3)] ${B}_{t_1}>{B}_{t_2}$  if $t_1<t_2$.
\end{enumerate}

Using these properties, we have:
\begin{align}
  &\frac{1}{\lambda L(C_1+K)+\frac{2K}{\Pi^{t}_{i=t_c+2}\Delta_i}\big[\Pi^{t}_{i=t_c+3}\Delta_i{B}_{t_c+2}+\ldots+\Delta_t{B}_{t-1}+{B}_t\big]}\nonumber\\
  &\qquad\stackrel{\text{\ding{192}}}{\geq}\frac{1}{\lambda L(C_1+K)+\frac{2K}{\Pi_{i=t}^{t_c+2}\Delta_i}\big[\Pi_{i=t}^{t_c+3}\Delta_i{B}_{1}+\ldots+\Delta_t{B}_{1}+{B}_1\big]}\nonumber\\
  &\qquad=\frac{1}{\lambda L(C_1+K)+\frac{2K}{\Pi_{i=t}^{t-E+2}\Delta_i}{B}_{1}\big[\Pi_{i=t}^{t_c+3}\Delta_i+\ldots+\Delta_t+1\big]}\nonumber\\
  &\qquad\stackrel{\text{\ding{193}}}{\geq}\frac{1}{\lambda L(C_1+K)+\frac{2K}{\Pi_{i=t}^{t_c+2}\Delta_i}{B}_{1}\big[\Pi_{i=t}^{t_c+3}\Delta_T+\ldots+\Delta_T+1\big]}\nonumber\\
  &\qquad\stackrel{\text{\ding{194}}}{\geq}\frac{1}{\lambda L(C_1+K)+\frac{2K}{\Pi_{i=t}^{t_c+2}\Delta_1}{B}_{1}\big[\Pi_{i=t}^{t_c+3}\Delta_T+\ldots+\Delta_T+1\big]}\nonumber\\
  &\qquad\stackrel{\text{\ding{195}}}{\geq}\frac{1}{\lambda L(C_1+K)+\frac{2K}{\Delta^{E-1}_1}{B}_{1}\big[\Delta^{E-2}_T+\ldots+\Delta_T+1\big]}\nonumber\\
  &\qquad\stackrel{\text{\ding{196}}}{\geq}\frac{1}{\lambda L(C_1+K)+\frac{2K}{\Delta^{E-1}_1}{B}_{1}\big[E-1\big]}\nonumber
\end{align}

 \ding{192} is due to item (3), \ding{193} comes from property item (2) and finally \ding{194} holds because of property item (2), \ding{195} follows from $t-(t_c+3)\leq E-2$, and \ding{196} follows from $\Delta_T\leq 1$.

Next, we show that for the choice of $a=\alpha E+4$ where $\alpha \exp{(-\frac{2}{\alpha})}<\kappa\sqrt{192\lambda\left(\frac{K+1}{K}\right)}$ the conditions hold. To this end, we have 
\begin{align}
    \eta_t&=\frac{4}{\mu(t+a)}\nonumber\\
    &\leq\eta_1\nonumber\\
    &\leq \frac{1}{\lambda L(C_1+K)+\frac{2K}{\Delta^{E-1}_1}{B}_{1}\big[E-1\big]}\nonumber\\
    &=\frac{\Delta^{E-1}_1}{\Delta^{E-1}_1\lambda L(C_1+K)+2K{B}_{1}\big[E-1\big]}\nonumber\\
    &=\frac{\big({\frac{1+a-4}{a+1}\big)}^{E-1}}{\big({\frac{1+a-4}{a+1}\big)}^{E-1}L(C_1+K)+2K{B}_{1}\big[E-1\big]}\nonumber\\
    &=\frac{\big({\frac{1+a-4}{a+1}\big)}^{E-1}}{\big({\frac{1+a-4}{a+1}\big)}^{E-1}\lambda L(C_1+K)+2K\big(\frac{4L^2\left(\frac{K+1}{K}\right)(C_1+E)}{\mu K(a+1)}\big)(E-1)}\nonumber\\
    &=\frac{{(a-3)}^{E-1}}{{(a-3)}^{E-1}\lambda L(C_1+K)+(\frac{K+1}{K})\frac{8\lambda L^2}{\mu}\big(C_1(E-1)+(E-1)E\big)(a+1)^{E-2}},\label{eq:main-condition}
\end{align}

{From (\ref{eq:main-condition}), we have:
\begin{align}
  4\lambda {(a-3)}^{E-1}L(C_1+K)&+\lambda\frac{32L^2}{\mu}\left(\frac{K+1}{K}\right)\big(C_1(E-1)+E(E-1)\big)(a+1)^{E-2}\nonumber\\
  &\stackrel{\text{\ding{192}}}{\leq} \frac{192\lambda  L^2}{\mu^2}\left(\frac{K+1}{K}\right)(E-1)E(a+1)^{E-2}\nonumber\\
  &\leq \mu\big[(t+a)(a-3)\big]{(a-3)}^{E-2},\label{eq:m-c-o-c} 
\end{align}
where \ding{192} follows from the fact that ${(a-3)}^{E-1}L(C_1+K)\leq \frac{16 L^2}{\mu}E(E-1)(a+1)^{E-2}$ and $\frac{32L^2}{\mu}C_1\left(\frac{K+1}{K}\right)(E-1)(a+1)^{E-2}\leq (\frac{K+1}{K})\frac{64L^2}{\mu}(E-1)^2(a+1)^{E-2}$.}

Letting $a=\alpha E+4$ and to analyze the worse case set $t=1$ in inequality (\ref{eq:m-c-o-c}) which leads to the following condition:
\begin{align}
   \frac{\alpha^2E^2+6\alpha E+5}{192\lambda\left(\frac{K+1}{K}\right)\frac{L^2}{\mu^2}E(E-1)}&\leq \left(\frac{a+1}{a-3}\right)^{E-2}\nonumber\\
   &=\left(1+\frac{4}{a-3}\right)^{E-2}\nonumber\\
   &=\left(1+\frac{4}{\alpha E+4-3}\right)^{E-2}\nonumber\\
   &\leq e^{\frac{4(E-2)}{\alpha E+1}}\nonumber\\
   &\stackrel{\text{\ding{192}}}{\leq} e^{\frac{4}{\alpha}},\label{eq:final-chain-of-condition}
\end{align}
where \ding{192} follows from the property that $\frac{E-2}{\alpha E+1}$ is non-decreasing with respect to $E$. From (\ref{eq:final-chain-of-condition}) we get our condition over $\alpha$ as follows:

\begin{align}
\left(\lambda\left(\frac{K+1}{K}\right)192\kappa^2e^{\frac{4}{\alpha}}-\alpha^2\right)E^2-\left(\lambda\left(\frac{K+1}{K}\right)192\lambda\kappa^2e^{\frac{4}{\alpha}}+6\alpha\right)E-5\geq0
\end{align}
which implies first $\frac{\alpha}{e^{\frac{2}{\alpha}}}\leq \kappa \sqrt{192\lambda\left(\frac{K+1}{K}\right)}$ and second \begin{align}
    E\geq \frac{\left(\lambda\left(\frac{K+1}{K}\right)192\kappa^2e^{\frac{4}{\alpha}}+6\alpha\right)+\sqrt{{\left(\lambda\left(\frac{K+1}{K}\right)192\kappa^2e^{\frac{4}{\alpha}}+6\alpha\right)}^2+20\left(\lambda\left(\frac{K+1}{K}\right)192\kappa^2e^{\frac{4}{\alpha}}-\alpha^2\right)}}{2\left(\lambda\left(\frac{K+1}{K}\right)192\kappa^2e^{\frac{4}{\alpha}}-\alpha^2\right)}\label{eq:cnd-alpha}
\end{align} which holds for $a=\alpha E+4$ and $E=O\left(\frac{T^{\frac{2}{3}}}{K^{\frac{1}{3}}}\right)$.
Note that using the inequality $\sqrt{a^2+b}\leq a+\sqrt{b}$ we can upper bound right-hand side of~(\ref{eq:cnd-alpha}) by
\begin{align}
     &\frac{\left(\lambda\left(\frac{K+1}{K}\right)192\kappa^2e^{\frac{2}{\alpha}}+6\alpha\right)+\sqrt{{\left(\lambda\left(\frac{K+1}{K}\right)192\kappa^2e^{\frac{4}{\alpha}}+6\alpha\right)}^2+20\left(\lambda\left(\frac{K+1}{K}\right)192\kappa^2e^{\frac{4}{\alpha}}-\alpha^2\right)}}{2\left(\lambda(\frac{K+1}{K})192\kappa^2e^{\frac{4}{\alpha}}-\alpha^2\right) }\\
     &\quad\leq \frac{2\left(\left(\lambda\frac{K+1}{K}\right)192\kappa^2e^{\frac{4}{\alpha}}+6\alpha\right)+\sqrt{20\left(\lambda\left(\frac{K+1}{K}\right)192\kappa^2e^{\frac{4}{\alpha}}-\alpha^2\right)}}{2\left(\lambda\left(\frac{K+1}{K}\right)192\kappa^2e^{\frac{4}{\alpha}}-\alpha^2\right)}\nonumber\\
     &\quad=1+\frac{\alpha^2+6\alpha}{\lambda\left(\frac{K+1}{K}\right)192\kappa^2e^{\frac{4}{\alpha}}-\alpha^2}+\frac{\sqrt{5}}{\sqrt{\lambda\left(\frac{K+1}{K}\right)192\kappa^2e^{\frac{4}{\alpha}}-\alpha^2}}
\end{align}
}

\subsection{Proof of Lemmas~\ref{lemma:tasbihb44}}

Now letting $\zeta{(t)}\triangleq \mathbb{E}[f(\bar{\boldsymbol{w}}^{(t)})-f^*]$ and multiplying both sides of (\ref{eq:denoised-iteratt}) with $(t+b+1)^2$ we get:
\begin{align}
    \zeta(t+1)\leq \Delta_t\zeta(t)+{c}_t
\end{align}
Next, by defining $z_t\triangleq (t+a)^2$ similar to  \cite{stich2018local}, we have
\begin{align}
    \Delta_t\frac{z_t}{\eta_t}=(1-\mu\eta_t)\mu\frac{(t+a)^3}{4}=\frac{\mu(a+t-4)(a+t)^2}{4}\leq \mu \frac{(a+t-1)^3}{4}=\frac{z_{t-1}}{\eta_{t-1}}\label{eq:seb-ineq}
\end{align}

Now by multiplying both sides of (\ref{eq:final-iter}) with $\frac{z_t}{\eta_t}$ we have:
\begin{align}
    \frac{z_t}{\eta_t}\zeta(t+1)&\leq \zeta(t)\Delta_t\frac{z_t}{\eta_t}+\frac{z_t}{\eta_t}{c}_t\nonumber\\
    &\stackrel{\text{\ding{192}}}{\leq} \zeta(t)\frac{z_{t-1}}{\eta_{t-1}}+\frac{z_t}{\eta_t}{c}_t,\label{eq:final-iter}
\end{align}
where \ding{192} follows from (\ref{eq:seb-ineq}). Next iterating over (\ref{eq:final-iter}) leads to the following bound:
\begin{align}
    \zeta(T)\frac{z_{T-1}}{\eta_{T-1}}&\leq (1-\mu \eta_0)\frac{z_{0}}{\eta_{0}}\zeta(0)+\sum_{k=0}^{T-1}\frac{z_k}{\eta_k}{c}_k\nonumber\\
    \label{eq:l-t-final-step}
\end{align}
Final step in proof is to bound $\sum_{k=0}^{T-1}\frac{z_k}{\eta_k}{c}_k$ as follows:
{\begin{align}
  \sum_{k=0}^{T-1}\frac{z_k}{\eta_k}{c}_{k} &= \frac{\mu}{4}\sum_{k=0}^{T-1}(k+a)^3 \Big(\frac{L\eta_k^2 \sigma^2}{2pB}+\frac{\eta_k L^2}{p}\big(\sum_{k=t_c+1}^{k-1}\eta_{k}^2\frac{(p+1)\sigma^2}{pB}\big)\Big)\nonumber\\
  &\stackrel{\text{\ding{192}}}{\leq} \frac{\mu}{4}\sum_{k=0}^{T-1}(k+a)^3 \Big(\frac{L\eta_k^2\sigma^2}{2pB}+\frac{\eta_kL^2}{p}\eta^2_{\big(\floor{\frac{k}{E}}E\big)}(E-1)\frac{\sigma^2}{b}(\frac{p+1}{p})\Big)\nonumber\\
  &=\frac{L\sigma^2\mu}{8pB}\sum_{k=0}^{T-1}(k+a)^3\eta_k^2+\frac{L^2\frac{\sigma^2}{B}(p+1)(E-1)\mu}{4p^2}\sum_{k=0}^{T-1}(k+a)^3\eta_k\eta^2_{\big(\floor{\frac{k}{E}}E\big)},\label{eq:bounding-e-fin}
\end{align}}

\ding{192} is due to fact that $\eta_t$ is non-increasing. 

Next we bound two terms in (\ref{eq:bounding-e-fin}) as follows:

\begin{align}
    \sum_{k=0}^{T-1}(k+a)^3\eta_k^2 &=\sum_{k=0}^{T-1}(k+a)^3\frac{16}{\mu^2 (k+a)^2}\nonumber\\
    &=\frac{16}{\mu^2}\sum_{k=0}^{T-1}(k+a)\nonumber\\
    &= \frac{16}{\mu^2}\left(\frac{T(T-1)}{2}+aT\right)\nonumber\\
    &\leq \frac{8T(T+2a)}{\mu^2},
\end{align}
and similarly we have:
\begin{align}
  \sum_{k=0}^{T-1}(k+a)^3\eta_k\eta^2_{\big(\ceil{\frac{k}{E}}E\big)}&= \frac{64}{\mu^3}\sum_{k=0}^{T-1}(k+a)^3\frac{1}{k+a}\left(\frac{1}{\floor{\frac{k}{E}}E+a}\right)^2\nonumber\\
  &\stackrel{\text{\ding{192}}}{\leq} \frac{64}{\mu^3}\sum_{k=0}^{T-1}\left(\frac{k+a}{\floor{k+a}}\right)^2\nonumber\\
  &\stackrel{\text{\ding{193}}}{\leq}\frac{256}{\mu^3}T, 
\end{align}
where \ding{192} follows from $\floor{\frac{k}{E}}E+a\geq \floor{k+a}$ and \ding{193} comes from the fact that $\frac{n}{\floor{n}}\leq 2$ for any integer $n>0$.

Now, we get:

{\begin{align}
    \sum_{k=0}^{T-1}\frac{z_k}{\eta_k}{c}_{k-1}(k)&\leq \frac{L\sigma^2\mu}{8pB}(\frac{8T(T+2a)}{\mu^2})+\frac{L^2\frac{\sigma^2}{b}(p+1)(E-1)\mu}{4p^2}(\frac{256}{\mu^3}T)\nonumber\\
    &=\frac{L\sigma^2T(T+2a)}{pB\mu}+\frac{64L^2\sigma^2T(E-1)}{pB\mu^2}\nonumber\\
    &=\frac{\kappa \sigma^2T(T+2a)}{pB}+\frac{64\kappa^2\sigma^2T(E-1)}{pB},\label{eq:t-p-in-f-b}
\end{align}}
Then, the upper bound becomes as follows:
{\begin{align}
        \zeta(T)\frac{z_{T-1}}{\eta_{T-1}}&=\mathbb{E}\big[f(\bar{\boldsymbol{w}}^{(t)})-f^*\big]\frac{\mu(T+a)^3}{4}\nonumber\\
        &\leq (1-\mu \eta_0)\frac{z_{T-1}}{\eta_{T-1}}\zeta(0)+\sum_{k=0}^{T-1}\frac{z_k}{\eta_k}{c}_k\nonumber\\
        &\leq (1-\mu \eta_0)\frac{z_{0}}{\eta_{0}}\zeta(0)+\frac{\kappa \frac{\sigma^2}{b}T(T+2a)}{pB}+\frac{64\kappa^2\sigma^2T(E-1)}{pB}\nonumber\\
        &\leq \frac{\mu a^3}{4}\mathbb{E}\big[f(\bar{\boldsymbol{w}}^{(0)})-f^*\big]+\frac{\kappa \sigma^2T(T+2a)}{pB}+\frac{64\kappa^2\sigma^2T(E-1)}{pB}, \label{eq:v-final}
\end{align}}
Finally, from (\ref{eq:v-final}) we conclude:
{\begin{align}
   \mathbb{E}\big[f(\bar{\boldsymbol{w}}^{(t)})-f^*\big]&\leq \frac{a^3}{(T+a)^3}\mathbb{E}\big[f(\bar{\boldsymbol{w}}^{(0)})-f^*\big]+ \frac{4\kappa \sigma^2T(T+2a)}{\mu pB(T+a)^3}+\frac{256\kappa^2\sigma^2T(E-1)}{\mu pB(T+a)^3}
\end{align}}

\subsection{Proof of Lemma \ref{lemma:dif-b-u-x}}
Recalling $t_c\triangleq \floor{\frac{t}{E}}E$ and 
 $\bar{\boldsymbol{w}}^{(t_c+1)}=\frac{1}{K}\sum_{j\in\mathcal{P}_t} \boldsymbol{w}_j^{(t_c+1)}$, the local solution at $j$th machine at any particular iteration $t > t_c$  can be written  as:
 \begin{align}
 \boldsymbol{w}_j^{(t)}=  \boldsymbol{w}_j^{(t-1)}-\eta\tilde{\mathbf{g}}_{j}^{(t-1)}\stackrel{\text{\ding{192}}}{=}\boldsymbol{w}_j^{(t-2)}-\Big[\eta\tilde{\mathbf{g}}_{j}^{(t-2)}+\eta\tilde{\mathbf{g}}_{j}^{(t-1)}\Big]=\bar{\boldsymbol{w}}^{(t_c+1)}-\sum_{k=t_c+1}^{t-1}\eta\tilde{\mathbf{g}}_{j}^{(k)}, \label{eq:j-model-update}
 \end{align}
 where \ding{192} follows from the update rule of local solutions. Now, from (\ref{eq:j-model-update}) we compute the average virtual model at $t$th iteration as follows:
\begin{align}
    \bar{\boldsymbol{w}}^{(t)}=\bar{\boldsymbol{w}}^{(t_c+1)}-\frac{1}{K}\sum_{j\in\mathcal{P}_t}\sum_{k=t_c+1}^{t-1}\eta\tilde{\mathbf{g}}_{j}^{(k)}\label{eq:model-avg-time-t}
\end{align}
First, without loss of generality, suppose $t=s_tE+r$ where $s_t$ and $r$ denotes the indices of communication round and local updates, respectively.

Next consider that for $t_c+1 < t\leq t_c+E$, $\mathbb{E}_t\|\bar{\boldsymbol{w}}^{(t)}-\boldsymbol{w}_j^{(t)}\|^2$  does not depend on time $t\leq t_c$ for $1\leq j\leq p$. Therefore, for all iterations $0\leq t\leq T-1$ we can write: 
\begin{align}
\label{eqn:c1:overall}
    \frac{1}{T}\sum_{t=0}^{T-1}\sum_{j=1}^p\mathbb{E}\|\bar{\boldsymbol{w}}^{({t})}-\boldsymbol{w}_j^{({t})}\|^2 & =\frac{1}{T}\sum_{s_t=0}^{\frac{T-1}{E}-1}\sum_{r=1}^{E}\sum_{j=1}^p\mathbb{E}\|\bar{\boldsymbol{w}}^{({s_tE+r})}-\boldsymbol{w}_j^{({s_tE+r})}\|^2
\end{align}

We bound the term $\mathbb{E}\|\bar{\boldsymbol{w}}^{(t)}-\boldsymbol{w}_l^{({t})}\|^2$ for $t_c+1 \leq  {t}=s_tE+r\leq t_c+ E$ in three steps: (1) We first relate this quantity to the variance between stochastic gradient and full gradient, (2) We use Assumption~\ref{Ass:1} on unbiased estimation and i.i.d sampling, (3) We use Assumption~\ref{Ass:2} to bound the final terms. 

In what follows, we proceed to the details of each of these three steps. \\

 \noindent \textbf{Step 1: Relating to variance}
\begin{align}
 \mathbb{E}\|\bar{\boldsymbol{w}}^{(s_tE+r)}-&\boldsymbol{w}_l^{(s_tE+r)}\|^2 =\mathbb{E}\|\bar{\boldsymbol{w}}^{(t_c+1)}-\Big[\sum_{k=t_c+1}^{{t-1}}\eta\tilde{\mathbf{g}}_{l}^{(k)}\Big]-\bar{\boldsymbol{w}}^{(t_c+1)}+\Big[\frac{1}{K}\sum_{j\in\mathcal{P}_t}\sum_{k=t_c+1}^{{t-1}}\eta\tilde{\mathbf{g}}_{j}^{(k)}\Big]\|^2\nonumber\\
 &\stackrel{\text{\ding{192}}}{=}\mathbb{E}\|\sum_{k=1}^{r}\eta\tilde{\mathbf{g}}_{l}^{(s_tE+k)}-\frac{1}{K}\sum_{j\in\mathcal{P}_t}\sum_{k=1}^{r}\eta\tilde{\mathbf{g}}_{j}^{(s_tE+k)}\|^2\nonumber\\
 &\stackrel{\text{\ding{193}}}{\leq} 2\Big[\mathbb{E}\|\sum_{k=1}^{r}\eta\tilde{\mathbf{g}}_{l}^{(s_tE+k)}\|^2+\mathbb{E}\|\frac{1}{K}\sum_{j\in\mathcal{P}_t}\sum_{k=1}^{r}\eta\tilde{\mathbf{g}}_{j}^{(s_tE+k)}\|^2\Big]\nonumber\\
 &\stackrel{\text{\ding{194}}}{=}2\Big[\mathbb{E}\|\sum_{k=1}^{r}\eta\tilde{\mathbf{g}}_{l}^{(s_tE+k)}-\mathbb{E}\big[\sum_{k=1}^{r}\eta\tilde{\mathbf{g}}_{l}^{(s_tE+k)}\big]\|^2+\|\mathbb{E}\big[\sum_{k=1}^{{r}}\eta\tilde{\mathbf{g}}_{l}^{(s_tE+k)}\big]\|^2\nonumber\\
 &+\mathbb{E}\|\frac{1}{K}\sum_{j\in\mathcal{P}_t}\sum_{k=1}^{{r}}\eta\tilde{\mathbf{g}}_{j}^{(s_tE+k)}-\mathbb{E}\big[\frac{1}{K}\sum_{j\in\mathcal{P}_t}\sum_{k=1}^{r}\eta\tilde{\mathbf{g}}_{j}^{(s_tE+k)}\big]\|^2\Big]\nonumber\\
 &\qquad+\|\mathbb{E}\big[\frac{1}{K}\sum_{j\in\mathcal{P}_t}\sum_{k=1}^{r}\eta\tilde{\mathbf{g}}_{j}^{(s_tE+k)}\big]\|^2\nonumber\\
 &\stackrel{\text{\ding{195}}}{=}{2}\mathbb{E}\Big(\Big[\|\sum_{k=1}^{r}\eta\Big[\tilde{\mathbf{g}}_{l}^{(s_tE+k)}-\mathbf{g}_{l}^{(s_tE+k)}\Big]\|^2+\|\sum_{k=1}^{r}\eta\mathbf{g}_{l}^{(s_tE+k)}\|^2\Big]\nonumber\\
 &\quad+\|\frac{1}{K}\sum_{j\in\mathcal{P}_t}\sum_{k=1}^{r}\eta\Big[\tilde{\mathbf{g}}_{j}^{(s_tE+k)}-\mathbf{g}_{j}^{(s_tE+k)}\Big]\|^2+\|\frac{1}{K}\sum_{j\in\mathcal{P}_t}\sum_{k=1}^{r}\eta\mathbf{g}_{j}^{(s_tE+k)}\|^2\Big),\nonumber\\
\end{align}

where \ding{192} holds because ${t}=s_tE+r\leq t_c+ E$, \ding{193} is due to $\|\mathbf{a}-\mathbf{b}\|^2\leq 2(\|\mathbf{a}\|^2+\|\mathbf{b}\|^2)$, \ding{194} comes from $\mathbb{E}[\boldsymbol{w}^2]=\mathbb{E}[[\boldsymbol{w}-\mathbb{E}[\boldsymbol{w}]]^2]+\mathbb{E}[\boldsymbol{w}]^2$, \ding{195} comes from unbiased estimation Assumption~\ref{Ass:1}.\\

\noindent\textbf{Step 2: Unbiased estimation and i.i.d. sampling}

\begin{align}
    &{=}{2}\mathbb{E}\Big(\Big[\sum_{k=1}^{r}\eta^2\|\tilde{\mathbf{g}}_{l}^{(s_tE+k)}-\mathbf{g}_{l}^{(s_tE+k)}\|^2\nonumber\\&+\sum_{w\neq z \vee l\neq v}\Big\langle\eta\tilde{\mathbf{g}}_{l}^{(w)}-\eta\mathbf{g}_{l}^{(w)},\eta\tilde{\mathbf{g}}_{v}^{(z)}-\eta\mathbf{g}_{v}^{(z)}\Big\rangle+\|\sum_{k=1}^{r}\eta\mathbf{g}_{l}^{(s_tE+k)}\|^2\Big]\nonumber\\
 &\quad+\frac{1}{K^2}\sum_{l\in\mathcal{P}_t}\sum_{k=1}^{r}\eta^2\|\tilde{\mathbf{g}}_{l}^{(s_tE+k)}-\mathbf{g}_{l}^{(s_tE+k)}\|^2\nonumber\\
 &+\frac{1}{K^2}\sum_{w\neq z \vee l\neq v}\Big\langle\eta\tilde{\mathbf{g}}_{l}^{(w)}-\eta\mathbf{g}_{l}^{(w)},\eta\tilde{\mathbf{g}}_{v}^{(z)}-\eta\mathbf{g}_{v}^{(z)}\Big\rangle+\|\frac{1}{K}\sum_{j\in\mathcal{P}_t}\sum_{k=1}^{r}\eta\mathbf{g}_{j}^{(s_tE+k)}\|^2\Big)\nonumber\\
 &\stackrel{\text{\ding{196}}}{=}{2}\mathbb{E}\Big(\Big[\sum_{k=1}^{r}\eta^2\|\tilde{\mathbf{g}}_{l}^{(s_tE+k)}-\mathbf{g}_{l}^{(s_tE+k)}\|^2+\|\sum_{k=1}^{r}\eta\mathbf{g}_{l}^{(s_tE+k)}\|^2\Big]\nonumber\\
 &\quad+\frac{1}{K^2}\sum_{j\in\mathcal{P}_t}\sum_{k=1}^{r}\eta^2\|\tilde{\mathbf{g}}_{j}^{(s_tE+k)}-\mathbf{g}_{j}^{(s_tE+k)}\|^2+\|\frac{1}{K}\sum_{j\in\mathcal{P}_t}\sum_{k=1}^{r}\eta\mathbf{g}_{j}^{(s_tE+k)}\|^2\Big)\nonumber\\
 &\stackrel{\text{\ding{197}}}{\leq}{2}\mathbb{E}\Big(\Big[\sum_{k=1}^{r}\eta^2\|\tilde{\mathbf{g}}_{l}^{(s_tE+k)}-\mathbf{g}_{l}^{(s_tE+k)}\|^2+r\sum_{k=1}^{r}\eta^2\|\mathbf{g}_{l}^{(s_tE+k)}\|^2\Big]\nonumber\\
 &\quad+\frac{1}{K^2}\sum_{j\in\mathcal{P}_t}\sum_{k=1}^{r}\|\tilde{\mathbf{g}}_{j}^{(s_tE+k)}-\mathbf{g}_{j}^{(s_tE+k)}\|^2+\frac{r}{K^2}\sum_{j\in\mathcal{P}_t}\sum_{k=1}^{r}\eta_{s_tE+k}^2\|\mathbf{g}_{j}^{(s_tE+k)}\|^2\Big)\nonumber\\
 &={2}\Big(\Big[\sum_{k=1}^{r}\eta^2\mathbb{E}\|\tilde{\mathbf{g}}_{l}^{(s_tE+k)}-\mathbf{g}_{l}^{(s_tE+k)}\|^2+r\sum_{k=1}^{r}\eta^2\mathbb{E}\|\mathbf{g}_{l}^{(s_tE+k)}\|^2\Big]\nonumber\\
 &\quad+\frac{1}{K^2}\sum_{j\in\mathcal{P}_t}\sum_{k=1}^{r}\eta^2\mathbb{E}\|\tilde{\mathbf{g}}_{j}^{(s_tE+k)}-\mathbf{g}_{j}^{(s_tE+k)}\|^2+\frac{r}{K^2}\sum_{j\in\mathcal{P}_t}\sum_{k=1}^{r}\eta^2\mathbb{E}\|\mathbf{g}_{j}^{(s_tE+k)}\|^2\Big),
\label{eq:four-term-boundingg}
\end{align}
\ding{196} is due to independent mini-batch sampling as well as unbiased estimation assumption, and  \ding{197}  follows from inequality $\|\sum_{i=1}^m\mathbf{a}_i\|^2\leq m\sum_{i=1}^m\|\mathbf{a}_i\|^2$. \\

\noindent \textbf{Step 3: Using Assumption \ref{Ass:2}}

Next step is to bound the terms in (\ref{eq:four-term-boundingg}) using Assumption \ref{Ass:2} as follow:

\begin{align}\mathbb{E}\|\bar{\boldsymbol{w}}^{({t})}-\boldsymbol{w}_l^{({t})}\|^2&\leq {2}\Big(\Big[\sum_{k=1}^{r}\eta^2\Big[C_1\|{\mathbf{g}}_l^{(s_tE+k)})\|^2+{\frac{\sigma^2}{B}}\Big]+r\sum_{k=1}^{r}\eta^2\|\Big[{\mathbf{g}}^{(s_tE+k)}_l\Big]\|^2\Big]\nonumber\\
 &\quad+\frac{1}{K^2}\sum_{j\in\mathcal{P}_t}\sum_{k=1}^{r}\eta^2\Big[C_1\|{\mathbf{g}}^{(s_tE+k)}_j\|^2+{\frac{\sigma^2}{B}}\Big]+\frac{r}{K^2}\sum_{j\in\mathcal{P}_t}\sum_{k=1}^{r}\eta^2\|\Big[{\mathbf{g}}_j^{(s_tE+k)}\Big]\|^2\Big)\nonumber\\
 &={2}\Big(\Big[\sum_{k=1}^{r}\eta^2C_1\|\mathbf{g}_{l}^{(s_tE+k)}\|^2+\sum_{k=1}^{r}\eta^2{\frac{\sigma^2}{B}}+r\sum_{k=1}^{r}\eta^2\|\mathbf{g}_{l}^{(s_tE+k)}\|^2\Big]\nonumber\\
 &\quad+\frac{1}{K^2}\sum_{j\in\mathcal{P}_t}\sum_{k=1}^{r}\eta^2C_1\|\mathbf{g}_{j}^{(s_tE+k)}\|^2+\sum_{k=1}^r\eta^2\frac{\sigma^2}{KB}+\frac{r}{K^2}\sum_{j\in\mathcal{P}_t}\sum_{k=1}^{r}\eta^2\|\mathbf{g}_{j}^{(s_tE+k)}\|^2\Big),\label{eq:var-n-bound}
\end{align}
Now taking expectation over the random selection of workers of (\ref{eq:var-n-bound}), we obtain:
\begin{align}
    \mathbb{E}_{\mathcal{P}_t}\Big[\mathbb{E}\|\bar{\boldsymbol{w}}^{(t)}-\boldsymbol{w}_l^{(t)}\|^2\Big]&\leq2\mathbb{E}_{\mathcal{P}_t}\Big(\Big[\sum_{k=1}^{r}\eta^2C_1\|\mathbf{g}_{l}^{(s_tE+k)}\|^2+\sum_{k=1}^{r}\eta^2{\frac{\sigma^2}{B}}+r\sum_{k=1}^{r}\eta^2\|\mathbf{g}_{l}^{(s_tE+k)}\|^2\Big]\nonumber\\
 &\qquad+\frac{1}{K^2}\sum_{j\in\mathcal{P}_t}\sum_{k=1}^{r}\eta^2C_1\|\mathbf{g}_{j}^{(s_tE+k)}\|^2+\sum_{k=1}^{r}\eta^2\frac{\sigma^2}{KB}+\frac{r}{K^2}\sum_{j\in\mathcal{P}_t}\sum_{k=1}^{r}\eta^2\mathbb{E}\|\mathbf{g}_{j}^{(s_tE+k)}\|^2\Big)\nonumber\\
 &\stackrel{\text{\ding{192}}}{=}2\Big(\Big[\sum_{k=1}^{r}\eta^2C_1\|\mathbf{g}_{l}^{(s_tE+k)}\|^2+\sum_{k=1}^{r}\eta^2{\frac{\sigma^2}{B}}+r\sum_{k=1}^{r}\eta^2\|\mathbf{g}_{l}^{(s_tE+k)}\|^2\Big]\nonumber\\
 &\qquad+\frac{1}{K^2}K\sum_{j=1}^pq_j\sum_{k=1}^{r}\eta^2C_1\|\mathbf{g}_{j}^{(s_tE+k)}\|^2+\sum_{k=1}^{r}\eta^2\frac{\sigma^2}{KB}\\
 &\qquad+\frac{Kr}{K^2}\sum_{j=1}^pq_j\sum_{k=1}^{r}\eta^2\mathbb{E}\|\mathbf{g}_{j}^{(s_tE+k)}\|^2\Big)\label{eq:step}
 \end{align}
where \ding{192} comes from applying Fact~\ref{fact:1}. Now, we upper bound $\sum_{r=1}^{E}\sum_{j=1}^pq_j\mathbb{E}_{\mathcal{P}_t}\Big[\mathbb{E}\|\bar{\boldsymbol{w}}^{(t)}-\boldsymbol{w}_j^{(t)}\|\Big]$ using (\ref{eq:step}) as follows:
 \begin{align}
 \sum_{r=1}^{E}\sum_{j=1}^pq_j\mathbb{E}_{\mathcal{P}_t}\Big[\mathbb{E}\|\bar{\boldsymbol{w}}^{(s_tE+k)}-\boldsymbol{w}_j^{(s_tE+k)}\|\Big]&\leq 2\sum_{r=1}^E\sum_{l=1}^pq_l\Big(\Big[\sum_{k=1}^{r}\eta^2C_1\|\mathbf{g}_{l}^{(s_tE+k)}\|^2+\sum_{k=1}^{r}\eta^2{\frac{\sigma^2}{B}}+r\sum_{k=1}^{r}\eta^2\|\mathbf{g}_{l}^{(s_tE+k)}\|^2\Big]\nonumber\\
 &\qquad+\frac{1}{K^2}K\sum_{j=1}^pq_j\sum_{k=1}^{r}\eta^2C_1\|\mathbf{g}_{j}^{(s_tE+k)}\|^2+\sum_{k=1}^{r}\eta^2\frac{\sigma^2}{KB}\nonumber\\
 &\qquad+\frac{Kr}{K^2}\sum_{j=1}^pq_j\sum_{k=1}^{E}\eta^2\|\mathbf{g}_{j}^{(s_tE+k)}\|^2\Big)\\
 &{=}2\eta^2\sum_{r=1}^E\Big(\Big[\sum_{k=1}^{r}C_1\sum_{l=1}^pq_l\|\mathbf{g}_{l}^{(s_tE+k)}\|^2+\frac{r\sigma^2}{B}+r\sum_{k=1}^{r}\sum_{l=1}^pq_l\|\mathbf{g}_{l}^{(s_tE+k)}\|^2\Big]\nonumber\\
 &\qquad+\frac{1}{K}\sum_{j=1}^pq_j\sum_{k=1}^{r}C_1\|\mathbf{g}_{j}^{(s_tE+k)}\|^2+\frac{(r)\sigma^2}{KB}+\frac{r}{K}\sum_{j=1}^pq_j\sum_{k=1}^{r}\|\mathbf{g}_{j}^{(s_tE+k)}\|^2\Big)\\
 &\stackrel{\text{\ding{192}}}{\leq}2\eta^2\Big(\Big[\sum_{k=1}^{E}C_1\sum_{l=1}^pq_l\|\mathbf{g}_{l}^{(s_tE+k)}\|^2+\frac{E(E+1)\sigma^2}{2B}\nonumber\\
 &\qquad+\frac{E(E+1)}{2}\sum_{k=1}^{E}\sum_{l=1}^pq_l\|\mathbf{g}_{l}^{(s_tE+k)}\|^2\Big]\nonumber\\
 &\qquad+\frac{1}{K}\sum_{j=1}^pq_j\sum_{k=1}^{E}C_1\|\mathbf{g}_{j}^{(s_tE+k)}\|^2+\frac{E(E+1)\sigma^2}{2KB}\nonumber\\&\qquad+\frac{E(E+1)}{2K}\sum_{j=1}^pq_j\sum_{k=1}^{E}\|\mathbf{g}_{j}^{(s_tE+k)}\|^2\Big)\nonumber\\
 &= \frac{\eta^2(K+1)}{K}\Big(\Big[\big(2C_1+E(E+1)\big)\sum_{k=1}^{E}\sum_{j=1}^pq_j\|\mathbf{g}_{j}^{(s_tE+k)}\|^2\Big]+\frac{E(E+1)\sigma^2}{B}\Big),\label{eq:local-sum-timee}
\end{align}
where \ding{192} follows from the fact that the terms $\|\mathbf{g}_l\|^2$ are positive.

Finally, taking summation over communication periods in (\ref{eq:local-sum-timee}) gives:
\begin{align}
    \sum_{s_t=0}^{(T-1)/E-1}\sum_{r=1}^{E}\sum_{j=1}^pq_j\mathbb{E}_{\mathcal{P}_t}\Big[\mathbb{E}\|\bar{\boldsymbol{w}}^{(s_tE+k)}-\boldsymbol{w}_j^{(s_tE+k)}\|\Big]&\leq \frac{\eta^2(K+1)}{K}\Big(\Big[\big(2C_1+E(E+1)\big)\    \sum_{s_t=0}^{(T-1)/E-1}\sum_{k=1}^{E}\sum_{j=1}^pq_j\|\mathbf{g}_{j}^{(s_tE+k)}\|^2\Big]\nonumber\\
    &\quad+\frac{T(E+1)\sigma^2}{2B}\Big)\nonumber\\
    &=\frac{\eta^2(K+1)}{K}\Big(\Big[\big(2C_1+E(E+1)\big)\    \sum_{t=0}^{T-1}\sum_{j=1}^pq_j\|\mathbf{g}_{j}^{(t)}\|^2\Big]+\frac{T(E+1)\sigma^2}{B}\Big)
\end{align}
which leads to
\begin{align}
    \frac{1}{T}\sum_{t=0}^{T-1}\sum_{j=1}^pq_j\mathbb{E}_{\mathcal{P}_t}\Big[\mathbb{E}\|\bar{\boldsymbol{w}}^{(t)}-\boldsymbol{w}_j^{(t)}\|\Big]&\leq \frac{\big(2C_1+E(E+1)\big)}{T}\frac{\eta^2(K+1)}{K}\sum_{t=0}^{T-1}\sum_{j=1}^pq_j\|\mathbf{g}_{j}^{(t)}\|^2+\frac{\eta^2(K+1)(E+1)\sigma^2}{KB}\nonumber\\
    &\stackrel{\text{\ding{192}}}{\leq} \frac{\big(2C_1+E(E+1)\big)}{T}\frac{\lambda\eta^2(K+1)}{K}\sum_{t=0}^{T-1}\|\sum_{j=1}^pq_j\mathbf{g}_{j}^{(t)}\|^2+\frac{\eta^2(K+1)(E+1)\sigma^2}{KB}
\end{align}
where \ding{192} follows from the definition of weighted gradient diversity and bound $\Lambda(\boldsymbol{w},\mathbf{q})\leq \lambda$.
